\documentclass{article} % For LaTeX2e
\usepackage{nips15submit_e,times}

\usepackage{url}

\usepackage{times}
\usepackage{epsfig}
\usepackage{graphicx}
\usepackage{amsmath}
\usepackage{amssymb}
\usepackage{bbold}
\usepackage{dsfont}

\usepackage{enumerate}

\usepackage[toc,page]{appendix}
\usepackage[font=small,labelfont=bf]{caption}

\usepackage{xcolor}
\usepackage{amsthm}
\usepackage{amsfonts}
\usepackage{caption}
\usepackage{subcaption}
\usepackage{bm}
\usepackage{isomath}
\usepackage[numbers]{natbib}
\usepackage{footmisc}
\usepackage{stmaryrd}
\usepackage{fixltx2e}
\usepackage{dblfloatfix}
\usepackage{pbox}
\usepackage{capt-of}
\usepackage[normalem]{ulem}
\usepackage{multirow}

\usepackage{colortbl}

\makeatletter
\@namedef{ver@everyshi.sty}{}
\makeatother
\usepackage{pgf}

\newcommand{\xpt}{\edef\f@size{\@xpt}\rm}

\nipsfinaltrue

\def\ie{\emph{i.e.}}
\def\etc{\emph{etc}}

\renewcommand\vec[1]{\ensuremath\boldsymbol{#1}}
\renewcommand\cdots{...}

\newcommand{\mY}{\boldsymbol{Y}}
\newcommand{\mZ}{\boldsymbol{Z}}

\newcommand{\vy}{\boldsymbol{y}}

\newcommand{\mX}{\boldsymbol{X}}
\newcommand{\mW}{\boldsymbol{W}}
\newcommand{\vx}{\boldsymbol{x}}

\newcommand{\mbr}[1]{\mathbb{R}^{#1}}

\newcommand{\idx}[1]{\mathcal{I}_{#1}}
\newcommand{\semipd}[1]{\mathcal{S}_{+}^{#1}}
\newcommand{\spd}[1]{\mathcal{S}_{++}^{#1}}

\newcommand{\vz}{\boldsymbol{z}}

\newcommand{\vphi}{\boldsymbol{\phi}}

\newcommand{\bigoh}{\mathcal{O}}

\newcommand{\vj}{\vec{j}}

\newcommand{\fnorm}[1]{\left\|{#1}\right\|_F}

\newcommand{\set}[1]{\left\{#1\right\}}

\DeclareMathOperator*{\argmin}{arg\,min}

\DeclareMathOperator*{\trace}{Tr}

\DeclareMathOperator*{\avg}{Avg}

\newtheorem{proposition}{Proposition}
\newtheorem{remark}{Remark}

\newcommand{\mLambda}{\bm{\lambda}}
\newcommand{\mU}{\bm{U}}
\newcommand{\mV}{\bm{V}}

\newcommand{\mPi}{\bm{\Pi}}

\def\eg{\emph{e.g.}}

\newcommand{\vOnes}{\mathbb{1}}

\newcommand{\mJ}{\boldsymbol{J}}

\newcommand{\overbartwo}[1]{\mkern 2mu\overline{\mkern-2mu#1\mkern-2mu}\mkern 2mu}

\newcommand{\mK}{\boldsymbol{K}}

\newcommand{\cov}{\boldsymbol{\Sigma}}

\newcommand{\mPhi}{\boldsymbol{\Phi}}
\newcommand{\mLam}{\boldsymbol{\Lambda}}

\newcommand{\vmu}{\boldsymbol{\mu}}

\newcommand{\mP}{\boldsymbol{\Theta}}

\newcommand{\stkout}[1]{{\ifmmode\text{\sout{\ensuremath{#1}}}\else\sout{#1}\fi}}

\newcommand{\dsAW}{$\mathcal{A}\!\!\shortrightarrow\!\mathcal{W}$}
\newcommand{\dsAD}{$\mathcal{A}\!\!\shortrightarrow\!\!\mathcal{D}$}

\newcommand{\dsClAr}{$Cl\!\!\shortrightarrow\!\!Ar$}
\newcommand{\dsPrAr}{$Pr\!\!\shortrightarrow\!\!Ar$}

\newcommand{\mR}{\boldsymbol{R}}
\newcommand{\mA}{\boldsymbol{A}}
\newcommand{\comment}[1]{}

\title{\Large Museum Exhibit Identification Challenge for Domain Adaptation and Beyond}

\author{Piotr Koniusz\thanks{Both authors contributed equally.\newline\indent\indent$\!\!$This work is under review and will be updated shortly. Please respect the authors' efforts by not copying/borrowing/plagiarizing bits and pieces of this work for your own gain.}\textsuperscript{$\;\,$,1,2}\qquad Yusuf Tas\textsuperscript{$*$,1,2}\qquad Hongguang Zhang\textsuperscript{2,1}\\\textbf{Mehrtash Harandi\textsuperscript{1,2}\qquad Fatih Porikli\textsuperscript{2}\qquad Rui Zhang\textsuperscript{3}}\\
$^1$Data61/CSIRO, $^2$Australian National University, $^3$Hubei University of Arts and Science\\
firstname.lastname@\{data61.csiro.au\textsuperscript{1}, anu.edu.au\textsuperscript{2}\}, renata\_zhang@sina.com\textsuperscript{3}
}

\newcommand\keywords[1]{}
\nipsfinalcopy % Uncomment for camera-ready version

\begin{document}

\maketitle

\def\arxiv{arxiv.tex}
\begin{abstract}
In this paper, we approach an open problem of artwork identification and propose a new dataset dubbed Open Museum Identification Challenge (Open MIC). It contains photos of exhibits captured in 10 distinct exhibition spaces of several museums which showcase paintings, timepieces, sculptures, glassware, relics, science exhibits%(appliances, interactive mechanisms)
, natural history pieces%(rocks, plants and animals)
, ceramics, pottery, tools and indigenous crafts. The goal of Open MIC is to stimulate research in domain adaptation, egocentric recognition and few-shot learning by providing a testbed complementary to the famous Office dataset which reaches $\sim$90\% accuracy \cite{me_domain}. To form our dataset, we captured a number of images per art piece with a mobile phone and wearable cameras to form the source and target data splits, respectively. To achieve robust baselines, we build on a recent approach that aligns per-class scatter matrices of the source and target CNN streams~\cite{me_domain}. Moreover, we exploit the positive definite nature of such representations by using end-to-end Bregman divergences and the Riemannian metric. We present baselines such as training/evaluation per exhibition and training/evaluation on the combined set covering 866 exhibit identities. %The combined $\sim$60\% accuracy highlights the challenging nature of Open MIC.
%Furthermore, 
As each exhibition poses distinct challenges {\em \eg}, quality of lighting, motion blur, occlusions, clutter, viewpoint and scale variations, rotations, glares, transparency, non-planarity, clipping, we break down results w.r.t. these factors.
\end{abstract}

%In this paper, we approach an open problem of artworks identification and propose a new dataset dubbed Open Museum Identification Challenge (Open MIC). It contains photos of exhibits captured in 10 distinct exhibition spaces of several museums which showcase paintings, timepieces, sculptures, glassware, relics, science exhibits, natural history pieces, ceramics, pottery, tools and indigenous crafts. The goal of Open MIC is to stimulate research in domain adaptation, egocentric recognition and few-shot learning by providing a testbed complementary to the famous Office dataset which reaches ~90% accuracy. To form our dataset, we captured a number of images per art piece with a mobile phone and wearable cameras to form the source and target data splits, respectively. To achieve robust baselines, we build on a recent approach that aligns per-class scatter matrices of the source and target CNN streams. Moreover, we exploit the positive definite nature of such representations by using end-to-end Bregman divergences and the Riemannian metric. We present baselines such as training/evaluation per exhibition and training on the combined set covering 866 exhibit identities. As each exhibition poses distinct challenges \eg, quality of lighting, motion blur, occlusions, clutter, viewpoint and scale variations, rotations, glares, transparency, non-planarity, clipping, we break down results w.r.t. these factors.
\section{Introduction}
\label{sec:intro}

Domain adaptation and transfer learning  are the problems widely studied in computer vision and machine learning communities~\cite{transfer_workshop_1995, transfer_workshop_2016}. They are  inspired by the human cognitive capacity to learn new concepts from very few data samples (cf. training classifier on millions of labeled images from the ImageNet dataset~\cite{ILSVRC15}). Generally, given a new (target) task to learn, the arising question is how to identify the so-called {\em commonality} \cite{tommasi_cvpr10,me_domain} between this task and previous (source) tasks, and transfer knowledge from the source tasks to the target one. Therefore, one has to address three questions: what to transfer, how, and when~\cite{tommasi_cvpr10}.

Domain adaptation and transfer learning utilize annotated and/or unlabeled data and perform tasks-in-hand on the target data \eg, learning new categories from few annotated samples (supervised domain adaptation~\cite{chopra_icml_workshop, tzeng_transfer}), utilizing available unlabeled data (unsupervised~\cite{frustrating_domain_return, ganin_jmlr_adversal} or semi-supervised domain adaptation~\cite{frustrating_domain, tzeng_transfer}), recognizing new categories in embedded spaces (\eg attribute-based) without any training samples (zero-shot learning~\cite{feifei_oneshot}). Problems such as one- and few-shoot learning attempt to train robust class predictors from at most few data points \cite{feifei_oneshot}. 

Recently, algorithms for supervised domain adaptation such as {\em Simultaneous Deep Transfer Across Domains and Tasks}~\cite{tzeng_transfer} and {\em Second- or Higher-order Transfer (So-HoT)} of knowledge \cite{me_domain} combined with Convolutional Neural Networks (CNN) \cite{krizhevsky_alexnet,simonyan_vgg} in end-to-end fashion have reached state-of-the-art results $\sim$90\% accuracy on classic benchmarks such as the Office dataset \cite{saenko_office}. By and large, such an increase in performance is due to fine-tuning of CNNs on the large-scale datasets such as ImageNet \cite{ILSVRC15} and Places Database \cite{places_dataset}. Indeed, fine-tuning of CNN is a powerful domain adaptation and transfer learning tool by itself~\cite{girshick_rich_feat,sermanet_overfeat}.
Furthermore, recent semi-supervised and unsupervised approach to {\em Learning an Invariant Hilbert Space} \cite{samita_domain} has also reached $\sim$90\% accuracy by using generic CNN descriptors vs. $\sim$56\% for SURF. The gap between CNN-based and simpler representations is also visible in the {\em CORAL} method \cite{frustrating_domain_return}, for which performance varies between 46\% and 70\% accuracy. Thereby, these works exhibit saturation for CNN features when evaluated on the Office \cite{saenko_office} dataset or its newer Office+Caltech 10 variant \cite{office_calt10}. 
 
Therefore, we propose a new dataset for the task of exhibit identification in museum spaces that challenges domain adaptation and fine-tuning due to its significant domain shifts between the source and target subsets.

For the source domain, we captured the photos in a controlled fashion by Android phones \eg, we ensured that each exhibit is centered and non-occluded in photos. We prevented adverse capturing conditions and did not mix multiple objects per photo unless they were all part of one exhibit. We captured 2--30 photos of each art piece from different viewpoints and distances in their natural settings.

For the target domain, we employed an egocentric setup to ensure {\em in-the-wild} capturing process. We equipped 2 volunteers per exhibition with cheap wearable cameras and let them  stroll and interact with artworks at their discretion. Such a capturing setup is applicable to preference and recommendation systems \eg, a curator takes training photos of exhibits with an Android phone while visitors stroll with wearable cameras to capture data from the egocentric perspective for a system to reason about the most popular exhibits. 
Open MIC contains 10 distinct source-target subsets %\footnote{\label{foot:subset_seven}We are in process of adding the 7\textsuperscript{th} subset containing Natural History Exhibition. We will make the dataset publicly available for research.}
of images from 10 different kinds of museum exhibition spaces, each exhibiting various photometric and geometric challenges, as detailed in Section \ref{sec:expts}. %Paintings from Shenzhen Museum ({\em Shn}), the Clock and Watch Gallery ({\em Clk}) and the Indian and Chinese Sculptures ({\em Scl}) from the Palace Museum, the Xiangyang Science Museum ({\em Sci}), the European Glass Art ({\em Gls}) and the Collection of Cultural Relics ({\em Rel}) from the Hubei Provincial Museum, the Nature, Animals and Plants in Ancient Times from Shanghai Natural History Museum ({\em Shg}), the Comprehensive Historical and Cultural Relics from Shaanxi History Museum ({\em Shx}), the Sculptures, Pottery and Bronze Figurines from the Cleveland Museum of Arts ({\em Clv}), and Indigenous Arts from Honolulu Museum Of Arts ({\em Hon}).

%In detail, we annotated each image with labels of art pieces visible in it. The wearable cameras were set to capture an image every 10s and they operated {\em in-the-wild}, \eg, volunteers had no control over shutter, focus, centering, \etc. Therefore, the collected target subsets exhibit many realistic challenges, \eg, sensor noises, motion blur, occlusions, background clutter, varying viewpoints, scale changes, rotations, glares, transparency, non-planar surfaces, clipping, multiple exhibits, active light, color inconstancy, very large or small exhibits, to name but a few phenomena. Every subset contains 37--166 exhibits to identify and 5 train, val, and test splits. In total, our dataset contains 866 unique exhibit labels, 8560 source and 7596 target images.

%To demonstrate the intrinsic difficulty of the Museum dataset, we start by fine-tuning the VGG16 architecture \cite{simonyan_vgg} on the source subsets and perform testing on the target splits; each split corresponds to one volunteer's walk. This protocol emulates the following scenario: museum curator takes a few of images per art piece with a mobile phone to train quickly a system which will then identify exhibits interacted with by the visitors from wearable cameras attached to their clothes at the exhibition entrance. 

To demonstrate the intrinsic difficulty of Open MIC, we chose useful baselines in supervised domain adaptation detailed in Section \ref{sec:expts}. They include fine-tuning CNNs on the source and/or target data and training a state-of-the-art So-HoT model \cite{me_domain} which we equip with non-Euclidean distances \cite{anoop_logdet,PEN06} for robust end-to-end learning.

We provide various evaluation protocols which include: (i) training/evaluation per exhibition subset, (ii) training/testing on the combined set that covers all 866 identity labels, (iii) testing w.r.t. various scene factors annotated by us such as quality of lighting, motion blur, occlusions, clutter, viewpoint and scale variations, rotations, glares, transparency, non-planarity, clipping, \etc.

Moreover, we introduce a new evaluation metric inspired by a saliency problem detailed next. As numerous exhibits can be captured in a target image, we asked our volunteers to enumerate in descending order the labels of most salient/central exhibits they had interest in at a given time followed by less salient/distant exhibits. As we ideally want to understand the volunteers' preferences, the classifier has to decide which detected exhibit is the most salient. We note that the annotation- and classification-related processes are not free of noise. Therefore, we propose to not only look at the top-$k$ accuracy known from ImageNet \cite{ILSVRC15} but to also check if any of top-$k$ predictions are contained within the top-$n$ fraction of all ground-truth labels enumerated for a target image. We refer to this as a top-$k$-$n$ measure.

%Moreover, we propose an alternative approach to the supervised domain adaptation which is based on the weighted mixture of alignments of second-order scatter statistics computed from the source and target domains. %
%
%We first briefly compare the VGG16 \cite{simonyan_vgg} and GoogLeNet architectures \cite{google_net} as well as as the Eucldiean distance, the Jensen-Bregman LogDet Divergence (JBLD) \cite{anoop_logdet}

To obtain convincing baselines, we balance the use of an existing approach \cite{me_domain} with our mathematical contributions and evaluations. The So-HoT model \cite{me_domain} uses the Frobenius metric for partial alignment of within-class statistics obtained from CNNs. 
The hypothesis behind such modeling is that the partially aligned statistics capture so-called {\em commonality} \cite{tommasi_cvpr10,me_domain} between the source and target domains; thus facilitating knowledge transfer. 
For the pipeline in Figure \ref{fig:cnn_all}, we use two CNN streams of the VGG16 network~\cite{simonyan_vgg} which correspond to the source and target domains. We build scatter matrices, one per stream per class, from feature vectors of the {\em fc} layers. To exploit benefits of geometry of positive definite matrices, we regularize and align scatters  by the Jensen-Bregman LogDet Divergence ({\em JBLD}) \cite{anoop_logdet} in end-to-end manner and compare to the Affine-Invariant Riemannian Metric ({\em AIRM}) \cite{PEN06,bhatia_pdm}. However, evaluations of gradients of non-Euclidean distances are slow for typical $4096\!\times\!4096$ dimensional matrices. We show by the use of Nystr\" om projections that, with typical numbers of data samples per source/target per class being $\sim$50 in domain adaptation, evaluating such distances can be fast and exact. %evaluating non-Euclidean distances/derivatives in the alignment phase can be fast and exact.

To summarize, our contributions are as follows: (i) we collect and annotate a new challenging Open MIC dataset with domains consisting of the pictures taken by Android phones and wearable cameras; the latter exhibiting a series of realistic distortions due to the egocentric capturing process, (ii) we compute useful baselines, provide various evaluation protocols, statistics and top-$k$-$n$ results, as well as include breakdown of results w.r.t. annotated by us scene factors, (iii) we use non-Euclidean JBLD and AIRM distances for end-to-end training of the supervised domain adaptation approach and we exploit the Nystr\" om projections to make this training tractable. 
To our best knowledge, these distances have not been used before in the supervised domain adaptation due to their high computational complexity.

\begin{figure}[t]%htbp % left bottom right top
\centering%%%%\vspace{-0.3cm}
%
%\hspace{0.1cm}
\begin{subfigure}[b]{0.99\linewidth}
\centering\includegraphics[trim=0 0 0 0, clip=true, width=8.2cm]{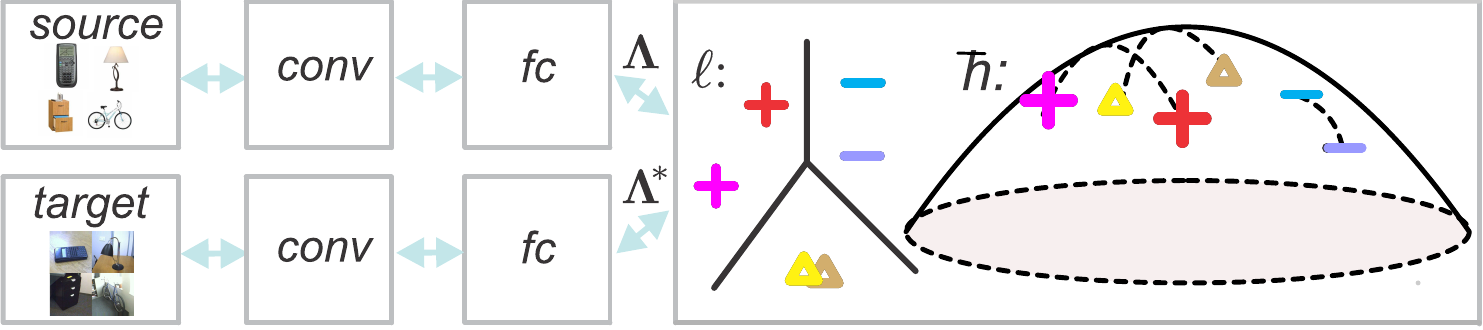}
%\phantomcaption
\caption{\label{fig:cnn1}}
\vspace{-0.1cm}
\end{subfigure}
\begin{subfigure}[b]{0.99\linewidth}
\centering\includegraphics[trim=0 0 0 0, clip=true, width=8.2cm]{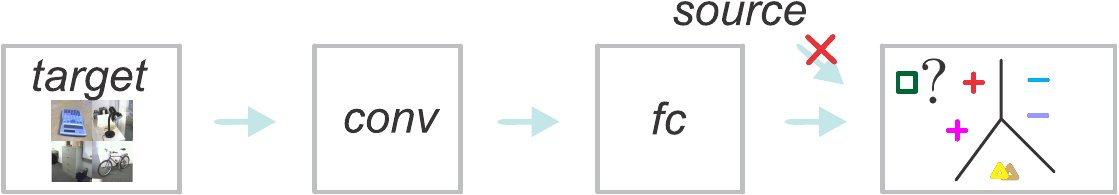}
%\phantomcaption
\caption{\label{fig:cnn2}}
\end{subfigure}
\vspace{-0.2cm}
\caption{The pipeline. Figure \ref{fig:cnn1} shows the source and target network streams which merge at the classifier level. The classification and alignment losses $\ell$ and $\hbar$ take the data  $\mLam$ and $\mLam^{*\!}$ from both streams and participate in end-to-end learning. Loss $\hbar$ aligns covariances on the manifold of $\spd{}$ matrices. At the test time, we use the target stream and the trained classifier as in Figure \ref{fig:cnn2}.
}\vspace{-0.3cm}
\label{fig:cnn_all}
\end{figure}

\section{Related Work}
\label{sec:related_work}

We start by describing the most popular datasets for the problem at hand and explain how the Open MIC dataset differs from them. Subsequently, we describe various domain adaptation approaches which are related to our work.

\vspace{0.05cm}
\noindent{\textbf{Datasets.}}
A popular dataset for evaluating against the effect of domain shift is the Office dataset~\cite{saenko_office} which contains 31 object categories and three domains: Amazon, DSLR and Webcam. The 31 categories in the dataset consist of objects commonly encountered in the office setting, such as keyboards, file cabinets, and laptops. The Amazon domain contains images which were collected from a website of on-line merchants. Its objects appear on clean backgrounds and at a fixed scale. 
The DSLR domain contains low-noise high resolution images of object captured from different viewpoints while Webcam contains low resolution images. %, images of low resolution exhibit significant noise and color as well as white balance artifacts. However, this domain consists of objects and layouts highly similar to the DSLR domain. 
The Office dataset has been used in numerous publications \cite{frustrating_domain_return, tzeng_transfer, ganin_jmlr_adversal, chopra_icml_workshop, xiong_eccv16, kuzborskij_cvpr16, tommasi_eccv16,samita_domain} that address domain adaptation, to name but a few of approaches. %Moreover, recent extension of the Office dataset includes a new Caltech 10 domain \cite{office_calt10}.
Its recent extension includes a new Caltech 10 domain \cite{office_calt10}.

The Office dataset is primarily used for the transfer of knowledge about object categories between domains. In contrast, our dataset addresses the transfer of instances between domains. Each domain of the Open MIC dataset contains 37--166 specific instances to distinguish from (866 in total) compared to relatively low number of 31 classes in the Office dataset. Moreover, our target subsets are captured in an egocentric manner \eg, we did not align objects to the center of images or control the shutter \etc.

A recent large collection of datasets for domain adaptation was proposed in technical report \cite{tomassi_tesbed} to study cross-dataset domain shifts in object recognition with use of %. For instance, supervised approach \cite{tzeng_transfer} utilizes 40 categories shared between 
the ImageNet, Caltech-256, SUN, and Bing datasets. % from this collection.
Even larger is the latest Visual Domain Decathlon challenge \cite{decathlon_challenge} which combines datasets such as ImageNet, CIFAR--100, Aircraft, Daimler pedestrian classification, Describable textures, German traffic signs, Omniglot, SVHN, UCF101 Dynamic Images, VGG--Flowers. In contrast, our dataset contains highly varied target appearances which are challenging in few-shot learning scenarios. We target the identity recognition across exhibits captured in egocentric setting which vary from paintings to sculptures to glass to pottery to figurines. Moreover, some artworks in our dataset exhibit fine-grained traits as they are hard to distinguish from without the expert knowledge. %Our exhibits contains sole planar and non-planar objects, collections of objects with varied reflectance properties and even interactive light emitting devices.

The PIE Multiview dataset \cite{multi_pie} includes face images of 67 subjects and exhibits different viewpoints, varies in illumination and expressions. It has been used in the instance-based domain adaptation \cite{samita_domain}. Our Open MIC however is not limited to instances of faces or controlled capture setting. Open MIC contains diverse 10 subsets with paintings, timepieces, sculptures, science exhibits, glasswork, relics, ancient animals, plants, figurines, ceramics, native arts \etc.

%Noteworthy are heterogeneous combinations of RGB-D with Caltech256 and Pascal VOC2007 with TU Berlin datasets used in domain adaptation \cite{me_domain} for object category transfers, and the 20 Newsgroups dataset \cite{news20} used in the machine learning community for domain adaptation on text.
%
%the machine learning community have also their favorite datasets \ie, the 20 Newsgroups data set \cite{news20} used in problems such as text-related domain adaptation.

\vspace{0.05cm}
\noindent{\textbf{Domain adaptation algorithms.}}  
Deep learning has been used in the context of domain adaptation in numerous recent works \eg,~\cite{tzeng_transfer, ganin_jmlr_adversal, chopra_icml_workshop, xiong_eccv16, kuzborskij_cvpr16, tommasi_eccv16,me_domain}. These works establish the so-called commonality between domains. In~\cite{tzeng_transfer}, the authors propose to align both domains via the cross entropy which `maximally confuses' both domains for supervised and semi-supervised settings. 
%In~\cite{ganin_jmlr_adversal}, an unsupervised approach utilizes the assumption that predictions must be made based on features which cannot discriminate between the source and target domains. Our work differs from these approaches in that we define the commonality as the desired degree of overlap between the symmetric positive semidefinite scatters of the source and target measured by the Jensen-Bregman LogDet divergence. After such an alignment, we allow the non-overlapping tails of targer distribution to also guide the learning process.
%
%
%
\begin{table*}[t]%[H]
\vspace{-0.4cm}
\begin{center}
{
%\setlength{\tabcolsep}{0.15em}
%\footnotesize
%\setlength{\tabcolsep}{0.15em}
%\footnotesize
\setlength{\tabcolsep}{0.2em}
\centering
\hspace{-0.7cm}
\begin{tabular}{c | c c c c c c c}
%\hline
\multirow{2}{*}{Dist./Ref.} & \multirow{2}{*}{$d^2(\cov,\cov^{*})$} & \multirow{2}{*}{Invar.} & Tr. & \multirow{2}{*}{Geo.} & $d$ if & $\triangledown_{\cov}$ &\multirow{2}{*}{$\frac{\partial d^2(\cov,\cov^{*})}{\partial\cov}$}\\
& & & \kern-0.4em Ineq. & & $\semipd{}$ & if $\semipd{}$ &\\
\hline
Frobenius & $||\cov\!-\!\cov^{*}||_F^2$ & rot. & yes & no & fin. & fin. & $2(\cov\!-\!\cov^{*})$\\
AIRM \cite{PEN06} & $||\cov^{-\frac{1}{2}}\cov^{*}\cov^{-\frac{1}{2}}||_F^2$ & aff./inv. & yes & yes & $\infty$ & $\infty$ & \kern-0.6em$-2\cov^{-\frac{1}{2}}\!\log(\cov^{-\frac{1}{2}}\cov^{*}\cov^{-\frac{1}{2}})\cov^{-\frac{1}{2}}$\kern-0.9em\\
%KLDM \cite{wang_jeffreys}& $\frac{1}{2}\!\trace(\cov\!\cov^{*-1}\!\!+\!\!\cov^{-1}\!\cov^*\!)\!-\!d'\!$ & aff./inv. & no & no & $\infty$ & $\infty$ & $\frac{1}{2}(\cov^{*-1}\!-\!\cov^{-1}\!\cov^{*\!}\!\cov^{-1})$\\
JBLD \cite{anoop_logdet} & $\log\!\left|\frac{\cov\!+\!\cov^*\!}{2}\right|\!-\!\frac{1}{2}\log\!\left|\cov\!\cov^*\!\right|$ & aff./inv. & no & no & $\infty$ & $\infty$ & $(\cov\!+\!\cov^*\!)^{-1}\!-\!\frac{1}{2}\!\cov^{-1}$\\
%LogE \cite{arsigny_logeuclid} & $||\log\cov\!-\!\log\cov^{*}||_F^2$ & \kern-1.2em rot./sc./inv. & yes & yes & $\infty$ & $\infty$ & \text{See \cite{anoop_logdet} for details.}
\end{tabular}
}
\end{center}
%\vspace{-0.5cm}
\caption{Frobenius, JBLD and AIRM distances and their properties from the literature. These distances operate between a pair of arbitrary matrices $\cov$ and $\cov^*\!$ which are points in $\spd{}$ (and/or $\semipd{}$ for Frobenius).}
\label{tab:non-euclid}
\vspace{-0.3cm}
\end{table*}
%
%
%
%Moreover, 
In~\cite{chopra_icml_workshop}, the authors capture the `interpolating path' between the source and target domains using linear projections into a low-dimensional subspace on the Grassman manifold. In~\cite{xiong_eccv16}, the authors propose to learn the transformation between the source and target by the deep regression network. %based on an assumption that, given the deep nature of CNNs, they can perform highly non-linear regression. 
Our model differs in that our source and target network streams co-regularize each other via the JBLD or AIRM distance that respects the non-Euclidean geometry of the source and target matrices. We perform an alignment of scatter matrices advocated in \cite{me_domain}. %accommodate the domain-specific parts that should not be aligned by learning weighted alignments.

For visual domains, the domain adaptation can be applied in the spatially-local sense to target so-called {\em roots} of domain shift. In~\cite{tommasi_eccv16}, the authors utilize so-called `domainness maps' which capture locally the degree of domain specificity. Our work is orthogonal to this method. We perform domain adaptation globally in the spatial sense, however, our ideas can be extended to a spatially-local setting. 

Some recent works enforce correlation between the source and target distributions \eg, the authors of \cite{yeh_cca_hetero} utilize a correlation subspace as a joint representation for associating the data across different domains. They also use kernelized CCA. In~\cite{frustrating_domain_return}, the authors propose an unsupervised domain adaptation by the correlation alignment. In \cite{me_domain}, the authors perform class-specific alignment of source and target distributions with use of tensors and the Frobenius norm. Our work is similar in spirit as it utilizes a similar general setup. However, we first project class-specific vector representations from the {\em fc} layers of the source and target CNN streams to the common space via Nystr\" om projections for tractability and then we combine them with the JBLD or AIRM distance to exploit the (semi)definite positive nature of scatter matrices. We perform end-to-end learning which requires non-trivial derivatives of JBLD/AIRM distance and Nystr\" om projections for computational efficiency. %To our best knowledge, these distances have not been used before in the supervised domain adaptation due to high computational complexity. % for matrices of large size.

\section{Background}
\label{sec:background}

In this section, we review our notations and the necessary background on scatter matrices,
Nystr\" om projections, the Jensen-Bregman LogDet ({\em JBLD}) divergence \cite{anoop_logdet} and the Affine-Invariant Riemannian Metric ({\em AIRM}) \cite{PEN06,bhatia_pdm}. % which are used in the sequel.

\subsection{Notations}
\label{sec:notations}
Let $\vx\in\mbr{d}$ be a $d$-dimensional feature vector. $\idx{N}$ stands for the index set $\set{1, 2,\cdots,N}$. 
The Frobenius norm of matrix is given by  $\fnorm{\mX}\!\!=\!\!\!\sqrt{\sum\limits_{m,n} \!\!X_{mn}^2}$, where $X_{mn}$ represents the $\left(m,n\right)$-th element of $\mX$.  
The spaces of symmetric positive semidefinite and definite matrices are $\semipd{d}$ and $\spd{d}$. A vector with all coefficients equal one is denoted by $\vOnes$ and $\mJ_{mn}$ is a matrix of all zeros with one at position $(m,n)$.
 %The MATLAB-style operators for matrix vectorization and matrix reshaping to the size of $(m,n)$ are denoted by $(:)$, \ie, $\mX_{(:)}$, and $\res(\mX,m,n)$.

\subsection{Nystr\" om Approximation}
\label{sec:nystrom}
In our domain adaptation model, we rely on Nystr\" om projections, thus, we review their general mechanism first.

\begin{proposition}
Suppose $\mX\!\in\!\mbr{d\times N}$ and $\mZ\!\in\!\mbr{d\times N'\!}$ store $N$ feature vectors and  $N'$ pivots (vectors used in approximation) of dimension $d$ in their  columns, respectively. Let $k:\mbr{d}\times \mbr{d}\to\mathbb{R}$ be a positive definite kernel. 
We form two kernel matrices 
$\mK_{\mZ\mZ}\!\in\!\spd{N'\!}$ and $\mK_{\mZ\mX}\!\in\!\mbr{N'\!\!\times\!N}$ with their $(i,j)$-th elements being
$k(\vz_i,\vz_j)$ and $k(\vz_i,\vx_j)$, respectively. 
Then, the Nystr\" om feature map $\tilde{\mPhi}\!\in\!\mbr{N'\!\!\times\!N}\!\!$, whose columns correspond to the input vectors in $\mX$,
and the Nystr\" om approximation of kernel $\mK_{\mX\mX}$ for which $k(\vx_i,\vx_j)$ is its $(i,j)$-th entry,
are given by:
\begin{align}
& \tilde{\mPhi}= \mK_{\mZ\mZ}^{-0.5}\mK_{\mZ\mX} \quad\text{{\normalfont and}}\quad \mK_{\mX\mX}\approx\tilde{\mPhi}^T\tilde{\mPhi}.\label{eq:nyst2}
\end{align}
%Moreover, the Nystr\" om approximation of kernel $\mK_{\mX\mX}$, for which $k(\vx_i,\vx_j)$ is its (i,j)-th entry, is: 
%\begin{align}
%\mK_{\mX\mX}\approx\tilde{\mPhi}^T\tilde{\mPhi}.\label{eq:nyst2}
%\end{align}
\end{proposition}
\begin{proof}
\vspace{-0.2cm}
See \cite{bo_nystrom} for details.
\end{proof}
\begin{remark}
The quality of approximation of \eqref{eq:nyst2} depends on the kernel $k$, data points $\mX$, pivots $\mZ$ and their number $N'\!$. In the sequel, we exploit a specific setting under which $\mK_{\mX\mX}\!=\!\tilde{\mPhi}^T\tilde{\mPhi}$ which indicates no approximation loss.
\end{remark}

\subsection{Scatter Matrices}
\label{sec:cov_mat}
%Our distances are applicable to any $\spd{}$ and $\semipd{}$ matrices. However, 
We make a frequent use of distances $d^2(\cov,\cov^{*\!})$ that operate between covariances $\cov\!\equiv\!\cov(\mPhi)$ and $\cov^*\!\!\equiv\!\cov(\mPhi^*\!)$ on feature vectors. Therefore, we provide a useful derivative of $d^2(\cov,\cov^{*\!})$ w.r.t. feature vectors $\mPhi$.

\begin{proposition}
\label{prop:chain_for_phi}
Suppose $\mPhi\!=\![\vphi_1,\cdots,\vphi_N]$ and $\mPhi^{*}\!\!=\![\vphi^*_1\!,\cdots,\vphi^*_{N^*}\!]$ are some feature vectors of quantity $N$ and $N^{*\!}$, \eg, formed by Eq. \eqref{eq:nyst2} and used to evaluate $\cov$ and $\cov^{*}\!$ with $\vmu$ and $\vmu^*\!$ being 
the mean of $\mPhi$ and $\mPhi^*\!$, respectively. 
Then, derivatives of $d^2\!\equiv\!d^2(\cov,\cov^{*})$ w.r.t. $\mPhi$ and $\mPhi^{*}\!$ are:
\begin{align}
&\!\!\!\!\textstyle\frac{\partial d^2(\cov,\cov^{*})}{\partial\mPhi}\!=\!\frac{2}{N}\!\frac{\partial d^2}{\partial\cov}\!\scriptstyle\left(\mPhi\!-\!\vmu\vOnes^T\right),%\nonumber\\
%& 
\textstyle\frac{\partial d^2(\cov,\cov^{*})}{\partial\mPhi^*}\!=\!\frac{2}{N^*}\!\frac{\partial d^2}{\partial\cov^*}\!\scriptstyle\left(\mPhi^{*}\!\!-\!\vmu^{*}\vOnes^T\right).
%
% & \textstyle\frac{\partial d^2(\cov,\cov^{*})}{\partial\mPhi}\!=\!\frac{2}{N}\!\frac{\partial d^2(\cov,\cov^{*})}{\partial\cov}\!\scriptstyle\left(\mPhi\!-\!\vmu\vOnes^T\right),%\nonumber\\
\label{eq:general_chain}
\end{align}
Moreover, assume some projection matrix $\stkout{\mZ}$. Then for $\mPhi'\!=\!\stkout{\mZ}[\vphi_1,\cdots,\vphi_N]$ and $\mPhi'^{*}\!\!=\!\stkout{\mZ}[\vphi^*_1\!,\cdots,\vphi^*_{N^*}\!]$ with covariances $\cov'$, $\cov'^*\!$, means $\vmu'$, $\vmu'^*\!$ and $d'^2\!\equiv\!d^2(\cov'\!,\cov'^*\!)$, we obtain:
\begin{align}
&\!\!\!\!\!\!\!\!\textstyle\frac{\partial d^2(\cov,\cov^{*})}{\partial\mPhi}\!=\!\frac{2\stkout{\mZ}^T}{N}\!\frac{\partial d'^2}{\partial\cov'}\!\scriptstyle\left(\mPhi'\!\!-\!\vmu'\vOnes^T\right),%\nonumber\\
%& 
\textstyle\frac{\partial d^2(\cov,\cov^{*})}{\partial\mPhi^*}\!=\!-\frac{2\stkout{\mZ}^T}{N^*}\!\frac{\partial d'^2}{\partial\cov'^*}\!\scriptstyle\left(\mPhi'^{*}\!\!-\!\vmu'^{*}\vOnes^T\right).
\label{eq:general_chain2}
\end{align}
\end{proposition}
\begin{proof}
\vspace{-0.2cm}
%Follows from simply applying the chain rule.
See our supplementary material. %\MH{do we have a supp material here?}
\end{proof}

\subsection{Non-Euclidean Distances}
\label{sec:noneuclid}

In Table~\ref{tab:non-euclid}, we list the distances $d$ with derivatives w.r.t. $\cov$ used in the sequel. 
We indicate properties such as invariance to rotation ({\em rot.}), %scaling ({\em sc.}), 
affine mainpulations ({\em aff.}) and inversion ({\em inv.}). Moreover, we indicate which distances meet the triangle inequality ({\em Tr. Ineq.}) and which are geodesic distances ({\em Geo.}). Lastly, we indicate if the distance $d$ and its gradient $\triangledown_{\cov}$ are finite ({\em fin.}) or infinite ($\infty$) for $\semipd{}$ matrices. This last property indicates that JBLD and AIRM distances require some regularization as our covariances are $\semipd{}$.

\section{Problem Formulation}
\label{sec:problem}

In this section, we equip the supervised domain adaptation approach So-HoT \cite{me_domain} with the JBLD and AIRM distances. Moreover, we show how to use the Nystr\" om projections to make our computations fast.

\subsection{Supervised Domain Adaptation}
\label{sec:dom_adapt}

Suppose $\idx{N}$ and $\idx{N^*}\!$ are the indexes of $N$ source and $N^*\!$ target training data points. $\idx{N_c}$ and $\idx{N_c^*}\!$ are the class-specific indexes for $c\!\in\!\idx{C}$, where $C$ is the number of classes (exhibit identities). Furthermore, suppose 
we have feature vectors from an {\em fc} layer 
of the source network stream, one per image, and their associated labels. Such pairs are given by $\mLam\!\equiv\!\{(\vphi_n, y_n)\}_{n\in\idx{N}}$, where $\vphi_n\!\in\!\mbr{d}$ and $y_n\!\in\!\idx{C}$, $\forall n\!\in\!\idx{N}$. %, as shown in Figure \ref{fig:cnn1}. 
For the target data, by analogy, we define pairs $\mLam^{*\!}\!\equiv\!\{(\vphi^*_n, y^*_n)\}_{n\in\idx{N}^*}$, where $\vphi^*\!\!\in\!\mbr{d}$ and $y^*_n\!\!\in\!\idx{C}$, $\forall n\!\in\!\idx{N}^*$. %, as the pairs of feature vectors and labels from the target network stream. 
Class-specific sets of feature vectors are given as $\mPhi_c\!\equiv\!\{\vphi^c_n\}_{n\in\idx{N_c}}$ and $\mPhi_c^*\!\!\equiv\!\{\vphi^{*c}_n\}_{n\in\idx{N_c^*\!}}$, $\forall c\!\in\!\idx{C}$. Then, $\mPhi\!\equiv\!(\mPhi_1,\cdots,\mPhi_C)$ and $\mPhi^*\!\!\equiv\!(\mPhi^*_1,\cdots,\mPhi^*_C)$.
 Note that we write the asterisk symbol in superscript (\eg~${\vphi}^*$) to denote variables related to the target network while the source-related and generic variables have no such indicator. 
Figure \ref{fig:cnn_all} shows our setup. We formulate our problem as a trade-off between the classifier and alignment losses $\ell$ and $\hbar$:
\vspace{0.00cm}
\begin{align}
\hspace{0.2cm}
&\!\!\!\!\!\!\!\!\!\!\!\!\!\!\!\!\!\!\argmin\limits_{\;\;\substack{\mW\!,\mW^*\!\!\!\!,\mP,\mP^*\!\!\comment{ ,\vec{\zeta},\vec{\overbartwo{\zeta}} }\\\;\;\;\;\text{s. t. }||\vphi_n||_2^2\leq\tau,\\\;\;\;\;\;\;\;\;\,||\vphi^*_{n'}||_2^2\leq\tau,\\\;\;\;\;\forall n\in\idx{N}\!, n'\!\in\idx{N}^*}} %\left\{\vphi_n, y_n\right\}_{n\in\idx{N}} \left\{\vphi^*_n, y^*_n\right\}_{n\in\idx{N}^*}
\ell\!\left(\mW\!,\mLam\right)\!+\!\ell\!\left(\mW^{*\!}\!,\mLam^{*\!}\right)\!+\!\eta||\mW\!-\!\mW^{*\!}||_F^2\;+\label{eq:main_obj1}\\[-35pt]
%&\qquad\qquad\quad\!\!\!\!\!\,\frac{\alpha_1}{C}||\vec{\zeta}\!\!-\!\!\vOnes||_2^2\!+\!\frac{\alpha_2}{C}||\vec{\overbartwo{\zeta}}\!\!-\!\!\vOnes||_2^2\;+\nonumber\\%[-33pt]
%
&\qquad\qquad\quad\!\!\!\!\!\underbrace{\frac{\sigma_1}{C}\!\!\sum_{c\in\idx{C}}\!\comment{ \zeta_{c} }d^2_g\left(\cov_c,\cov^*_c\right)%\nonumber\\
\!+\!\!\frac{\sigma_2}{C}\!\!\sum_{c\in\idx{C}}\!\comment{ \overbartwo{\zeta}_c}||\vmu_c\!\!-\!\!\vmu_c^*||_2^2.
}_{\hbar(\mPhi,\mPhi^*\!)}\nonumber%+\nonumber\\
%
%&\qquad\qquad\quad\!\!\!\!\!\frac{\alpha_1}{C}||\vec{\zeta}\!\!-\!\!\vOnes||_2^2\!+\!\frac{\alpha_2}{C}||\vec{\overbartwo{\zeta}}\!\!-\!\!\vOnes||_2^2\!
%.\nonumber
%
\vspace{0.2cm}
\end{align}
Note that Figure \ref{fig:cnn1} indicates by the elliptical/curved shape that $\hbar$ performs the alignment on the $\semipd{}$ manifold along exact (or approximate) geodesics. 
For $\ell$, we employ a generic loss used by CNNs \eg, Softmax. For the source and target streams, the matrices $\mW,\mW^*\!\!\in\!\mbr{d\times C}$ contain unnormalized probabilities (c.f. hyperplanes of two SVMs). %, $\vb,\vb^*\!\!\in\!\mbr{C}$ are the bias terms. %For the target stream, $\mW^*\!$ and $\vb^*\!$ play an analogous role. 
In Equation \eqref{eq:main_obj1}, separating the class-specific distributions is addressed by $\ell$ while attracting the within-class scatters of both network streams is handled by $\hbar$. Variable %$\eta$ and $\eta^*\!$ control regularization of $\mW$ and $\mW^*\!$ while 
$\eta$ controls the proximity between $\mW$ and $\mW^*\!$ which encourages the similarity between decision boundaries of classifiers.

Our loss $\hbar$ depends on two sets of variables $(\mPhi_1,\cdots,\mPhi_C)$ and $(\mPhi^*_1,\cdots,\mPhi^*_C)$ -- one set per network stream. Feature vectors $\mPhi(\mP)$ and $\mPhi^*\!(\mP^*\!)$ depend on the parameters of the source and target network streams $\mP$ and $\mP^*\!$ that we optimize over.
%\eg, they represent %coefficients of convolutional filters and weights of {\em fc} layers.
%
$\cov_c\!\equiv\!\cov(\mPi(\mPhi_c))$, $\cov^*_c\!\equiv\!\cov(\mPi(\mPhi^*_c))$, $\vmu_c(\mPhi)$ and $\vmu^*_c(\mPhi^*)$ denote the covariances and means, respectively, one covariance/mean pair per network stream per class. Coeffs. $\sigma_1$, $\sigma_2$ control the degree of the scatter and mean alignment, $\tau$ controls the $\ell_2$-norm of feature vectors. %Class-specific weights $\vzeta,\vzetabar\!\in\!\mbr{C}$ adjust the degree of alignment per class between the within-class scatters and the means. Variables $\alpha_1$ and $\alpha_2$ control weight deviation.

The Nystr\" om projections are denoted  by $\mPi$. 
Table \ref{tab:non-euclid} indicates that back-propagation on the JBLD and AIRM distances involves inversions of $\cov_c$ and $\cov^*$ to be performed for each  $c\!\in\!\idx{C}$ according to \eqref{eq:main_obj1}. As these covariances are formed from $4096$ dimensional feature vectors of the {\em fc} layer, such inversions are too costly to run fine-tuning \eg, $4s$ per iteration is prohibitive. 
Thus, we demonstrate next how the Nystr\" om projections can be combined with $d_g$.

\ifdefined\arxiv
\newcommand{\SrcImgW}{0.060}
\newcommand{\SrcImgH}{1.49cm}
\else
\newcommand{\SrcImgW}{0.062}
\newcommand{\SrcImgH}{1.9cm}
\fi

\begin{figure*}[t]%htbp % left bottom right top
\centering%%%%\vspace{-0.3cm}
%
%\comment{
\begin{subfigure}[b]{\SrcImgW\linewidth}
\centering\includegraphics[trim=0 0 0 0, clip=true, height=\SrcImgH]{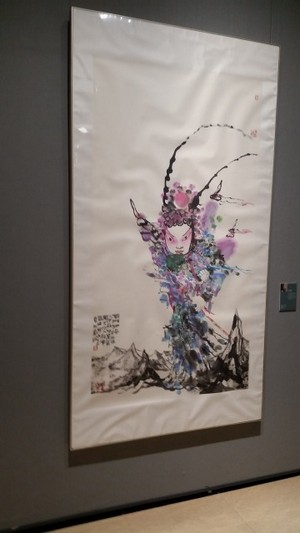}
\end{subfigure}
\begin{subfigure}[b]{\SrcImgW\linewidth}
\centering\includegraphics[trim=0 0 0 0, clip=true, height=\SrcImgH]{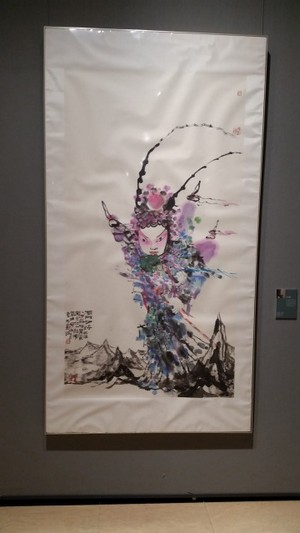}
\end{subfigure}
\begin{subfigure}[b]{\SrcImgW\linewidth}
\centering\includegraphics[trim=0 0 0 0, clip=true, height=\SrcImgH]{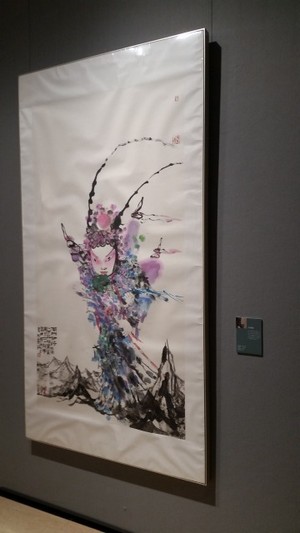}
\end{subfigure}
\begin{subfigure}[b]{\SrcImgW\linewidth}
\centering\includegraphics[trim=0 0 0 0, clip=true, height=\SrcImgH]{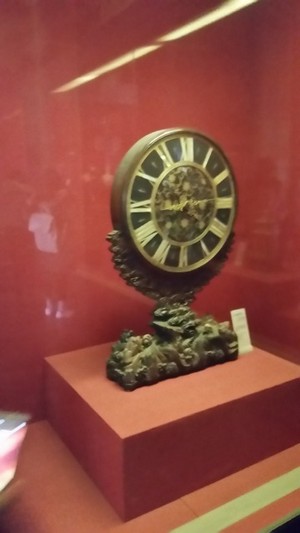}
\end{subfigure}
\begin{subfigure}[b]{\SrcImgW\linewidth}
\centering\includegraphics[trim=0 0 0 0, clip=true, height=\SrcImgH]{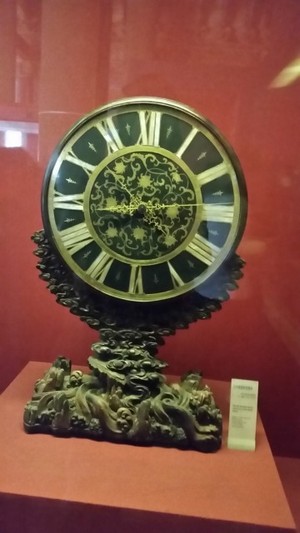}
\end{subfigure}
\begin{subfigure}[b]{\SrcImgW\linewidth}
\centering\includegraphics[trim=0 0 0 0, clip=true, height=\SrcImgH]{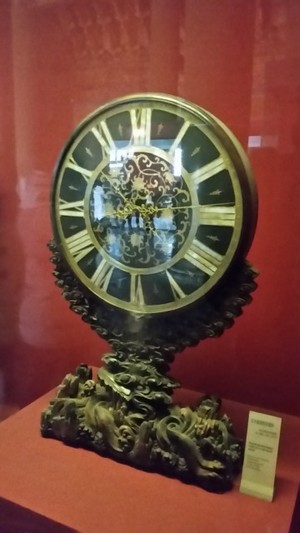}
\end{subfigure}
\begin{subfigure}[b]{\SrcImgW\linewidth}
\centering\includegraphics[trim=0 0 0 0, clip=true, height=\SrcImgH]{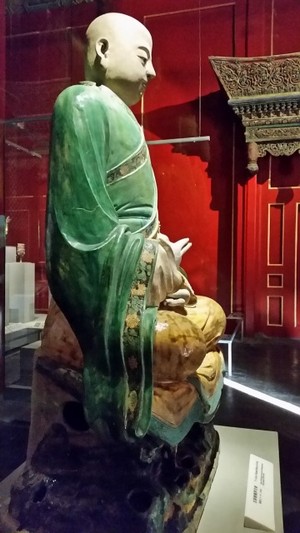}
\end{subfigure}
\begin{subfigure}[b]{\SrcImgW\linewidth}
\centering\includegraphics[trim=0 0 0 0, clip=true, height=\SrcImgH]{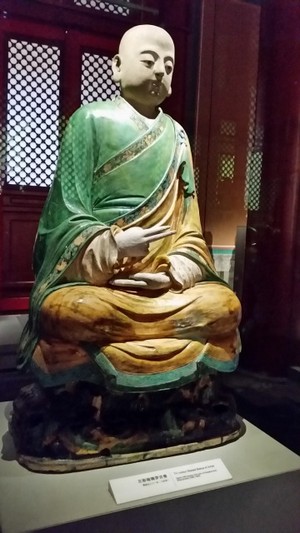}
\end{subfigure}
\begin{subfigure}[b]{\SrcImgW\linewidth}
\centering\includegraphics[trim=0 0 0 0, clip=true, height=\SrcImgH]{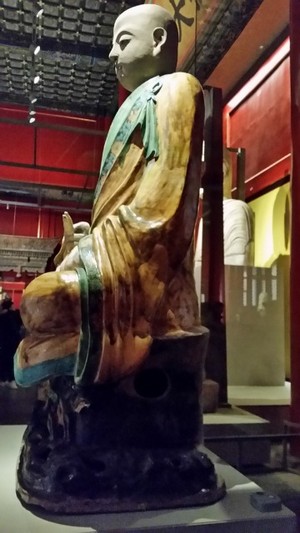}
\end{subfigure}
\begin{subfigure}[b]{\SrcImgW\linewidth}
\centering\includegraphics[trim=0 0 0 0, clip=true, height=\SrcImgH]{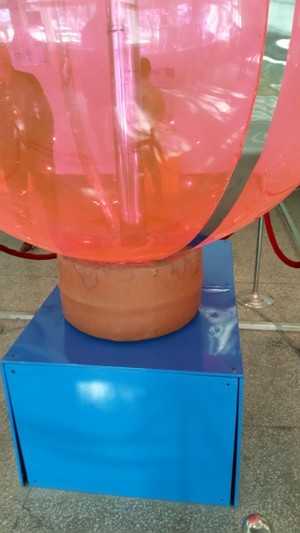}
\end{subfigure}
\begin{subfigure}[b]{\SrcImgW\linewidth}
\centering\includegraphics[trim=0 0 0 0, clip=true, height=\SrcImgH]{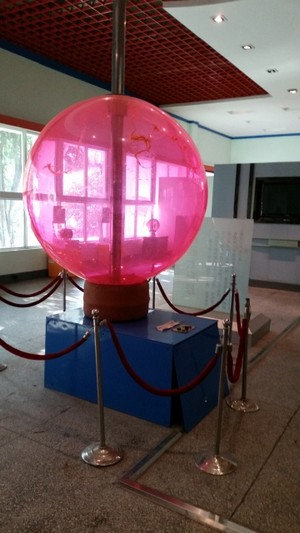}
\end{subfigure}
\begin{subfigure}[b]{\SrcImgW\linewidth}
\centering\includegraphics[trim=0 0 0 0, clip=true, height=\SrcImgH]{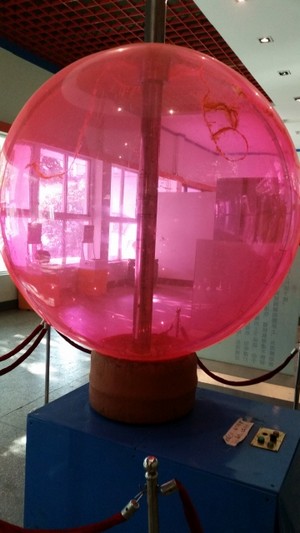}
\end{subfigure}
\begin{subfigure}[b]{\SrcImgW\linewidth}
\centering\includegraphics[trim=0 0 0 0, clip=true, height=\SrcImgH]{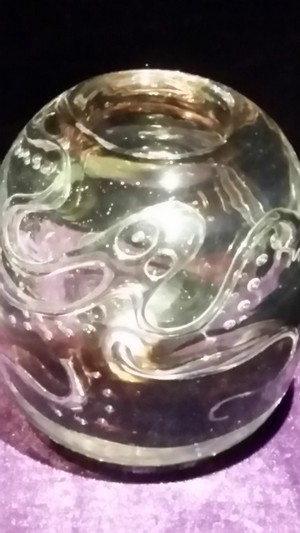}
\end{subfigure}
\begin{subfigure}[b]{\SrcImgW\linewidth}
\centering\includegraphics[trim=0 0 0 0, clip=true, height=\SrcImgH]{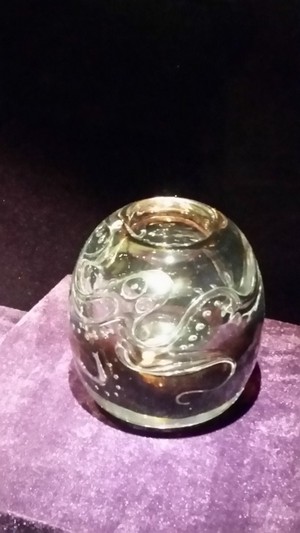}
\end{subfigure}
\begin{subfigure}[b]{\SrcImgW\linewidth}
\centering\includegraphics[trim=0 0 0 0, clip=true, height=\SrcImgH]{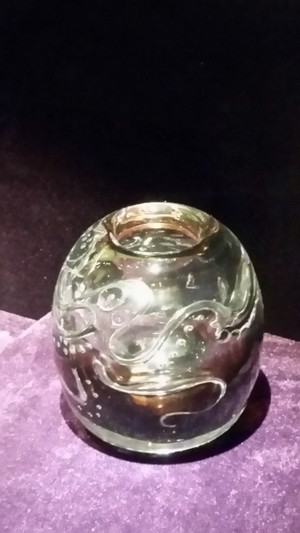}
\end{subfigure}\\
\begin{subfigure}[b]{\SrcImgW\linewidth}
\vspace{0.062cm}
\centering\includegraphics[trim=0 0 0 0, clip=true, height=\SrcImgH]{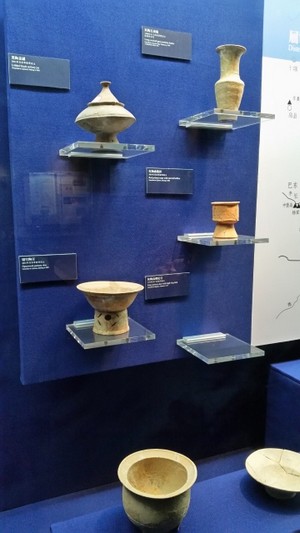}
\end{subfigure}
\begin{subfigure}[b]{\SrcImgW\linewidth}
\centering\includegraphics[trim=0 0 0 0, clip=true, height=\SrcImgH]{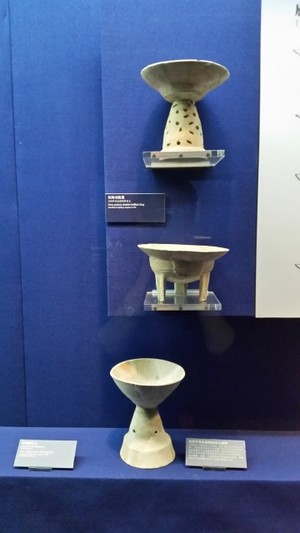}
\end{subfigure}
\begin{subfigure}[b]{\SrcImgW\linewidth}
\centering\includegraphics[trim=0 0 0 0, clip=true, height=\SrcImgH]{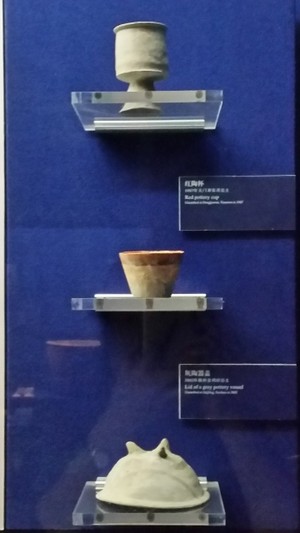}
\end{subfigure}
\begin{subfigure}[b]{\SrcImgW\linewidth}
\centering\includegraphics[trim=0 0 0 0, clip=true, height=\SrcImgH]{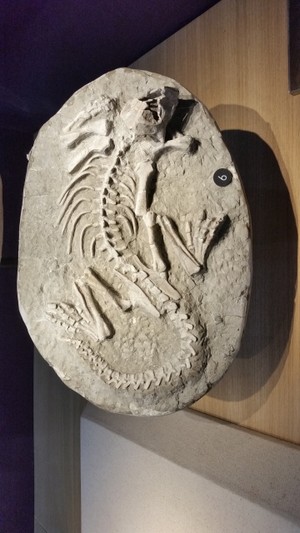}
\end{subfigure}
\begin{subfigure}[b]{\SrcImgW\linewidth}
\centering\includegraphics[trim=0 0 0 0, clip=true, height=\SrcImgH]{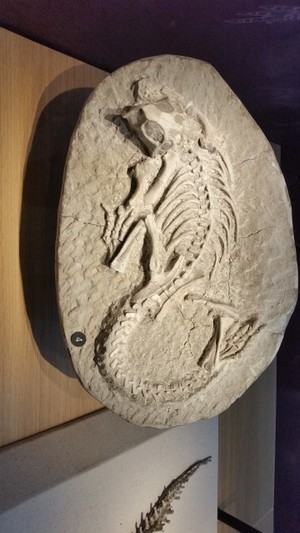}
\end{subfigure}
\begin{subfigure}[b]{\SrcImgW\linewidth}
\centering\includegraphics[trim=0 0 0 0, clip=true, height=\SrcImgH]{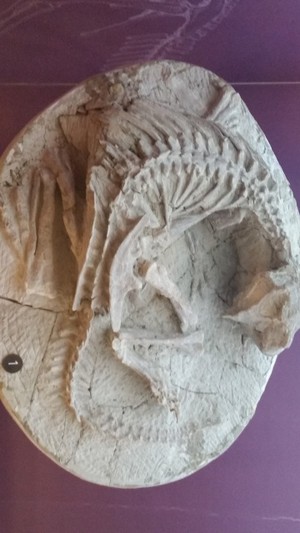}
\end{subfigure}
\begin{subfigure}[b]{\SrcImgW\linewidth}
\centering\includegraphics[trim=0 0 0 0, clip=true, height=\SrcImgH]{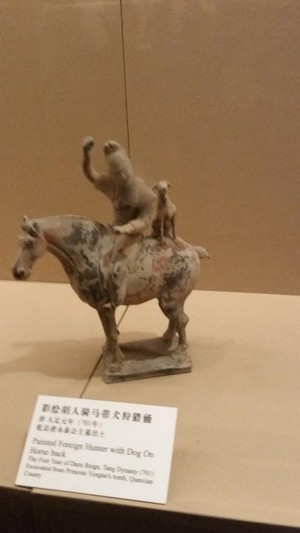}
\end{subfigure}
\begin{subfigure}[b]{\SrcImgW\linewidth}
\centering\includegraphics[trim=0 0 0 0, clip=true, height=\SrcImgH]{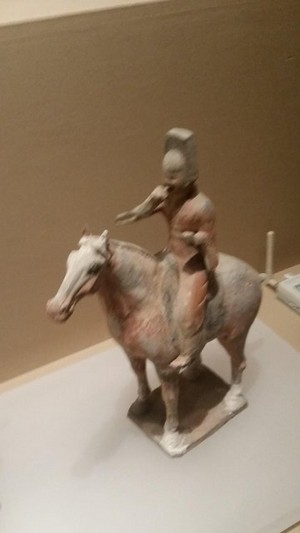}
\end{subfigure}
\begin{subfigure}[b]{\SrcImgW\linewidth}
\centering\includegraphics[trim=0 0 0 0, clip=true, height=\SrcImgH]{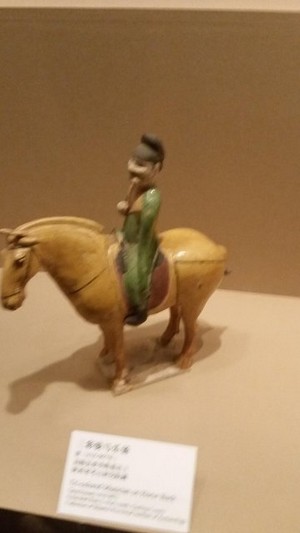}
\end{subfigure}
\begin{subfigure}[b]{\SrcImgW\linewidth}
\centering\includegraphics[trim=0 0 0 0, clip=true, height=\SrcImgH]{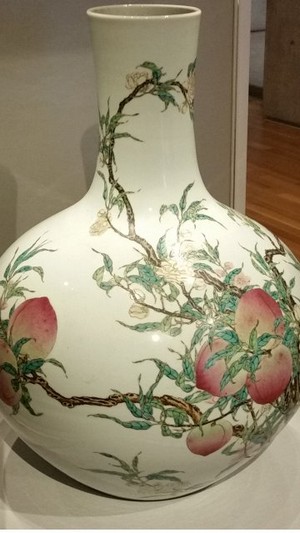}
\end{subfigure}
\begin{subfigure}[b]{\SrcImgW\linewidth}
\centering\includegraphics[trim=0 0 0 0, clip=true, height=\SrcImgH]{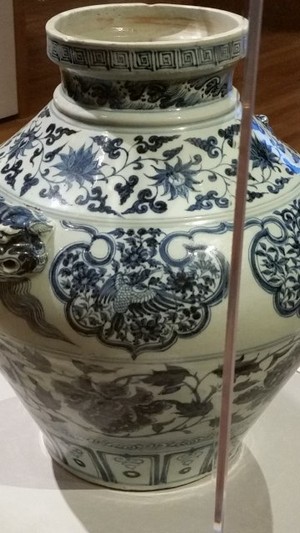}
\end{subfigure}
\begin{subfigure}[b]{\SrcImgW\linewidth}
\centering\includegraphics[trim=0 0 0 0, clip=true, height=\SrcImgH]{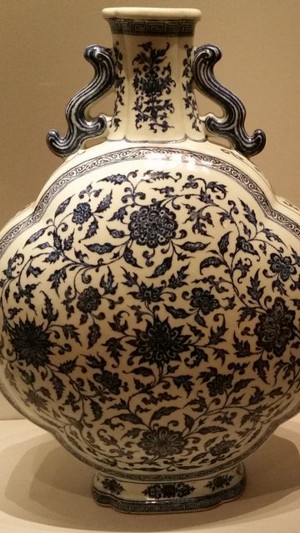}
\end{subfigure}
\begin{subfigure}[b]{\SrcImgW\linewidth}
\centering\includegraphics[trim=0 0 0 0, clip=true, height=\SrcImgH]{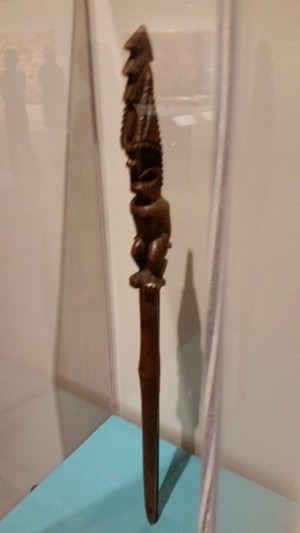}
\end{subfigure}
\begin{subfigure}[b]{\SrcImgW\linewidth}
\centering\includegraphics[trim=0 0 0 0, clip=true, height=\SrcImgH]{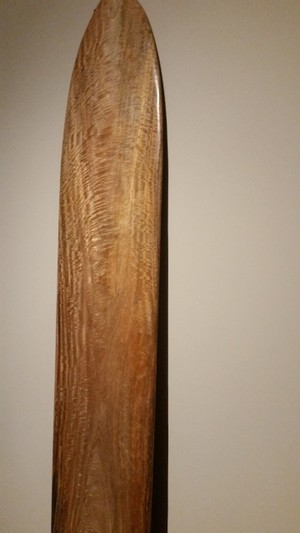}
\end{subfigure}
\begin{subfigure}[b]{\SrcImgW\linewidth}
\centering\includegraphics[trim=0 0 0 0, clip=true, height=\SrcImgH]{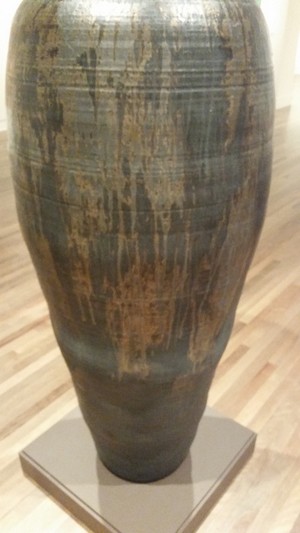}
\end{subfigure}

%}
%
%\includegraphics[trim=0 0 0 0, clip=true, height=4.65cm]{images/data3.jpg}
%
%\vspace{-0.7cm}
\caption{Examples of the source subsets of Open MIC. Top row includes Paintings ({\em Shn}), Clocks ({\em Shg}), Sculptures ({\em Scl}), Science Exhibits ({\em Sci}) and Glasswork ({\em Gls}). As 3 images per exhibit demonstrate, we covered different viewpoints and scales during capturing. Bottom row includes 3 different art pieces per exhibition such as Cultural Relics ({\em Rel}), Natural History Exhibits ({\em Nat}), Historical/Cultural Exhibits ({\em Shx}), Porcelain ({\em Clv}) and Indigenous Arts ({\em Hon}). Note the composite scenes of Relics, fine-grained nature of Natural History and Cultural Exhibits and non-planarity of exhibits.}\vspace{-0.45cm}
\label{fig:images_source}
\end{figure*}

\begin{proposition}
Let us choose $\!\mZ\!=\!\mX\!=[\mPhi\!,\mPhi^*\!]$ for pivots and source/target feature vectors, and kernel $k$ to be linear. Substitute these assumptions into Eq. \eqref{eq:nyst2}. As a result, we obtain $\mPi(\mX)\!=\!(\mZ^T\!\mZ)^{-0.5}\mZ^T\!\mX\!=\stkout{\mZ}\mX\!=\!(\mZ^T\!\mZ)^{0.5}\!=\!(\mX^T\!\mX)^{0.5}\!$ where $\mPi(\mX)$ is 
a projection of $\mX$ on itself that is isometric \eg, distances between column vectors of $(\mX^T\!\mX)^{0.5}$ correspond to distances of column vectors in $\mX$. Thus, $\mPi(\mX)$ is an isometric transformation w.r.t. distances in Table \ref{tab:non-euclid}, that is $d^2_g(\cov(\mPhi),\cov(\mPhi^*\!))\!=\!d^2_g(\cov(\mPi(\mPhi)),\cov(\mPi(\mPhi^*\!)))$.
%\MH{this doesn't seem like a theorem. I guess what you want to say is $\pi(X)$ is an isometric transformation according to the SAD distances.Please rephrase it, it is very hard to follow}
\label{prop:proj}
\end{proposition}
\begin{proof}
%\vspace{-0.2cm}
Firstly, we note that the following holds:%equalities hold:
\begin{align}
&  \!\!\!\!\!\!\mK_{\mX\mX}\!=\!\mPi(\mX)^T\!\mPi(\mX)\!=\!(\mX^T\!\mX)^{0.5}(\mX^T\!\mX)^{0.5}\!\!\!\!\!\!=\!\mX^T\!\mX.\!\!\!\label{eq:nyst3}
\end{align}
Note that $\mPi(\mX)\!=\!\stkout{\mZ}\mX$ projects $\mX$ into a more compact subspace of size $d'\!\!=\!N\!+\!N^*\!$ if $d'\!\ll\!d$ which includes the spanning space for $\mX$ by construction as $\mZ\!=\mX$. Eq. \eqref{eq:nyst3} implies that $\mPi(\mX)$ performs at most rotation on $\mX$ as the dot-product (used to obtain entries of $\mK_{\mX\mX}$)  just like the Euclidean distance is rotation-invariant only \eg, has no affine invariance. As spectra of $(\mX^T\!\mX)^{0.5}$ and $\mX$ are equal, this implies $\mPi(\mX)$ performs no scaling, shear or inverse. Distances in Table \ref{tab:non-euclid} are all rotation-invariant, thus
%$d^2_g(\cov,\cov^*\!)\!=\!d^2_g(\mPi(\cov),\mPi(\cov^*\!))$.
$d^2_g(\cov(\mPhi),\cov(\mPhi^*\!))\!=\!d^2_g(\cov(\mPi(\mPhi)),\cov(\mPi(\mPhi^*\!)))$.

A stricter proof is to show that $\stkout{\mZ}$ performs a composite rotation $\mV\mU^T$. Let us use SVD of $\mZ$ equal $\mU\mLambda\mV^T$. Then:
\begin{align}
& \stkout{\mZ}\!=\!(\mZ^T\!\mZ)^{-0.5}\mZ^T\!=\!(\mV\mLambda\mU^T\mU\mLambda\mV^T)^{-0.5}\,\mV\mLambda\mU^T\\
& \qquad\qquad\qquad\qquad\!=\mV\mLambda^{-1}\mV^T\mV\mLambda\mU^T\!\!=\!\mV\mU^T\nonumber\\[-30pt]\nonumber
\end{align}
\vspace{-0.2cm}
\end{proof}
In practice, for each class $c\!\in\!\idx{C}$, we choose $\mX\!=\!\mZ\!=[\mPhi_c, \mPhi_c^*]$. 
Then, as $\stkout{\mZ}[\mPhi, \mPhi^*\!]\!=\!(\mX^T\!\mX)^{0.5}$, we have $\mPi(\mPhi)\!=\![\vy_1,\cdots,\vy_N]$ and $\mPi(\mPhi^*\!)\!=\![\vy_{N\!+\!1},\cdots,\vy_{N\!+\!N*\!}]$ where $\mY\!=\![\vy_{1},\cdots,\vy_{N\!+\!N*\!}]\!=\!(\mX^T\!\mX)^{0.5}\!$. With typical $N\!\approx\!30$ and $N^*\!\approx\!3$, we obtain covariances of side size $\!d'\!\approx\!33$ rather than $d\!=\!4096$.

\footnotetext[1]{\label{foot:chain_oper}For simplicity of notation, operator $\odot$ denotes the typical summation over multiplications in chain rules.}
\footnotetext[2]{\label{foot:complexity1}We assume that the eigenvalue decomposition of large matrices ($d\!=\!4096$) in CUDA BLAS is fast and efficient--which is not the case.}

\begin{proposition}
Typically, the inverse square root $(\mX^T\!\mX)^{-0.5}$ of $\stkout{\mZ}(\mX)$ can be only differentiated via the costly eigenvalue decomposition.
However, if $\!\mX\!=[\mPhi\!,\mPhi^*\!]$, $\stkout{\mZ}(\mX)\!=\!(\mX^T\!\mX)^{-0.5}\mX^T$  and $\mPi(\mX)\!=\!\stkout{\mZ}(\mX)\mX$ as in Prop. \ref{prop:proj}, and if we consider the chain rule we require:
\begin{align}
&\textstyle\frac{\partial d^2_g(\cov(\mPi(\mPhi)),\cov(\mPi(\mPhi^*\!)))}{\partial \cov(\mPi(\mPhi))}\odot\frac{\partial \cov(\mPi(\mPhi))}{\partial \mPi(\mPhi)}\odot\frac{\partial \mPi(\mPhi)}{\partial \mPhi},\text{{\color{red}\footnotemark[1]}}
\label{eq:chain_rule_proj}
\end{align}
then $\stkout{\mZ}(\mX)$ can be treated as a constant in differentiation:
\begin{align}
&\textstyle\frac{\partial\mPi(\mX)}{\partial X_{mn}}\!=\!\frac{\partial\stkout{\mZ}(\mX)\mX}{\partial X_{mn}}\!=\!\stkout{\mZ}(\mX)\frac{\partial\mX}{\partial X_{mn}}\!=\!\stkout{\mZ}(\mX)\mJ_{mn}.
\label{eq:nyst_diff}
\end{align}
%Eq. \eqref{eq:nyst_diff} holds for the following chain rule used by us:
\end{proposition}
%
%For simplicity of notation, operator $\odot$ denotes the typical summation over multiplications in chain rules.
%
\begin{proof}
%\vspace{-0.2cm}
It follows from the rotation-invariance of the Euclidean, JBLD and AIRM distances. Let us write $\stkout{\mZ}(\mX)\!=\!\mR(\mX)\!=\!\mR$, where $\mR$ is a rotation matrix. Thus, we have: $d^2_g(\cov(\mPi(\mPhi)),\cov(\mPi(\mPhi^*\!)))\!=\!d^2_g(\cov(\mR\mPhi),\cov(\mR\mPhi^*\!))\!=\!d^2_g(\mR\cov(\mPhi)\mR^T\!,\mR\cov(\mPhi^*\!)\mR^T)$. Therefore, even if $\mR$ depends on $\mX$, the distance $d^2_g$ is unchanged by any choice of valid $\mR$ \ie, for the Frobenius norm we have:
%$||\mR\cov\!-\!\mR\cov^*\!||_F^2\!=\!\trace\left((\cov\!-\!\cov^*)^T\!\mR^T\!\mR(\cov\!-\!\cov^*)\right)\!=\!\trace\left((\cov\!-\!\cov^*)^T\!(\cov\!-\!\cov^*)\right)\!=\!||\cov\!-\!\cov^*\!||_F^2$
$||\mR\cov\mR^T\!-\!\mR\cov^*\!\mR^T||_F^2\!=\!\trace\left(\mR\mA^T\!{\mR}^T\!{\mR}\mA{\mR^T}\right)\!=\!\trace\left(\mR^T\!\mR\mA^T\!\mA\right)\!=\!\trace\left(\mA^T\!\mA\right)\!=\!||\cov\!-\!\cov^*||_F^2$, where $\mA\!=\!\cov\!-\!\cov^*\!$.
Therefore, we obtain:
$\frac{\partial ||\mR\cov(\mPhi)\mR^T\!\!-\!\mR\cov(\mPhi^*\!)\mR^T\!||_F^2}{\partial \mR\cov(\mPhi)\mR^T}\odot\frac{\partial \mR\cov(\mPhi)\mR^T}{\partial \cov(\mPhi)}\odot\frac{\partial \cov(\mPhi)}{\partial \mPhi}\!=\!
\frac{\partial ||\cov(\mPhi)\!-\!\cov(\mPhi^*\!)\!||_F^2}{\partial \cov(\mPhi)}\odot\frac{\partial \cov(\mPhi)}{\partial \mPhi}
$ {\color{red}\footnotemark[1]}%\footnote{\label{foot:chain_oper}For simplicity of notation, operator $\odot$ denotes the typical summation over multiplications in chain rules.} 
which completes the proof.
\end{proof}

\vspace{0.05cm}
\noindent{\textbf{Complexity.}} The Frobenius norm between covariances plus their computation have combined complexity $\bigoh((d'\!\!+\!1)d^2)$, where $d'\!\!=\!N\!+\!N^*\!$. For non-Euclidean distances, we take into account the dominant cost of evaluating the square root of matrix and/or inversions by the eigenvalue decomposition, as well as the cost of building scatter matrices. Thus, we have $\bigoh((d'\!\!+\!1)d^2 + d^\omega)$, where constant $2\!<\!\omega\!<\!2.376$ concerns complexity of eigenvalue decomposition. Lastly, evaluating the Nystr\" om projections combined with building covariances and running a non-Euclidean distance enjoys $\bigoh({d'}^2d + (d'\!\!+\!1){d'}^2 + {d'}^\omega)\!=\!\bigoh({d'}^2d)$ complexity for $d\!\gg\!d'\!$.

For typical $d'\!\!=\!33$ and $d\!=\!4096$,  the non-Euclidean distances are $~1.7\!\times$ slower{\color{red}\footnotemark[2]} than the Frobenius norm. However, non-Eucldiean distances combined with our projections are $210\!\times$ and $124\!\times$ faster than naively evaluated non-Eucldiean distances and the Frobenius norm, resp. This cuts the time of each training from few days to 6--8 hours and makes the cost of our loss negligible compared to CNN fine-tuning.

\newcommand{\RE}{\color{blue!20!black!30!red}}

\ifdefined\arxiv
%\fontsize{7.25}{9}\selectfont
\newcommand{\SSS}{{\em S}}
\newcommand{\TTT}{{\em T}}
\newcommand{\SPT}{{\em S+T}}
\newcommand{\SOO}{{\em So}}
\newcommand{\JBL}{{\em JBLD}}
\newcommand{\JBD}{{\em JBLD}}
\newcommand{\AIR}{{\em AIRM}}
\else
\newcommand{\SSS}{{\em S}}
\newcommand{\TTT}{{\em T}}
\newcommand{\SPT}{{\footnotesize\em S+T}}
\newcommand{\SOO}{{\em So}}
\newcommand{\JBL}{{\footnotesize\em JBLD}}
\newcommand{\JBD}{{\em JBLD}}
\newcommand{\AIR}{{\footnotesize\em AIRM}}
\fi

\newcommand{\BL}{\color{black!40!green}}
\newcommand{\BU}[1]{\color{black!40!green}\textbf{#1}}
\newcommand{\BO}[1]{\textbf{#1}}
\newcommand{\TKN}[2]{{\footnotesize top-{#1}-{#2}}}
\newcommand{\TKO}{\footnotesize\pbox{3cm}{$\avg_k$\\top-$k$-$k$}}
%\footnotesize \scriptscriptstyle
%\colorbox{blue!30}{\em Shn}

\newcommand{\Shn}{\multirow{5}{*}{\rotatebox[origin=c]{90}{{\color{blue!20!black!30!red}{\em Shn}}}\kern-0.3em }}
\newcommand{\Clk}{\multirow{5}{*}{\rotatebox[origin=c]{90}{{\color{blue!20!black!30!red}\em Clk}}\kern-0.3em }}
\newcommand{\Scl}{\multirow{5}{*}{\rotatebox[origin=c]{90}{{\color{blue!20!black!30!red}\em Scl}}\kern-0.3em }}
\newcommand{\Sci}{\multirow{5}{*}{\rotatebox[origin=c]{90}{{\color{blue!20!black!30!red}\em Sci}}\kern-0.3em }}
\newcommand{\Gls}{\multirow{5}{*}{\rotatebox[origin=c]{90}{{\color{blue!20!black!30!red}\em Gls}}\kern-0.3em }}

\newcommand{\Rel}{\multirow{5}{*}{\rotatebox[origin=c]{90}{{\color{blue!20!black!30!red}\em Rel}}\kern-0.3em }}
\newcommand{\Nat}{\multirow{5}{*}{\rotatebox[origin=c]{90}{{\color{blue!20!black!30!red}\em Nat}}\kern-0.3em }}
\newcommand{\Shx}{\multirow{5}{*}{\rotatebox[origin=c]{90}{{\color{blue!20!black!30!red}\em Shx}}\kern-0.3em }}
\newcommand{\Clv}{\multirow{5}{*}{\rotatebox[origin=c]{90}{{\color{blue!20!black!30!red}\em Clv}}\kern-0.3em }}
\newcommand{\Hon}{\multirow{5}{*}{\rotatebox[origin=c]{90}{{\color{blue!20!black!30!red}\em Hon}}\kern-0.3em }}

\ifdefined\arxiv
\newcommand{\SrcImgWW}{0.104}
\newcommand{\SrcImgWWW}{0.058}
\newcommand{\SrcImgHH}{1.08cm}
\else
\newcommand{\SrcImgWW}{0.104}
\newcommand{\SrcImgWWW}{0.058}
\newcommand{\SrcImgHH}{1.36cm}
\fi

\begin{figure*}[b]%htbp % left bottom right top
\centering%%%%\vspace{-0.3cm}
\begin{subfigure}[b]{\SrcImgWW\linewidth}
\centering\includegraphics[trim=0 0 0 0, clip=true, height=\SrcImgHH]{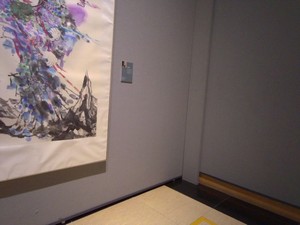}
\end{subfigure}
\begin{subfigure}[b]{\SrcImgWW\linewidth}
\centering\includegraphics[trim=0 0 0 0, clip=true, height=\SrcImgHH]{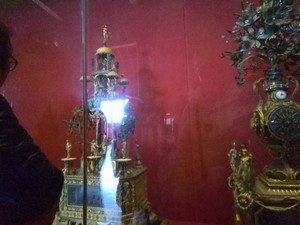}
\end{subfigure}
\begin{subfigure}[b]{\SrcImgWW\linewidth}
\centering\includegraphics[trim=0 0 0 0, clip=true, height=\SrcImgHH]{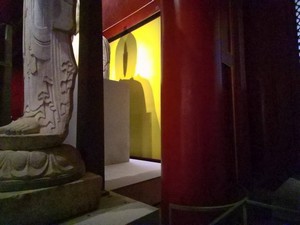}
\end{subfigure}
\begin{subfigure}[b]{\SrcImgWW\linewidth}
\centering\includegraphics[trim=0 0 0 0, clip=true, height=\SrcImgHH]{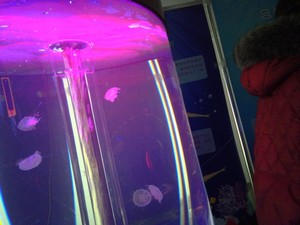}
\end{subfigure}
\begin{subfigure}[b]{\SrcImgWW\linewidth}
\centering\includegraphics[trim=0 0 0 0, clip=true, height=\SrcImgHH]{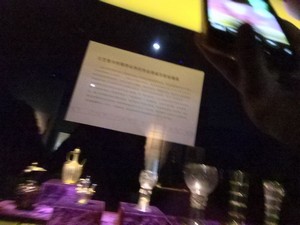}
\end{subfigure}
\begin{subfigure}[b]{\SrcImgWW\linewidth}
\centering\includegraphics[trim=0 0 0 0, clip=true, height=\SrcImgHH]{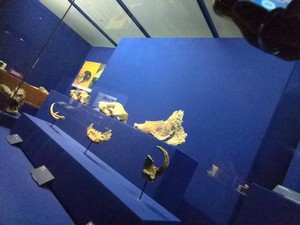}
\end{subfigure}
\begin{subfigure}[b]{\SrcImgWW\linewidth}
\centering\includegraphics[trim=0 0 0 0, clip=true, height=\SrcImgHH]{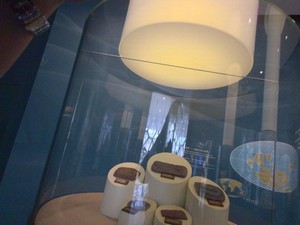}
\end{subfigure}
\begin{subfigure}[b]{\SrcImgWWW\linewidth}
\centering\includegraphics[trim=0 0 0 0, clip=true, height=\SrcImgHH]{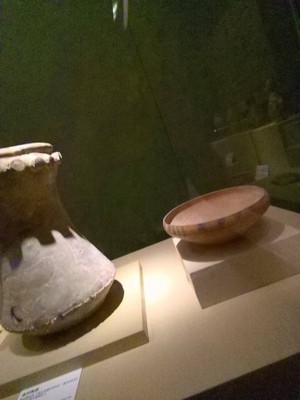}
\end{subfigure}
\begin{subfigure}[b]{\SrcImgWW\linewidth}
\centering\includegraphics[trim=0 0 0 0, clip=true, height=\SrcImgHH]{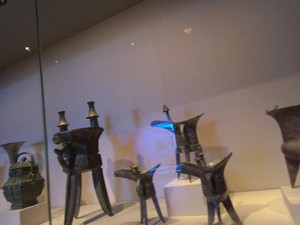}
\end{subfigure}
\begin{subfigure}[b]{\SrcImgWWW\linewidth}
\centering\includegraphics[trim=0 0 0 0, clip=true, height=\SrcImgHH]{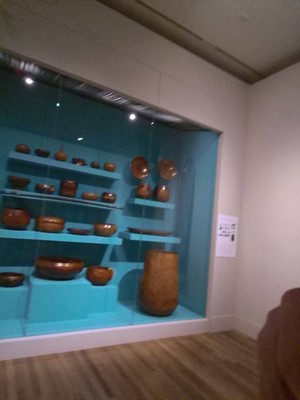}
\end{subfigure}
\\
\vspace{0.062cm}
\begin{subfigure}[b]{\SrcImgWW\linewidth}
\centering\includegraphics[trim=0 0 0 0, clip=true, height=\SrcImgHH]{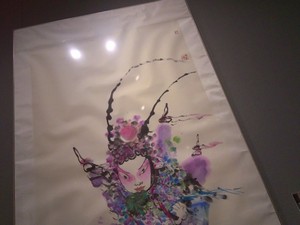}
\end{subfigure}
\begin{subfigure}[b]{\SrcImgWW\linewidth}
\centering\includegraphics[trim=0 0 0 0, clip=true, height=\SrcImgHH]{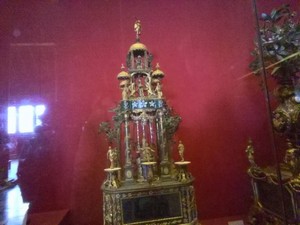}
\end{subfigure}
\begin{subfigure}[b]{\SrcImgWW\linewidth}
\centering\includegraphics[trim=0 0 0 0, clip=true, height=\SrcImgHH]{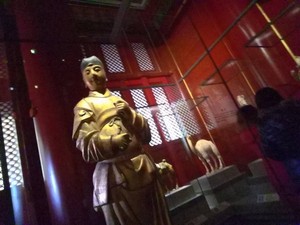}
\end{subfigure}
\begin{subfigure}[b]{\SrcImgWW\linewidth}
\centering\includegraphics[trim=0 0 0 0, clip=true, height=\SrcImgHH]{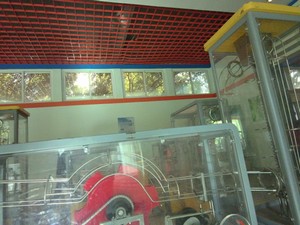}
\end{subfigure}
\begin{subfigure}[b]{\SrcImgWW\linewidth}
\centering\includegraphics[trim=0 0 0 0, clip=true, height=\SrcImgHH]{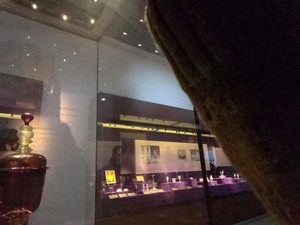}
\end{subfigure}
\begin{subfigure}[b]{\SrcImgWW\linewidth}
\centering\includegraphics[trim=0 0 0 0, clip=true, height=\SrcImgHH]{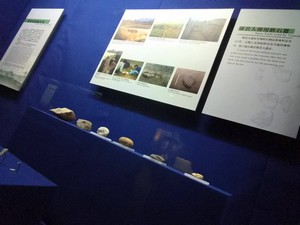}
\end{subfigure}
\begin{subfigure}[b]{\SrcImgWW\linewidth}
\centering\includegraphics[trim=0 0 0 0, clip=true, height=\SrcImgHH]{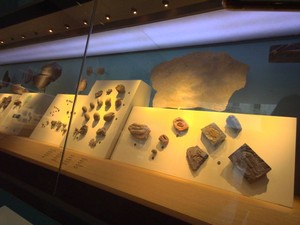}
\end{subfigure}
\begin{subfigure}[b]{\SrcImgWWW\linewidth}
\centering\includegraphics[trim=0 0 0 0, clip=true, height=\SrcImgHH]{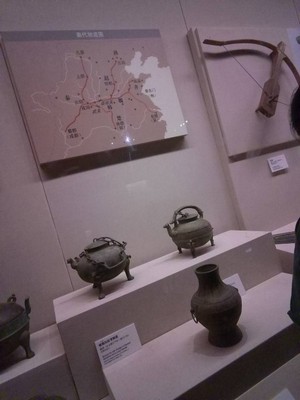}
\end{subfigure}
\begin{subfigure}[b]{\SrcImgWW\linewidth}
\centering\includegraphics[trim=0 0 0 0, clip=true, height=\SrcImgHH]{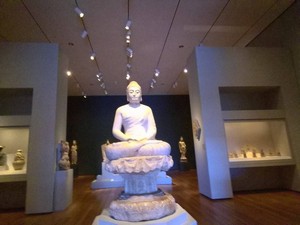}
\end{subfigure}
\begin{subfigure}[b]{\SrcImgWWW\linewidth}
\centering\includegraphics[trim=0 0 0 0, clip=true, height=\SrcImgHH]{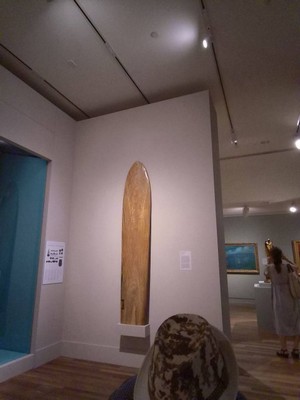}
\end{subfigure}
\\
\vspace{0.062cm}
\begin{subfigure}[b]{\SrcImgWW\linewidth}
\centering\includegraphics[trim=0 0 0 0, clip=true, height=\SrcImgHH]{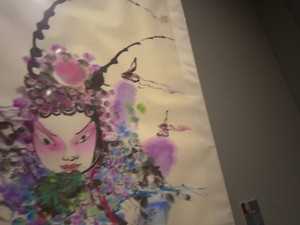}
\end{subfigure}
\begin{subfigure}[b]{\SrcImgWW\linewidth}
\centering\includegraphics[trim=0 0 0 0, clip=true, height=\SrcImgHH]{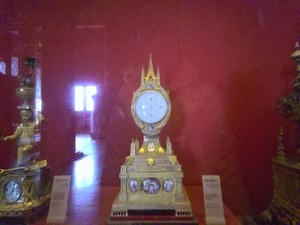}
\end{subfigure}
\begin{subfigure}[b]{\SrcImgWW\linewidth}
\centering\includegraphics[trim=0 0 0 0, clip=true, height=\SrcImgHH]{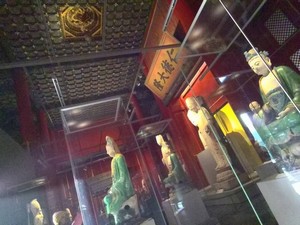}
\end{subfigure}
\begin{subfigure}[b]{\SrcImgWW\linewidth}
\centering\includegraphics[trim=0 0 0 0, clip=true, height=\SrcImgHH]{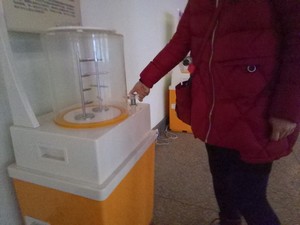}
\end{subfigure}
\begin{subfigure}[b]{\SrcImgWW\linewidth}
\centering\includegraphics[trim=0 0 0 0, clip=true, height=\SrcImgHH]{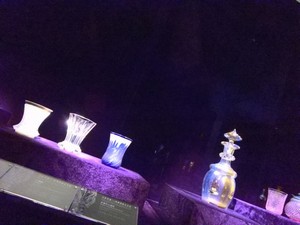}
\end{subfigure}
\begin{subfigure}[b]{\SrcImgWW\linewidth}
\centering\includegraphics[trim=0 0 0 0, clip=true, height=\SrcImgHH]{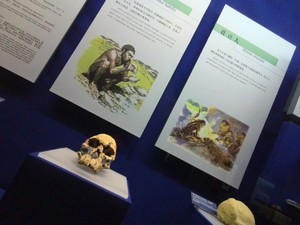}
\end{subfigure}
\begin{subfigure}[b]{\SrcImgWW\linewidth}
\centering\includegraphics[trim=0 0 0 0, clip=true, height=\SrcImgHH]{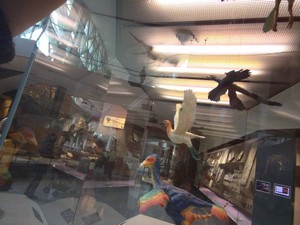}
\end{subfigure}
\begin{subfigure}[b]{\SrcImgWWW\linewidth}
\centering\includegraphics[trim=0 0 0 0, clip=true, height=\SrcImgHH]{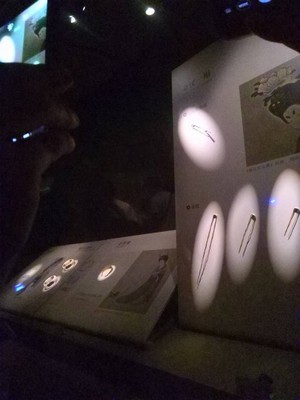}
\end{subfigure}
\begin{subfigure}[b]{\SrcImgWW\linewidth}
\centering\includegraphics[trim=0 0 0 0, clip=true, height=\SrcImgHH]{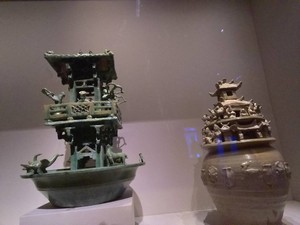}
\end{subfigure}
\begin{subfigure}[b]{\SrcImgWWW\linewidth}
\centering\includegraphics[trim=0 0 0 0, clip=true, height=\SrcImgHH]{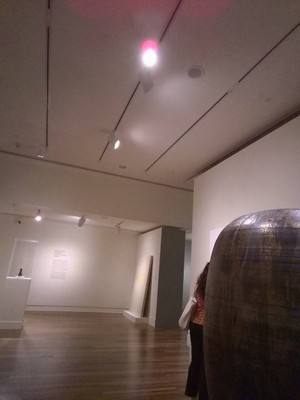}
\end{subfigure}
%
%\vspace{-0.7cm}
\caption{Examples of the target subsets of Open MIC. From left to right, each column illustrates Paintings ({\em Shn}), Clocks ({\em Shg}), Sculptures ({\em Scl}), Science Exhibits ({\em Sci}) and Glasswork ({\em Gls}), Cultural Relics ({\em Rel}), Natural History Exhibits ({\em Nat}), Historical/Cultural Exhibits ({\em Shx}), Porcelain ({\em Clv}) and Indigenous Arts ({\em Hon}). Note the variety of photometric and geometric distortions due to the use of wearable cameras.
}\vspace{-0.45cm}
\label{fig:images_target}
\end{figure*}

\section{Experiments}
\label{sec:expts}

In this section, we explain our CNN setup and give more details about our Open MIC and present our evaluations.

\vspace{0.05cm}
\noindent{\textbf{Setting.}}  At the training and testing time, we use the setting shown in Figures \ref{fig:cnn1} and \ref{fig:cnn2}, respectively. The images in our dataset are portrait or landscape oriented. Therefore, we extract 3 square patches per image that cover its entire region. For training, these patches serve as training data points. For testing, we average over 3 predictions from a group of patches to label image. %Moreover, we use mainly the VGG16 \cite{simonyan_vgg} streams.
We briefly compare the VGG16 \cite{simonyan_vgg} and GoogLeNet networks \cite{google_net} as well as the Eucldiean, JBLD and AIRM distances on subsets of the Office and Open MIC dataset. As demonstrated in Table \ref{tab:baseline_check}, VGG16 and GoogLeNet yield similar scores while JBLD and AIRM beat the Euclidean distance. Thus, we employ the VGG16 model and the JBLD distance in what follows.

\vspace{0.05cm}
\noindent{\textbf{Parameters.}} The networks are pre-trained on the ImageNet dataset~\cite{ILSVRC15} for the best results. We set non-zero learning rates on the fully-connected and the last two convolutional layers of the two streams. Subsequently, fine-tuning on the source and target data takes between 30--100K iterations. We set $\tau$ to the average value of the $\ell_2$ norm of {\em fc} feature vectors sampled on ImageNet and 
%. For simplicity, $\alpha_1\!=\!0$ and $\alpha_2\!=\!0$ while 
the hyperplane proximity $\eta\!=\!1$. Inverse in $\stkout{\mZ}(\mX)\!=\!(\mX^T\mX)^{-0.5}\mX^T$ and matrices $\cov$ and $\cov^*\!$ are regularized with a small constant 1e-6 on diagonals. Lastly, we set $\sigma_1$ and $\sigma_2$ between 0.005--1 %a small subset of 
to perform cross-validation.

\vspace{0.05cm}
\noindent{\textbf{Office.}}  %A popular dataset for evaluating algorithms against the effect of domain shift is the Office dataset~\cite{saenko_office} which contains 31 object categories that form three domains:
This dataset contains three domains: Amazon, DSLR and Webcam. The Amazon and Webcam domains contain 2817 and 795 images. For brevity, we first test our pipeline on the Amazon-Webcam domain shift (\dsAW) to ensure that we match results in the literature.

%\vspace{0.05cm}
\noindent{\textbf{Open MIC.}}
The proposed dataset contains 10 distinct source-target subsets of images from 10 different kinds of museum exhibition spaces which are illustrated in Figures \ref{fig:images_source} and \ref{fig:images_target}, respectively. They include Paintings from Shenzhen Museum ({\em Shn}), the Clock and Watch Gallery ({\em Clk}) and the Indian and Chinese Sculptures ({\em Scl}) from the Palace Museum, the Xiangyang Science Museum ({\em Sci}), the European Glass Art ({\em Gls}) and the Collection of Cultural Relics ({\em Rel}) from the Hubei Provincial Museum, the Nature, Animals and Plants in Ancient Times ({\em Nat}) from Shanghai Natural History Museum, the Comprehensive Historical and Cultural Exhibits from Shaanxi History Museum ({\em Shx}), the Sculptures, Pottery and Bronze Figurines from the Cleveland Museum of Arts ({\em Clv}), and Indigenous Arts from Honolulu Museum Of Arts ({\em Hon}).

\begin{table}[t]
\vspace{-0.3cm}
\setlength{\tabcolsep}{0.15em}
\centering
%\hspace{-0.75cm}
%
%\parbox{0.71\textwidth}
\renewcommand{\arraystretch}{0.8}
{
\centering
\begin{tabular}{c|c c c c c c c c c c | c}
				& {\RE\em Shn}   & {\RE\em Clk}   & {\RE\em Scl}   & {\RE\em Sci}   & {\RE\em Gls}   & {\RE\em Rel} & {\RE\em Nat} & {\RE\em Shx} & {\RE\em Clv} & {\RE\em Hon} & Total \\
\hline
\kern-0.4em{\em Inst.}\kern-0.1em & 79 & 113 & 41 & 37 & 98 & 100 & 111 & 166 & 81 & 40 & 866\\
\kern-0.4em{\em Src+}\kern-0.1em & 566 & 413 & 225 & 637 & 601 & 775 & 763 & 2928 & 531 & 1121 & 8560\\
\kern-0.4em{\em Src.}\kern-0.1em & 417 & 650 & 160 & 391 & 575 & 587 & 695 & 2697 & 503 & 970 & 7645\\
\kern-0.4em{\em Tgt+}\kern-0.1em & 515 & 323 & 130 & 1692 & 964 & 1229 & 868 & 776 & 682 & 417 & 7596 \\
\kern-0.4em{\em Tgt.}\kern-0.1em & 404 & 305 & 112 & 1342 & 863 & 863 & 668 & 546 & 625 & 364 & 6092
\end{tabular}
}
%\vspace{0.2cm}
\caption{Unique exhibit instances ({\em Inst.}) and numbers of images of Open MIC in the source ({\em Src.}) and target ({\em Tgt.}) subsets including their backgrounds ({\em Src+}) and ({\em Tgt+}).
\label{tab:museum_stats}
\vspace{-0.3cm}
}
\end{table}

For the target data, we annotated each image with labels of art pieces visible in it. The wearable cameras were set to capture an image every 10s and they operated {\em in-the-wild}, \eg, volunteers had no control over shutter, focus, centering, \etc. Therefore, the collected target subsets exhibit many realistic challenges, \eg, sensor noises, motion blur, occlusions, background clutter, varying viewpoints, scale changes, rotations, glares, transparency, non-planar surfaces, clipping, multiple exhibits, active light, color inconstancy, very large or small exhibits, to name but a few phenomena visible in Figure \ref{fig:images_target}. The numbers and statistics regarding the Open MIC dataset are given in  Table \ref{tab:museum_stats}. Every subset contains 37--166 exhibits to identify and 5 train, val., and test splits. In total, our dataset contains 866 unique exhibit labels, 8560 source (7645 exhibits and 915 backgrounds) and 7596 target (6092 exhibits and 1504 backgrounds including a few of unidentified exhibits) images.

\vspace{0.05cm}
\noindent{\textbf{Baselines.}}
To demonstrate the intrinsic difficulty of the Open MIC dataset, we provide the community with baseline accuracies obtained from (i) fine-tuning CNNs on the source subsets ({\em S}) and testing on the randomly chosen target splits, (ii) fine tuning on target only ({\em T}) and evaluating on remaining disjoint target splits, (iii) fine-tuning on the source+target ({\em S+T}) and evaluating on remaining disjoint target splits, (iv) training state-of-the-art domain adaptation So-HoT algorithm \cite{me_domain} equipped by us with non-Euclidean distances \cite{anoop_logdet,PEN06,bhatia_pdm} to enable robust end-to-end learning. % as detailed below.

We include evaluation protocols: (i) training/eval. per exhibition subset, (ii) training/testing on the combined set with all 866 identity labels, (iii) %testing w.r.t. the following scene factors annotated by us: low lighting, motion blur, occlusions, background clutter, large viewpoint angle, large zoom, large rotation, glares, object reflections, clipping.
testing w.r.t. scene factors annotated by us and detailed in Section \ref{sec:open_mic} (Challenge III).

\ifdefined\arxiv
\newcommand{\SrcTabAW}{0.25}
\newcommand{\SrcTabBW}{0.29}
\else
\newcommand{\SrcTabAW}{0.21}
\newcommand{\SrcTabBW}{0.25}
\fi

\begin{table}[t]
\vspace{-0.3cm}
\setlength{\tabcolsep}{0.12em}
\centering
%\hspace{-0.75cm}
%
\parbox{\SrcTabAW\textwidth}{
\renewcommand{\arraystretch}{0.80}
{
\centering
\begin{tabular}{c|c c}
%& {\em S+T}   & {\em So}   & {\em JBLD}   & {\em AIRM} \\
	& \multirow{2}{*}{VGG16} & GoogLe\\
	&   & Net\\
\hline
%VGG16 & 88.66 & 89.45 & \textbf{90.80} & 90.72 \\
%GoogLeNet & 88.92 & 89.70 & \textbf{91.33} & 91.20 \\
%
\kern-0.3em{\em S+T} & 88.66 & 88.92\\
\kern-0.3em{\em So} & 89.45 & 89.70\\
\kern-0.3em{\em JBLD} & \textbf{90.80} & \textbf{91.33}\\
\kern-0.3em{\em AIRM} & 90.72 & 91.20
\end{tabular}
}
}
%
%\hspace{0.1cm}
\parbox{\SrcTabBW\textwidth}{
\renewcommand{\arraystretch}{0.70}
{
\setlength{\tabcolsep}{0.1em}
\centering
\begin{tabular}{c c | c }
\hline
\kern-0.6em DLID & \cite{chopra_icml_workshop} & 51.9\\
\kern-0.6em DeCAF\textsubscript{6} S+T & \cite{donahue_decaf} & 80.7\\
\kern-0.6em DaNN & \cite{dann_com} & 53.6\\
\kern-0.6em Source CNN & \kern-0.0em\cite{tzeng_transfer}\kern-0.0em & 56.5\\
\kern-0.6em Target CNN & \kern-0.0em\cite{tzeng_transfer}\kern-0.0em & 80.5\\
\kern-0.6em Source+Target CNN\kern-0.6em  & \kern+0.0em\cite{tzeng_transfer}\kern-0.0em & 82.5\\
\kern-0.6em Dom. Conf.+Soft Labs.\kern-0.2em & \kern+0.0em\cite{tzeng_transfer}\kern-0.0em	& 82.7\\
\hline
\end{tabular}
}
}
\caption{The Office dataset (\dsAD~domain shift). (Left) Results on the VGG16 and GoogLeNet streams for the baseline fine-tuning on the combined source+target domains ({\em S+T}) and second-order ({\em So}) Euclidean-based method \cite{me_domain} are compared to our JBLD/AIRM dist. (Right) Comparisons to the state of the art.
\label{tab:baseline_check}}
\ifdefined\arxiv\vspace{-0.1cm}\else\vspace{-0.3cm}\fi
\end{table}

%The Office dataset on VGG streams. (Top) \dsAW~and (Bottom) \dsAD~domain shifts are evaluated on second-order ({\em So}), second- ({\em So}+$\vzeta$) and third-order+weights ({\em To}+$\vzeta$), second- and third- ({\em So+To}+$\vzeta$) and fourth-order ({\em So+To+Fo}+$\vzeta$) alignment with weight learning. Our baseline fine-tuning on the combined source and target domains ({\em S+T}) is also evaluated for comparison.

\subsection{Comparison to the State of the Art}
\label{sec:valid_setup}

%In order to obtain valid baselines on our Open MIC dataset, 
Firstly, we validate that our reference method performs on the par or better than the state-of-the-art approaches. Table \ref{tab:baseline_check} shows that the JBLD and AIRM distances outperform the Euclidean-based So-HoT method ({\em So}) \cite{me_domain} by $\sim\!1.6\%$ and other recent approaches \eg, \cite{tzeng_transfer} by $\sim\!8.6\%$ accuracy. We also observe that GoogLeNet outperforms the VGG16-based model by $\sim\!0.5\%$. Having validated our model, we opt to evaluate our proposed Open MIC dataset on VGG16 streams for consistency with the So-HoT model \cite{me_domain}.

\begin{table}[t]
\ifdefined\arxiv\vspace{-0.1cm}\else\vspace{-0.3cm}\fi
\setlength{\tabcolsep}{0.14em}
\centering
%\hspace{-0.75cm}
%
%\parbox{0.71\textwidth}
\renewcommand{\arraystretch}{0.8}
{
\centering
\begin{tabular}{c|c c c c c | c c | c c | c |}
				                         & sp1   & sp2   & sp3   & sp4   & sp5   & top-1 & \TKN{1}{5} & top-5 & \TKN{5}{5} & \TKO\\
\hline
\kern-0.4em{\em S}\kern-0.1em    & 33.9 & 34.2 & 34.8 & 34.2 & 33.8 & 34.2 & 36.0 & 49.2 & 53.7 & 46.0 \\
\kern-0.4em{\em T}\kern-0.1em    & 56.9 & 55.9 & 58.7 & 56.0 & 55.2 & 56.5 & 64.1 & 76.5 & 80.6 & 72.5 \\
\kern-0.4em{\em S+T}\kern-0.1em  & 56.4 & 55.2 & 57.1 & 56.3 & 54.4 & 55.9 & 62.5 & 75.8 & 79.2 & 71.6 \\
\kern-0.4em{\em So} \kern-0.1em  & 64.2 & 62.4 & 65.0 & 62.7 & 60.0 & 62.8 & 70.4 & 84.0 & 88.5 & 79.5 \\
\kern-0.4em\BL\JBL\kern-0.1em    & \BU{65.7} & \BU{63.8} & \BU{65.7} & \BU{63.7} & \BU{62.0} & \BU{64.2} & \BU{72.0} & \BU{85.7} & \BU{88.6} & \BU{80.8} \\
\end{tabular}
}
%\vspace{0.2cm}
\caption{Challenge II. Open MIC performance on the combined set for data 5 splits. Baselines ({\em S}), ({\em T}) and ({\em S+T}) are given as well as second-order ({\em So}) method \cite{me_domain} and our JBLD approach.
\label{tab:museum_all}
}
\vspace{-0.3cm}
\end{table}

\begin{table*}[!b]
\vspace{-0.1cm}
\setlength{\tabcolsep}{0.12em}
\renewcommand{\arraystretch}{0.8}
{
\ifdefined\arxiv\fontsize{7.25}{9}\selectfont\else\fi
\centering
\begin{tabular}{c|c c c c c | c c c c c c c | c c c c c | c c c c c | c c c c c | }
			            &      & \SSS & \TTT & \SPT & \BL\JBL &      & \SSS & \TTT & \SPT & \SOO & \BL\JBL & \AIR &      & \SSS & \TTT & \SPT & \BL\JBL &      & \SSS & \TTT & \SPT & \BL\JBL &      & \SSS & \TTT & \SPT & \BL\JBL \\
\hline
\kern-0.3em sp1   & \Shn & 45.3 & 45.3 & 59.0 & \BU{60.0} & \Clk & 55.8 & 51.9 & 55.8 & 55.8 & \BU{57.7} & 57.2 & \Scl & 56.5 & 60.9 & 65.2 & \BU{65.2} & \Sci & 59.3 & 58.9 & 65.6 & \BU{65.8} & \Gls & 64.1 & 67.1 & 62.8 & \BU{70.3} \\
\kern-0.3em sp2   &      & 48.4 & 52.6 & 53.7 & \BU{62.1} &      & 55.4 & 44.6 & 50.0 & 58.9 & \BU{58.9} & 58.9 &      & 44.4 & 50.0 & 44.4 & \BU{50.0} &      & 56.9 & 57.2 & 67.1 & \BU{69.1} &      & 59.9 & 61.9 & 59.2 & \BU{63.9} \\
\kern-0.3em sp3   &      & 46.1 & 52.7 & 60.4 & \BU{64.8} &      & 58.9 & 58.9 & 67.9 & 69.6 & \BU{71.4} & 71.4 &      & 55.6 & 38.9 & 44.4 & \BU{44.4} &      & 69.9 & 62.0 & 65.7 & \BU{68.2} &      & 65.9 & 69.3 & 64.9 & \BU{69.6} \\
\kern-0.3em sp4   &      & 49.5 & 50.5 & 54.8 & \BU{64.5} &      & 51.9 & 48.1 & 46.1 & 53.8 & \BU{57.7} & 57.7 &      & 55.0 & 55.0 & 55.0 & \BU{50.0} &      & 58.1 & 59.2 & 64.2 & \BU{66.3} &      & 62.3 & 67.0 & 61.6 & \BU{68.7} \\
\kern-0.3em sp5   &      & 49.5 & 57.0 & 63.4 & \BU{69.9} &      & 62.5 & 41.7 & 60.4 & 58.3 & \BU{60.4} & 60.4 &      & 56.2 & 56.2 & 62.5 & \BU{62.5} &      & 57.3 & 53.3 & 61.5 & \BU{64.5} &      & 60.1 & 64.5 & 59.0 & \BU{65.2} \\
%\cline{1-1}\cline{3-28}
\hline
\kern-0.3em top-1 &      & 47.7 & 51.6 & 58.3 & \BU{64.3} &    & 56.9 & 49.1 & 56.0 & 59.3 & \BU{61.2} & 61.1 &    & 53.5 & 52.2 & 54.3 & \BU{54.4} &    & 58.5 & 58.1 & 64.9 & \BU{66.8} &    & 62.5 & 65.9 & 61.6 &\BU{67.5}\\
\kern-0.3em\TKN{1}{5}&   & 48.2 & 54.2 & 60.2 & \BU{66.4} &    & 58.9 & 56.3 & 60.3 & 66.2 & \BU{68.9} & 68.9 &    & 54.7 & 55.4 & 57.3 & \BU{58.4} &    & 60.2 & 61.7 & 67.8 & \BU{70.2} &    & 77.3 & 84.4 & 76.7 &\BU{84.9}\\
\hline
\kern-0.3em top-5 &      & 64.5 & 68.8 & 76.9 & \BU{81.6} &    & 76.7 & 63.8 & 78.2 &\BO{87.5}&\BL{86.9}&87.2 &    & 67.4 & 66.6 & 70.0 & \BU{70.0} &    & 83.3 & 82.7 & 86.0 & \BU{88.6} &    & 85.2 & 89.4 & 83.1 &\BU{89.3}\\
\kern-0.3em\TKN{5}{5}&   & 66.0 & 73.3 & 79.5 & \BU{84.2} &    & 77.8 & 75.0 & 82.7 &\BO{91.6}&\BL{91.0}&91.4 &    & 69.4 & 69.8 & 71.1 & \BU{72.0} &    & 85.6 & 86.3 & 89.4 & \BU{91.3} &    & 87.3 & 95.0 & 89.7 &\BU{93.4}\\
\hline
\kern-0.3em\TKO   &      & 59.0 & 63.4 & 71.0 & \BU{76.6} &    & 69.4 & 65.6 & 73.6 &\BO{81.5}&\BL{81.2}&81.4 &    & 63.7 & 62.5 & 65.1 & \BU{65.1} &    & 75.3 & 76.0 & 80.7 & \BU{82.5} &    & 78.4 & 86.2 & 80.7 &\BU{86.2}\\
\end{tabular}\\
\vspace{0.062cm}
\ifdefined\arxiv\else\hspace{0.015cm}\fi
\begin{tabular}{c|c c c c c | c c c c c c c | c c c c c | c c c c c | c c c c c |}
			            &      & \SSS & \TTT & \SPT & \BL\JBL &      & \SSS & \TTT & \SPT & \SOO & \BL\JBL & \AIR &      & \SSS & \TTT & \SPT & \BL\JBL &      & \SSS & \TTT & \SPT & \BL\JBL &      & \SSS & \TTT & \SPT & \BL\JBL \\
\hline
\kern-0.3em sp1   & \Rel & 62.0 & 65.0 & 63.3 & \BU{66.3} & \Nat & 38.0 & 56.2 & 52.6 & 58.8 & \BU{58.8} & 58.5 & \Shx & 33.3 & 43.2 & 31.5 & \BU{58.6} & \Clv & 47.4 & 65.8 & 66.2 & \BU{71.4} & \Hon & 65.6 & 71.1 & 70.3 & \BU{75.8} \\
\kern-0.3em sp2   &      & 60.9 & 65.7 & 63.0 & \BU{68.0} &      & 39.9 & 52.5 & 52.5 & 59.6 & \BU{59.6} & 59.6 &      & 31.8 & 39.8 & 27.4 & \BU{47.8} &      & 47.0 & 70.2 & 65.1 & \BU{72.2} &      & 63.9 & 67.2 & 70.5 & \BU{74.6} \\
\kern-0.3em sp3   &      & 64.1 & 70.4 & 67.4 & \BU{70.7} &      & 43.7 & 56.2 & 59.4 & 59.9 & \BU{59.9} & 59.9 &      & 25.7 & 47.7 & 31.2 & \BU{47.7} &      & 49.7 & 64.1 & 61.5 & \BU{67.7} &      & 68.5 & 70.2 & 71.8 & \BU{79.0} \\
\kern-0.3em sp4   &      & 61.0 & 68.5 & 62.8 & \BU{67.1} &      & 41.8 & 59.8 & 62.0 & 66.3 & \BU{67.9} & 67.4 &      & 33.0 & 38.8 & 26.2 & \BU{44.7} &      & 48.3 & 63.0 & 64.0 & \BU{68.5} &      & 67.8 & 63.6 & 79.3 & \BU{76.9} \\
\kern-0.3em sp5   &      & 55.4 & 61.0 & 59.3 & \BU{62.6} &      & 44.6 & 62.0 & 63.0 & 66.8 & \BU{67.4} & 66.8 &      & 25.7 & 35.8 & 28.4 & \BU{44.0} &      & 42.3 & 62.8 & 54.1 & \BU{65.8} &      & 67.5 & 65.8 & 75.0 & \BU{80.0} \\
\hline
\kern-0.3em top-1 &      & 60.7 & 66.1 & 63.2 & \BU{67.0} &    & 41.6 & 57.3 & 57.9 & 62.2 & \BU{62.7} & 62.5 &    & 29.9 & 41.1 & 29.0 & \BU{48.5} &    & 47.0 & 65.2 & 62.2 & \BU{69.1} &    & 66.7 & 67.6 & 73.4 &\BU{77.3}\\
\kern-0.3em\TKN{1}{5}&   & 70.1 & 76.8 & 73.2 & \BU{79.5} &    & 43.5 & 62.8 & 61.9 & 67.3 & \BU{67.7} & 67.5 &    & 31.5 & 47.7 & 31.9 & \BU{56.3} &    & 50.8 & 69.5 & 66.6 & \BU{73.9} &    & 70.2 & 70.3 & 76.3 &\BU{79.7}\\
\hline
\kern-0.3em top-5 &      & 82.0 & 87.1 & 85.8 & \BU{90.3} &    & 60.6 & 79.3 & 75.5 &\BO{84.6}&\BL{84.3}& 84.3 &   & 51.6 & 62.5 & 51.2 & \BU{75.0} &    & 65.3 & 84.3 & 79.9 & \BU{87.7} &    & 82.1 & 85.2 & 88.3 &\BU{90.0}\\
\kern-0.3em\TKN{5}{5}&   & 86.3 & 90.0 & 89.4 & \BU{93.7} &    & 65.3 & 82.8 & 80.1 &\BO{87.5}&\BL{87.0}& 86.9 &   & 54.9 & 67.3 & 54.8 & \BU{77.6} &    & 70.5 & 89.2 & 84.4 & \BU{91.0} &    & 88.1 & 88.8 & 91.7 &\BU{92.7}\\
\hline
\kern-0.3em\TKO   &      & 77.4 & 82.8 & 80.5 & \BU{85.2} &    & 55.7 & 74.0 & 72.4 & 79.5 & \BU{79.6} & 79.4 &    & 45.1 & 57.1 & 44.5 & \BU{66.8} &    & 61.5 & 80.6 & 76.5 & \BU{83.5} &    & 79.7 & 81.0 & 84.5 &\BU{86.7}\\
\end{tabular}
}
\caption{Challenge I. Open MIC performance on the 10 subsets for data 5 splits. Baselines ({\em S}), ({\em T}) and ({\em S+T}) are given as well as our JBLD approach. We report top-1, top-1-5, top-5-1, top-5-5 accuracies and the combined scores $\avg_k\text{top-}k\text{-}k$. See Section\ref{sec:open_mic} for details. 
\label{tab:challengeI}}
\vspace{-0.3cm}
\end{table*}

\subsection{Open MIC Challenge}
\label{sec:open_mic}

In what follows, we detail our challenges on the Open MIC dataset and present our experimental results.

\vspace{0.05cm}
\noindent{\textbf{Challenge I.}} For this challenge, we run our supervised domain adaptation algorithm combined with the JBLD distance per subset. We prepare 5 training, validation and testing splits. For the source data, we use all available samples per class. %(but no more than 50 images). 
For the target data, we use $~$3 samples per class for training and validation, respectively, and the rest for testing. %We provide the splits we use for the community. 

We report top-1 and top-5 accuracies. Moreover, as our target images often contain multiple exhibits, we ask a question whether any of top-$k$ predictions match any of top-$n$ image labels ordered by our expert volunteers according to the perceived saliency. If so, we count it as a correctly recognized image. We count these valid predictions and normalize by the total number of testing images. We denote this measure as top-$k$-$n$ where $k,n\!\in\!\idx{5}$. Lastly, we indicate an {\em area-under-curve} type of measure $\avg_k\text{top-}k\text{-}k$ which rewards correct recognition of the most dominant object in the scene and offers some leniency if the order of top predictions is confused and/or if they match less dominant objects--a simple alternative to precision/recall plots. %This is This is a simple alternative to integrating the area under precision/recall plots.

We divided Open MIC into {\em Shn}, {\em Clk}, {\em Scl}, {\em Sci}, {\em Gls}, {\em Rel}, {\em Nat}, {\em Shx}, {\em Clv} and {\em Hon} subsets to allow short 6--8 hours long runs per experiment. We ran 150 jobs on ({\em S}), ({\em T}) and ({\em S+T}) baselines and 300 jobs on JBLD: 5 splits $\times$10 subsets $\times$6 hyperp. choices. Table \ref{tab:challengeI} shows that the exhibits in the Comprehensive Historical and Cultural Exhibits ({\em Shx}) and the Sculptures ({\em Scl}) were the hardest to identify given scores of $48.5$ and $54.4\%$ top-1 accuracy. This is consistent with volunteers' reports that both exhibitions were crowded, the lighting was dim, exhibits were occluded, fine-grained and non-planar. The easiest to identify were the
Sculptures, Pottery and Bronze Figurines ({\em Clv}) and the Indigenous Arts ({\em Hon}) as both exhibitions were spacious with good lighting. The average top-1 accuracy across all subsets on JBLD is $63.9\%$. Averages over baselines ({\em S}), ({\em T}) and ({\em S+T}) are $52.5$, $57.4$, and $58.5\%$ top-1 acc. To account for uncertainty of saliency-based labeling and classifier confusing which exhibit to label, we report our proposed average top-1-5 acc. to be $70.6\%$. Our average combined score $\avg_k\text{top-}k\text{-}k$ is $79.3\%$. These results show that Open MIC challenges CNNs due to {\em in-the-wild} capture with wearable cameras.

\noindent{\textbf{Challenge II.}} Having provided the above results per subset, we evaluate the combined set covering 866 exhibit identities. In this setting, a single experiment runs 80--120 hours. We ran 15 jobs on ({\em S}), ({\em T}) and ({\em S+T}) baselines and 60 jobs on ({\em So}) and JBLD: 2 distances $\times$5 splits $\times$6  hyperp. choices. %parameter choices for cross-validation. 
Table \ref{tab:museum_all} shows that our JBLD approach scores $64.2\%$ top-1 accuracy and outperforms baselines ({\em S}), ({\em T}) and ({\em S+T}) by $30$, $7.7$ and $8.3\%$. Fine-tuning CNNs on the source and testing on target ({\em S}) is especially a very poor performer due to the significant domain shift in Open MIC.

\noindent{\textbf{Challenge III.}} For this challenge, we break down performance on the combined set covering 866 exhibit identities w.r.t. the following 12 factors:
object clipping ({\em clp}), low lighting ({\em lgt}), blur ({\em blr}), light glares ({\em glr}), background clutter ({\em bgr}), occlusions ({\em ocl}), in-plane rotations ({\em rot}), zoom ({\em zom}), tilted viewpoint ({\em vpc}), small size/far away ({\em sml}), object shadows ({\em shd}), reflections ({\em rfl}) and the clean view ({\em ok}). 
Table \ref{tab:museum_III} shows results averaged over 5 data splits. We note that JBLD outperforms baselines. The factors most affecting the supervised domain adaptation are the small size ({\em sml}) of exhibits/distant view, low light ({\em lgt}) and blur ({\em blr}). The corresponding top-1 accuracies of $34.1$, $48.6$ and $51.6\%$ are below our average top-1 accuracy of $64.2\%$ listed in Table \ref{tab:museum_all}. %Therefore, we encourage algorithms that will target the above factors in order to obtain better domain adaptation algorithms. 
In contrast, images with shadows ({\em shd}), zoom ({\em zom}) and reflections ({\em rfl}) score $70.4$, $70.0$ and $67.5\%$ top-1 accuracy (above avg. $64.2\%$). Our wearable cameras captured also a few of clean shots scoring $81.0\%$ top-1 accuracy. This lets us form a claim that domain adaptation methods should evolve to deal with each of these adverse factors. 
Our suppl. material presents further statistics including the numbers of images per factor, top-$k$-$n$ accuracies, breakdowns per subset and even analysis of combined factors \eg, images that contain small exhibits under poor light and blur score only $30.0\%$ average top-1 accuracy.

\newcommand{\BV}[1]{\color{black!10!blue}{\em #1}}

\ifdefined\arxiv
\newcommand{\ParBoxAW}{0.70}
\else
\newcommand{\ParBoxAW}{0.48}
\fi

\begin{table}[t]
\vspace{-0.3cm}
\setlength{\tabcolsep}{0.10em}
\centering
\ifdefined\arxiv\else\hspace{-0.9cm}\fi
\parbox{\ParBoxAW\textwidth}
{
\renewcommand{\arraystretch}{0.8}
{
\centering
\begin{tabular}{c|c c c c c  c c c c c  c c c |}
				                         & \BV{clp}   & \BV{lgt}  & \BV{blr}   & \BV{glr}  & \BV{bgr}   & \BV{ocl} & \BV{rot} & \BV{zom} & \BV{vpc} & \BV{sml} & \BV{shd} & \BV{rfl} & \BV{ok}\\
\hline
\kern-0.4em{\em S}\kern-0.1em    & 41.4  & 17.0  & 23.8 & 27.3 & 40.3 & 34.5 & 29.7  & 52.7  & 33.4  & 14.2 & 10.4 & 32.3 & 65.5 \\
\kern-0.4em{\em T}\kern-0.1em    & 56.2  & 38.2  & 42.6 & 56.1 & 57.9 & 49.6 & 58.3  & 60.4  & 50.3  & 29.6 & 59.2 & 60.7 & 64.3 \\
\kern-0.4em{\em S+T}\kern-0.1em  & 56.6  & 34.6  & 39.8 & 54.9 & 56.2 & 48.3 & 56.7  & 65.9  & 48.7  & 27.3 & 56.5 & 59.0 & 72.6 \\
\kern-0.4em\BL\JBL\kern-0.1em    &\BU{65.3}&\BU{48.6}&\BU{51.6}&\BU{64.0}&\BU{65.9}&\BU{56.4}&\BU{65.0}&\BU{70.0}&\BU{58.6}&\BU{34.1}&\BU{70.4}&\BU{67.5}&\BU{81.0}\\
\end{tabular}
}
}
%\vspace{0.2cm}
\caption{Challenge III. Open MIC performance on the combined set w.r.t. 12 factors detailed in Section \ref{sec:open_mic}. Top-1 accuracies for baselines ({\em S}), ({\em T}), ({\em S+T}), and for our JBLD approach are listed.
\label{tab:museum_III}
}
\vspace{-0.5cm}
\end{table}
\section{Conclusions}
\label{sec:conclude}

We have collected, annotated and evaluated a new challenging Open MIC dataset with the source and target domains formed by images from Android and wearable cameras, respectively. We covered 10 distinct exhibition spaces in 10 different museums to collect a realistic {\em in-the-wild} target data in contrast to typical photos for which the users control the shutter. We have provided a number of useful baselines \eg, breakdowns of results per exhibition, combined scores and analysis of factors detrimental to domain adaptation and recognition. Unsupervised domain adaptation and few-shot learning methods can also be compared to our baselines. 
Moreover, we proposed orthogonal improvements to the supervised domain adaptation \eg, we integrated non-trivial non-Euclidean distances and Nystr\" om projections for better results and tractability. We will make our data and evaluation scripts available to the researchers.
\begin{appendices}

\ifdefined\arxiv
\else
\newcommand{\RE}{\color{blue!20!black!30!red}}

\newcommand{\SSS}{{\em S}}
\newcommand{\TTT}{{\em T}}
\newcommand{\SPT}{{\footnotesize\em S+T}}
\newcommand{\SOO}{{\em So}}
\newcommand{\JBL}{{\footnotesize\em JBLD}}
\newcommand{\JBD}{{\em JBLD}}
\newcommand{\AIR}{{\footnotesize\em AIRM}}

\newcommand{\BL}{\color{black!40!green}}
\newcommand{\BU}[1]{\color{black!40!green}\textbf{#1}}
\newcommand{\BO}[1]{\textbf{#1}}
\newcommand{\TKN}[2]{{\footnotesize top-{#1}-{#2}}}
\newcommand{\TKO}{\footnotesize\pbox{3cm}{$\avg_k$\\top-$k$-$k$}}
%\footnotesize \scriptscriptstyle
%\colorbox{blue!30}{\em Shn}

\newcommand{\Shn}{\multirow{5}{*}{\rotatebox[origin=c]{90}{{\color{blue!20!black!30!red}{\em Shn}}}\kern-0.3em }}
\newcommand{\Clk}{\multirow{5}{*}{\rotatebox[origin=c]{90}{{\color{blue!20!black!30!red}\em Clk}}\kern-0.3em }}
\newcommand{\Scl}{\multirow{5}{*}{\rotatebox[origin=c]{90}{{\color{blue!20!black!30!red}\em Scl}}\kern-0.3em }}
\newcommand{\Sci}{\multirow{5}{*}{\rotatebox[origin=c]{90}{{\color{blue!20!black!30!red}\em Sci}}\kern-0.3em }}
\newcommand{\Gls}{\multirow{5}{*}{\rotatebox[origin=c]{90}{{\color{blue!20!black!30!red}\em Gls}}\kern-0.3em }}

\newcommand{\Rel}{\multirow{5}{*}{\rotatebox[origin=c]{90}{{\color{blue!20!black!30!red}\em Rel}}\kern-0.3em }}
\newcommand{\Nat}{\multirow{5}{*}{\rotatebox[origin=c]{90}{{\color{blue!20!black!30!red}\em Nat}}\kern-0.3em }}
\newcommand{\Shx}{\multirow{5}{*}{\rotatebox[origin=c]{90}{{\color{blue!20!black!30!red}\em Shx}}\kern-0.3em }}
\newcommand{\Clv}{\multirow{5}{*}{\rotatebox[origin=c]{90}{{\color{blue!20!black!30!red}\em Clv}}\kern-0.3em }}
\newcommand{\Hon}{\multirow{5}{*}{\rotatebox[origin=c]{90}{{\color{blue!20!black!30!red}\em Hon}}\kern-0.3em }}

\newcommand{\BV}[1]{\color{black!10!blue}{\em #1}}

\fi

\newcommand{\BW}[1]{\color{black!40!green}#1}

\section{Additional Results on the OpenMIC dataset}
\label{sec:open_mic_det}

\ifdefined\arxiv
\newcommand{\PBoxAW}{5.75cm}
\newcommand{\VSpaceA}{0.1cm}
\begin{table*}[h]
\vspace{-0.3cm}
\setlength{\tabcolsep}{0.15em}
\centering
\fontsize{7.25}{9}\selectfont
%\hspace{-0.75cm}
%
\parbox{0.49\textwidth}
{
\renewcommand{\arraystretch}{0.8}
{
%\small
\centering
%\makebox[\linewidth]{
\begin{tabular}{c|l|}
abbr. & $\quad$details \\
\hline
\BV{clp} & \pbox{\PBoxAW}{\vspace{\VSpaceA}object clipping \eg, side, base or top including small or large fragments of an exhibit\vspace{\VSpaceA}}\\
\BV{lgt} & \pbox{\PBoxAW}{\vspace{\VSpaceA}poor lighting \eg, dark exhibition space, dark exhibit casing, strong light sources to which camera adapted leaving exhibit underexposed\vspace{\VSpaceA}}\\
\BV{blr} & \pbox{\PBoxAW}{\vspace{\VSpaceA}blur due to motion and/or poor lighting/long shutter exposure; full blur or part of the exhibit affected\vspace{\VSpaceA}}\\
\BV{glr} & \pbox{\PBoxAW}{\vspace{\VSpaceA}point-wise glares of light reflected from objects\vspace{\VSpaceA}}\\
\BV{bgr} & \pbox{\PBoxAW}{\vspace{\VSpaceA}background clutter: a non-uniform background behind an exhibit that changes with the camera viewpoint \eg, people, other exhibits, furniture \etc\vspace{\VSpaceA}}\\
\BV{ocl} & \pbox{\PBoxAW}{\vspace{\VSpaceA}side, frontal, large or partial exhibit occlusions due to humans, other objects or non-transparent protective casing\vspace{\VSpaceA}}\\
\BV{rot} & \pbox{\PBoxAW}{\vspace{\VSpaceA}in-plane rotations by more than 5 degrees due to a tilted camera or volunteers leaning towards exhibits\vspace{\VSpaceA}}\\
\BV{zom} & \pbox{\PBoxAW}{\vspace{\VSpaceA}large close-ups of an exhibit or a zoom of a part of exhibit\vspace{\VSpaceA}}\\
\BV{vpc} & \pbox{\PBoxAW}{\vspace{\VSpaceA}camera viewpoint that mismatches the normal to the surface of face of an exhibit--some exhibits have no frontal face, some have several faces due to their distinct axes of symmetry\vspace{\VSpaceA}}\\
\BV{sml} & \pbox{\PBoxAW}{\vspace{\VSpaceA}small object: an exhibit captured at a large distance \eg, across a hall; also small scale exhibits which cannot be closely approached\vspace{\VSpaceA}}\\
\BV{shd} & \pbox{\PBoxAW}{\vspace{\VSpaceA}a shadow cast over part of an exhibit\vspace{\VSpaceA}}\\
\BV{rfl} & \pbox{\PBoxAW}{\vspace{\VSpaceA}reflections affecting surfaces such as a protective glass casing of exhibits which acts like a mirror\vspace{\VSpaceA}}\\
\BV{ok} & \pbox{\PBoxAW}{\vspace{\VSpaceA}no visible distortions listed above\vspace{\VSpaceA}}%\vspace{0.3cm}}\\
\end{tabular}
%}
%
%
}
%\vspace{0.2cm}
\caption{Challenge III. The 12 factors w.r.t. which we evaluate our dataset.}
\label{tab:museum_stats}
}
\vspace{0.2cm}
\parbox{0.49\textwidth}
{
\renewcommand{\arraystretch}{0.8}
{
%\small
\centering
%\makebox[\linewidth]{
\begin{tabular}{c|l|}
abbr. & $\quad$details \\
\hline
\BV{lcl} & \pbox{\PBoxAW}{\vspace{\VSpaceA}light object clipping \eg, side, base or top including small fragments below 20\% of the exhibit area\vspace{\VSpaceA}}\\
\BV{hcl} & \pbox{\PBoxAW}{\vspace{\VSpaceA}heavy object clipping of large fragments \eg, more than 20\% of the exhibit area\vspace{\VSpaceA}}\\
\BV{bcl} & \pbox{\PBoxAW}{\vspace{\VSpaceA}clipping of the base of sculptures/exhibits \etc\vspace{\VSpaceA}}\\
\BV{scl} & \pbox{\PBoxAW}{\vspace{\VSpaceA}side occlusions of exhibits by humans or other objects\vspace{\VSpaceA}}\\
\BV{fcl} & \pbox{\PBoxAW}{\vspace{\VSpaceA}frontal/central occlusions of exhibit by humans or other objects\vspace{\VSpaceA}}\\
\BV{ooc} & \pbox{\PBoxAW}{\vspace{\VSpaceA}unclassified kind of occlusion\vspace{\VSpaceA}}\\
\BV{lzo} & \pbox{\PBoxAW}{\vspace{\VSpaceA}close-ups of an exhibit\vspace{\VSpaceA}}\\
\BV{hzo} & \pbox{\PBoxAW}{\vspace{\VSpaceA}large close-ups or a heavy zoom on a part of exhibit\vspace{\VSpaceA}}\\
\BV{lro} & \pbox{\PBoxAW}{\vspace{\VSpaceA}small in-plane rotations by no more than 15 degrees due to a tilted camera \etc.\vspace{\VSpaceA}}\\
\BV{hro} & \pbox{\PBoxAW}{\vspace{\VSpaceA}large in-plane rotations by more than 15 degrees due to a tilted camera \etc.\vspace{\VSpaceA}}\\
\BV{lvp} & \pbox{\PBoxAW}{\vspace{\VSpaceA}mismatches by less than 15 degrees between the camera viewpoint and the normal to the surface of face of an exhibit\vspace{\VSpaceA}}\\
\BV{hvp} & \pbox{\PBoxAW}{\vspace{\VSpaceA}mismatches by more than 15 degrees between the camera viewpoint and the normal to the surface of face of an exhibit\vspace{\VSpaceA}}\\
\BV{spc} & \pbox{\PBoxAW}{\vspace{\VSpaceA}light specularities and other reflections from surface\vspace{\VSpaceA}}
\end{tabular}
}
\caption{Challenge III. Additional factors w.r.t. which we evaluate our dataset.}
\label{tab:museum_stats_add}
}
%\vspace{-0.3cm}
\end{table*}

\else

\begin{table}[t]
\vspace{-0.3cm}
\setlength{\tabcolsep}{0.15em}
\centering
%\hspace{-0.75cm}
%
%\parbox{0.71\textwidth}
\renewcommand{\arraystretch}{0.8}
{
%\small
\centering
%\makebox[\linewidth]{
\begin{tabular}{c|l|}
abbr. & $\quad$details \\
\hline
\BV{clp} & \pbox{7.5cm}{\vspace{0.125cm}object clipping \eg, side, base or top including small or large fragments of an exhibit\vspace{0.125cm}}\\
\BV{lgt} & \pbox{7.5cm}{\vspace{0.125cm}poor lighting \eg, dark exhibition space, dark exhibit casing, strong light sources to which camera adapted leaving exhibit underexposed\vspace{0.125cm}}\\
\BV{blr} & \pbox{7.5cm}{\vspace{0.125cm}blur due to motion and/or poor lighting/long shutter exposure; full blur or part of the exhibit affected\vspace{0.125cm}}\\
\BV{glr} & \pbox{7.5cm}{\vspace{0.125cm}point-wise glares of light reflected from objects\vspace{0.125cm}}\\
\BV{bgr} & \pbox{7.5cm}{\vspace{0.125cm}background clutter: a non-uniform background behind an exhibit that changes with the camera viewpoint \eg, people, other exhibits, furniture \etc\vspace{0.125cm}}\\
\BV{ocl} & \pbox{7.5cm}{\vspace{0.125cm}side, frontal, large or partial exhibit occlusions due to humans, other objects or non-transparent protective casing\vspace{0.125cm}}\\
\BV{rot} & \pbox{7.5cm}{\vspace{0.125cm}in-plane rotations by more than 5 degrees due to a tilted camera or volunteers leaning towards exhibits\vspace{0.125cm}}\\
\BV{zom} & \pbox{7.5cm}{\vspace{0.125cm}large close-ups of an exhibit or a zoom of a part of exhibit\vspace{0.125cm}}\\
\BV{vpc} & \pbox{7.5cm}{\vspace{0.125cm}camera viewpoint that mismatches the normal to the surface of face of an exhibit--some exhibits have no frontal face, some have several faces due to their distinct axes of symmetry\vspace{0.125cm}}\\
\BV{sml} & \pbox{7.5cm}{\vspace{0.125cm}small object: an exhibit captured at a large distance \eg, across a hall; also small scale exhibits which cannot be closely approached\vspace{0.125cm}}\\
\BV{shd} & \pbox{7.5cm}{\vspace{0.125cm}a shadow cast over part of an exhibit\vspace{0.125cm}}\\
\BV{rfl} & \pbox{7.5cm}{\vspace{0.125cm}reflections affecting surfaces such as a protective glass casing of exhibits which acts like a mirror\vspace{0.125cm}}\\
\BV{ok} & \pbox{7.5cm}{\vspace{0.125cm}no visible distortions listed above\vspace{0.125cm}}%\vspace{0.3cm}}\\
\end{tabular}
%}
%
%
}
%\vspace{0.2cm}
\caption{Challenge III. The 12 factors w.r.t. which we evaluate our dataset.}
\label{tab:museum_stats}
\vspace{-0.3cm}
\end{table}
\fi

\newcommand{\absval}[1]{\ifnum#1<0 -\fi#1}
\newcommand{\signval}[1]{\ifnum#1<0 -\fi1}

\newcommand\ctr[2]{\pgfmathsetmacro{\compB}{0.5+(#1-#2)/(#1>#2?#1:#2)}\color{black}\edef\x{\noexpand\cellcolor[hsb]{0,\compB,1}}\x#1}

\newcommand\ctb[2]{\pgfmathsetmacro{\compB}{#1/#2}\color{black}\edef\x{\noexpand\cellcolor[hsb]{0.95,\compB,1}}\x#1}

\newcommand\ctc[2]{\pgfmathsetmacro{\compB}{#1/#2}\color{black}\edef\x{\noexpand\cellcolor[hsb]{0.7,\compB,1}}\x#1}

\newcommand\ctd[2]{\pgfmathsetmacro{\compB}{\ifnum#1>#2 0.5+(#1-#2)/#1/2\else 0.5+(#1-#2)/#2/2\fi}\color{black}\edef\x{\noexpand\cellcolor[hsb]{0.55,\compB,1}}\x#1}

\begin{table}[t]
\vspace{-0.3cm}
\setlength{\tabcolsep}{0.10em}
\centering
%\hspace{-0.9cm}
%
%\parbox{0.48\textwidth}
%{
%\renewcommand{\arraystretch}{0.8}
%{
\small
\centering
\begin{tabular}{c|c c c c c  c c c c c  c c c |}
$\cap$                                &\BV{clp}&\BV{lgt}&\BV{blr}&\BV{glr}&\BV{bgr}&\BV{ocl}&\BV{rot}&\BV{zom}&\BV{vpc}&\BV{sml}&\BV{shd}&\BV{rfl}&\BV{ok}\\
\hline
\kern-0.4em{\em all}\kern-0.1em       & \ctc{5136}{7344} & \ctc{335}{7344} & \ctc{1728}{7344} & \ctc{1346}{7344} & \ctc{2290}{7344} & \ctc{1529}{7344} & \ctc{7344}{7344} & \ctc{2278}{7344} & \ctc{4571}{7344} & \ctc{557}{7344} & \ctc{125}{7344} & \ctc{2000}{7344} & \ctc{84}{7344} \\
\hline
\kern-0.4em\BV{clp}\kern-0.1em        & \ctd{5136}{5136} & \ctd{216}{335} & \ctd{770}{1728} & \ctd{572}{1346} & \ctd{1415}{2290}& \ctd{873}{1529} & \ctd{3401}{7344} & \ctd{1803}{2278}& \ctd{2549}{4571}  & \ctd{167}{557}  & \ctd{66}{125} & \ctd{1009}{2000}   & 0     \\
\kern-0.4em\BV{lgt}\kern-0.1em        & \ctd{216}{5136}  & \ctd{335}{335} & \ctd{105}{1728} & \ctd{55}{1346}  & \ctd{92}{2290}  & \ctd{69}{1529}  & \ctd{232}{7344}  & \ctd{9}{2278}   & \ctd{234}{4571}  & \ctd{16}{557}   & \ctd{38}{125} & \ctd{21}{2000}     & 0     \\
\kern-0.4em\BV{blr}\kern-0.1em        & \ctd{770}{5136}  & \ctd{105}{335} & \ctd{1728}{1728}& \ctd{240}{1346} & \ctd{323}{2290} & \ctd{235}{1529} & \ctd{1348}{7344} & \ctd{240}{2278} & \ctd{820}{4571}  & \ctd{152}{557}  & \ctd{23}{125} & \ctd{330}{2000}    & 0     \\
\kern-0.4em\BV{glr}\kern-0.1em        & \ctd{572}{5136}  & \ctd{55}{335}  & \ctd{240}{1728} & \ctd{1346}{1346}& \ctd{183}{2290} & \ctd{143}{1529} & \ctd{1054}{7344} & \ctd{204}{2278} & \ctd{640}{4571}  & \ctd{52}{557}   & \ctd{12}{125} & \ctd{155}{2000}    & 0     \\
\kern-0.4em\BV{bgr}\kern-0.1em        & \ctd{1415}{5136} & \ctd{92}{335}  & \ctd{323}{1728} & \ctd{183}{1346} & \ctd{2290}{2290}& \ctd{565}{1529} & \ctd{1604}{7344} & \ctd{464}{2278} & \ctd{1409}{4571}  & \ctd{227}{557}  & \ctd{49}{125} & \ctd{395}{2000}    & 0     \\
\kern-0.4em\BV{ocl}\kern-0.1em        & \ctd{873}{5136}  & \ctd{69}{335}  & \ctd{235}{1728} & \ctd{143}{1346} & \ctd{565}{2290} & \ctd{1529}{1529}& \ctd{1090}{7344} & \ctd{183}{2278} & \ctd{978}{4571}  & \ctd{253}{557}  & \ctd{33}{125} & \ctd{219}{2000}    & 0     \\
\kern-0.4em\BV{rot}\kern-0.1em        & \ctd{3401}{5136} & \ctd{232}{335} & \ctd{1348}{1728}& \ctd{1054}{1346}& \ctd{1604}{2290}& \ctd{1090}{1529}& \ctd{7344}{7344} & \ctd{1380}{2278}& \ctd{3292}{4571}  & \ctd{405}{557}  & \ctd{113}{125}& \ctd{1522}{2000}   & 0     \\
\kern-0.4em\BV{zom}\kern-0.1em        & \ctd{1803}{5136} & \ctd{9}{335}   & \ctd{240}{1728} & \ctd{204}{1346} & \ctd{464}{2290} & \ctd{183}{1529} & \ctd{1380}{7344} & \ctd{2278}{2278}& \ctd{611}{4571}  & \ctd{0}{557}    & \ctd{18}{125} & \ctd{535}{2000}    & 0     \\
\kern-0.4em\BV{vpc}\kern-0.1em        & \ctd{2549}{5136} & \ctd{234}{335} & \ctd{820}{1728} & \ctd{640}{1346} & \ctd{1409}{2290}& \ctd{978}{1529} & \ctd{3292}{7344} & \ctd{611}{2278} & \ctd{4571}{4571}  & \ctd{370}{557}  & \ctd{39}{125} & \ctd{856}{2000}    & 0     \\
\kern-0.4em\BV{sml}\kern-0.1em        & \ctd{167}{5136}  & \ctd{16}{335}  & \ctd{152}{1728} & \ctd{52}{1346}  & \ctd{227}{2290} & \ctd{253}{1529} & \ctd{405}{7344}  & \ctd{0}{2278}   & \ctd{370}{4571}  & \ctd{557}{557}  & \ctd{0}{125}  & \ctd{69}{2000}     & 0     \\
\kern-0.4em\BV{shd}\kern-0.1em        & \ctd{66}{5136}   & \ctd{38}{335}  & \ctd{23}{1728}  & \ctd{12}{1346}  & \ctd{49}{2290}  & \ctd{33}{1529}  & \ctd{113}{7344}  & \ctd{18}{2278}  & \ctd{39}{4571}   & \ctd{0}{557}   & \ctd{125}{125}& \ctd{15}{2000}     & 0     \\
\kern-0.4em\BV{rfl}\kern-0.1em        & \ctd{1009}{5136} & \ctd{21}{335}  & \ctd{330}{1728} & \ctd{155}{1346} & \ctd{395}{2290} & \ctd{219}{1529} & \ctd{1522}{7344} & \ctd{535}{2278} & \ctd{856}{4571}  & \ctd{69}{557}   & \ctd{15}{125} & \ctd{2000}{2000}   & 0     
\end{tabular}
%
%
%}
%}
%\vspace{0.2cm}
\caption{Challenge III. Target image counts for pairs of factors. The top row shows the counts for the 12 factors detailed in Table \ref{tab:museum_stats}. The colors of each column are normalized w.r.t. the top cell in that column.
\label{tab:museum_counts}
}
%\vspace{-0.5cm}
\end{table}

\begin{table}[t]
%\vspace{-0.3cm}
\setlength{\tabcolsep}{0.10em}
\centering
%\hspace{-0.9cm}
%
%\parbox{0.48\textwidth}
%{
\renewcommand{\arraystretch}{0.8}
{
%\small
%\small
\fontsize{9.25}{11}\selectfont
\centering
\begin{tabular}{c|c c c c c  c c c c c  c c c |}
$\cap$                         &\BV{clp}&\BV{lgt}&\BV{blr}&\BV{glr}&\BV{bgr}&\BV{ocl}&\BV{rot}&\BV{zom}&\BV{vpc}&\BV{sml}&\BV{shd}&\BV{rfl}&\BV{ok}\\
\hline
\kern-0.4em{\em all}\kern-0.1em&\ctb{65.3}{81.0}&\ctb{48.6}{81.0}&\ctb{51.6}{81.0}&\ctb{64.0}{81.0}&\ctb{65.9}{81.0}&\ctb{56.4}{81.0}&\ctb{65.0}{81.0}&\ctb{70.0}{81.0}&\ctb{58.6}{81.0}  &\ctb{34.1}{81.0}&\ctb{70.4}{81.0}  &\ctb{67.5}{81.0}  &\ctb{81.0}{81.0}\\
\hline
%\midrule
\kern-0.4em\BV{clp}\kern-0.1em &\ctr{65.3}{65.3}&\ctr{55.1}{48.6}&\ctr{51.8}{51.6}&\ctr{67.5}{64.0}&\ctr{66.8}{65.9}&\ctr{61.5}{56.4}&\ctr{67.2}{65.0}&\ctr{68.1}{70.0}&\ctr{62.3}{58.6}  &\ctr{45.5}{34.1}&\ctr{72.7}{70.4}&\ctr{67.0}{67.5}  &n/a   \\
\kern-0.4em\BV{lgt}\kern-0.1em &\ctr{55.1}{65.3}&\ctr{48.6}{48.6}&\ctr{41.0}{51.6}&\ctr{43.6}{64.0}&\ctr{59.8}{65.9}&\ctr{43.5}{56.4}&\ctr{48.3}{65.0}&\ctr{44.4}{70.0}&\ctr{46.1}{58.6}  &\ctr{31.2}{34.1}&\ctr{57.9}{70.4}&\ctr{80.9}{67.5}  &n/a   \\
\kern-0.4em\BV{blr}\kern-0.1em &\ctr{51.8}{65.3}&\ctr{41.0}{48.6}&\ctr{51.6}{51.6}&\ctr{48.7}{64.0}&\ctr{48.6}{65.9}&\ctr{37.0}{56.4}&\ctr{52.3}{65.0}&\ctr{64.2}{70.0}&\ctr{43.3}{58.6}  &\ctr{21.0}{34.1}&\ctr{39.1}{70.4}&\ctr{59.4}{67.5}  &n/a   \\
\kern-0.4em\BV{glr}\kern-0.1em &\ctr{67.5}{65.3}&\ctr{43.6}{48.6}&\ctr{48.7}{51.6}&\ctr{64.0}{64.0}&\ctr{62.3}{65.9}&\ctr{47.9}{56.4}&\ctr{65.1}{65.0}&\ctr{67.1}{70.0}&\ctr{60.4}{58.6}  &\ctr{13.5}{34.1}&\ctr{50.0}{70.4}&\ctr{64.5}{67.5}  &n/a   \\
\kern-0.4em\BV{bgr}\kern-0.1em &\ctr{66.8}{65.3}&\ctr{59.8}{48.6}&\ctr{48.6}{51.6}&\ctr{62.3}{64.0}&\ctr{65.9}{65.9}&\ctr{59.6}{56.4}&\ctr{66.6}{65.0}&\ctr{76.1}{70.0}&\ctr{61.2}{58.6}  &\ctr{29.9}{34.1}&\ctr{79.6}{70.4}&\ctr{73.2}{67.5}  &n/a   \\
\kern-0.4em\BV{ocl}\kern-0.1em &\ctr{61.5}{65.3}&\ctr{43.5}{48.6}&\ctr{37.0}{51.6}&\ctr{47.9}{64.0}&\ctr{59.6}{65.9}&\ctr{56.4}{56.4}&\ctr{55.6}{65.0}&\ctr{75.4}{70.0}&\ctr{55.9}{58.6}  &\ctr{40.7}{34.1}&\ctr{78.8}{70.4}&\ctr{64.8}{67.5}  &n/a   \\
\kern-0.4em\BV{rot}\kern-0.1em &\ctr{67.2}{65.3}&\ctr{48.3}{48.6}&\ctr{52.3}{51.6}&\ctr{65.1}{64.0}&\ctr{66.6}{65.9}&\ctr{55.6}{56.4}&\ctr{65.0}{65.0}&\ctr{75.5}{70.0}&\ctr{57.6}{58.6}  &\ctr{32.6}{34.1}&\ctr{73.4}{70.4}&\ctr{70.4}{67.5}  &n/a   \\
\kern-0.4em\BV{zom}\kern-0.1em &\ctr{68.1}{65.3}&\ctr{44.4}{48.6}&\ctr{64.2}{51.6}&\ctr{67.1}{64.0}&\ctr{76.1}{65.9}&\ctr{75.4}{56.4}&\ctr{75.5}{65.0}&\ctr{70.0}{70.0}&\ctr{66.3}{58.6}  &n/a             &\ctr{83.3}{70.4}&\ctr{69.7}{67.5}  &n/a   \\
\kern-0.4em\BV{vpc}\kern-0.1em &\ctr{62.3}{65.3}&\ctr{46.1}{48.6}&\ctr{43.3}{51.6}&\ctr{60.4}{64.0}&\ctr{61.2}{65.9}&\ctr{55.9}{56.4}&\ctr{57.6}{65.0}&\ctr{66.3}{70.0}&\ctr{58.6}{58.6}  &\ctr{33.2}{34.1}&\ctr{64.1}{70.4}&\ctr{61.6}{67.5}  &n/a   \\
\kern-0.4em\BV{sml}\kern-0.1em &\ctr{45.5}{65.3}&\ctr{31.2}{48.6}&\ctr{21.0}{51.6}&\ctr{13.5}{64.0}&\ctr{29.9}{65.9}&\ctr{40.7}{56.4}&\ctr{32.6}{65.0}&n/a             &\ctr{33.2}{58.6}  &\ctr{34.1}{34.1}&n/a              &\ctr{46.4}{67.5}  &n/a   \\
\kern-0.4em\BV{shd}\kern-0.1em &\ctr{72.7}{65.3}&\ctr{57.9}{48.6}&\ctr{39.1}{51.6}&\ctr{50.0}{64.0}&\ctr{79.6}{65.9}&\ctr{78.8}{56.4}&\ctr{73.4}{65.0}&\ctr{83.3}{70.0}&\ctr{64.1}{58.6}  &n/a             &\ctr{70.4}{70.4}&\ctr{80.0}{67.5}  &n/a   \\
\kern-0.4em\BV{rfl}\kern-0.1em &\ctr{67.0}{65.3}&\ctr{80.9}{48.6}&\ctr{59.4}{51.6}&\ctr{64.5}{64.0}&\ctr{73.2}{65.9}&\ctr{64.8}{56.4}&\ctr{70.4}{65.0}&\ctr{69.7 }{70.0}& \ctr{61.6}{58.6}  &\ctr{46.4}{34.1}&\ctr{80.0}{70.4}&\ctr{67.5}{67.5}  &n/a   
\end{tabular}
%
%
%}
}
%\vspace{0.2cm}
\caption{Challenge III. Open MIC performance on the combined set w.r.t. the pairs of 12 factors detailed in Table \ref{tab:museum_stats}. Top-1 accuracies for our JBLD approach are listed. The top row shows results w.r.t. the original 12 factors. Color-coded cells are normalized w.r.t. entries of this row. For each column, intense/pale red indicates better/worse results compared to the top cell, respectively.
\label{tab:museum_cooc}
}
\vspace{-0.3cm}
\end{table}

Below, we give more details about our Open MIC dataset and present more evaluations. 
Table \ref{tab:museum_stats} contains a more detailed description of the 12 factors which we use to analyze performance on our Open MIC dataset. Additionally to the Table 6 in the main submission, which breaks down the performance w.r.t. these 12 factors, we performed an analysis w.r.t. pairs of factors.

Tables \ref{tab:museum_counts} and \ref{tab:museum_cooc} present the image counts and results w.r.t. pairs of factors co-occurring together. The combination of ({\em sml}) with ({\em glr}), ({\em blr}), ({\em bgr}), ({\em lgt}), ({\em rot}) and ({\em vpc}) results in 13.5, 21.0, 29.9, 31.2, 32.6 and 33.2\% mean top-$1$ accuracy, respectively. Therefore, these pairs of factors affect the quality of recognition the most.

\begin{table}[t]
\vspace{-0.3cm}
\setlength{\tabcolsep}{0.19em}
\centering
%\hspace{-0.9cm}
%
%\parbox{0.48\textwidth}
%{
\renewcommand{\arraystretch}{0.8}
{
%\small
%\small
%\fontsize{9.25}{11}\selectfont
\centering
\begin{tabular}{c|c c c c c  c c c c c  c c |}
\multirow{2}{*}{$\cap$}&\BV{sml}&\BV{sml}&\BV{sml}&\BV{sml}&\BV{sml}&\BV{sml}&\BV{blr}&\BV{blr}&\BV{sml}&\BV{lgt}&\BV{lgt}&\BV{lgt}\\
                       &\BV{glr}&\BV{blr}&\BV{bgr}&\BV{lgt}&\BV{rot}&\BV{vpc}&\BV{ocl}&\BV{shd}&\BV{ocl}&\BV{blr}&\BV{ocl}&\BV{glr}\\
\hline
\em{all}&\ctc{52}{405}&\ctc{152}{405}&\ctc{227}{405}& \ctc{16}{405}&\ctc{405}{405}& \ctc{370}{405}&\ctc{235}{405}&\ctc{23}{405}&\ctc{253}{405}& \ctc{105}{405}&\ctc{69}{405}&\ctc{55}{405}\\
\hline
\BV{clp} & \ctd{7}{52}  & \ctd{36}{152}  & \ctd{75}{227} & \ctd{3}{16}  & \ctd{98}{405} & \ctd{124}{370} & \ctd{133}{235} & \ctd{13}{23} & \ctd{90}{253} & \ctd{57}{105}  & \ctd{51}{69} & \ctd{35}{55}\\
\BV{lgt} & \ctd{2}{52}  & \ctd{10}{152}  & \ctd{5}{227}  & \ctd{16}{16} & \ctd{8}{405}  & \ctd{6}{370}   & \ctd{23}{235}  & \ctd{13}{23} & \ctd{7}{253}  & \ctd{105}{105} & \ctd{69}{69} & \ctd{55}{55}\\
\BV{blr} & \ctd{19}{52} & \ctd{152}{152} & \ctd{44}{227} & \ctd{10}{16} & \ctd{122}{405}& \ctd{101}{370} & \ctd{235}{235} & \ctd{23}{23} & \ctd{45}{253} & \ctd{105}{105} & \ctd{23}{69} & \ctd{19}{55}\\
\BV{glr} & \ctd{52}{52} & \ctd{19}{152}  & \ctd{13}{227} & \ctd{2}{16}  & \ctd{38}{405} & \ctd{20}{370}  & \ctd{36}{235}  & \ctd{6}{23}  & \ctd{36}{253} & \ctd{19}{105}  & \ctd{16}{69} & \ctd{55}{55}\\
\BV{bgr} & \ctd{13}{52} & \ctd{44}{152}  & \ctd{227}{227}& \ctd{5}{16}  & \ctd{166}{405}& \ctd{175}{370} & \ctd{78}{235}  & \ctd{10}{23} & \ctd{100}{253}& \ctd{26}{105}  & \ctd{35}{69} & \ctd{19}{55}\\
\BV{ocl} & \ctd{20}{52} & \ctd{45}{152}  & \ctd{100}{227}& \ctd{7}{16}  & \ctd{166}{405}& \ctd{161}{370} & \ctd{235}{235} & \ctd{6}{23}  & \ctd{253}{253}& \ctd{23}{105}  & \ctd{69}{69} & \ctd{16}{55}\\
\BV{rot} & \ctd{38}{52} & \ctd{122}{152} & \ctd{166}{227}& \ctd{8}{16}  & \ctd{405}{405}& \ctd{258}{370} & \ctd{171}{235} & \ctd{18}{23} & \ctd{166}{253}& \ctd{72}{105}  & \ctd{40}{69} & \ctd{31}{55}\\
\BV{zom} & \ctd{0}{52}  & \ctd{0}{152}   & \ctd{0}{227}  & \ctd{0}{16}  & \ctd{0}{405}  & \ctd{0}{370}   & \ctd{20}{235}  & \ctd{1}{23}  & \ctd{0}{253}  & \ctd{2}{105}   & \ctd{0}{69}  & \ctd{0}{55} \\
\BV{vpc} & \ctd{20}{52} & \ctd{101}{152} & \ctd{175}{227}& \ctd{6}{16}  & \ctd{258}{405}& \ctd{370}{370} & \ctd{150}{235} & \ctd{12}{23} & \ctd{161}{253}& \ctd{68}{105}  & \ctd{52}{69} & \ctd{50}{55}\\
\BV{sml} & \ctd{52}{52} & \ctd{152}{152} & \ctd{227}{227}& \ctd{16}{16} & \ctd{405}{405}& \ctd{370}{370} & \ctd{45}{235}  & \ctd{0}{23}  & \ctd{253}{253}& \ctd{10}{105}  & \ctd{7}{69}  & \ctd{2}{55} \\
\BV{shd} & \ctd{0}{52}  & \ctd{0}{152}   & \ctd{0}{227}  & \ctd{0}{16}  & \ctd{0}{405}  & \ctd{0}{370}   & \ctd{6}{235}   & \ctd{23}{23} & \ctd{0}{253}  & \ctd{13}{105}  & \ctd{12}{69} & \ctd{4}{55} \\
\BV{rfl} & \ctd{4}{52}  & \ctd{14}{152}  & \ctd{28}{227} & \ctd{0}{16}  & \ctd{54}{405} & \ctd{42}{370}  & \ctd{23}{235}  & \ctd{2}{23}  & \ctd{22}{253} & \ctd{5}{105}   & \ctd{6}{69}  & \ctd{4}{55}
\end{tabular}
%
%
%}
}
%\vspace{0.2cm}
\caption{Challenge III. Target image counts for the selected triplets of 12 factors detailed in Table \ref{tab:museum_stats}.
The top row shows the counts for the pairs of factors we chose to form triplets. The colors of each column are normalized w.r.t. the top cell in that column.
\label{tab:museum_cooc3}
}
%\vspace{-0.5cm}
\end{table}

\begin{table}[t]
%\vspace{-0.3cm}
\setlength{\tabcolsep}{0.10em}
\centering
%\hspace{-0.9cm}
%
%\parbox{0.48\textwidth}
%{
\renewcommand{\arraystretch}{0.8}
{
%\small
%\small
\fontsize{9.25}{11}\selectfont
\centering
\begin{tabular}{c|c c c c c  c c c c c  c c |}
\multirow{2}{*}{$\cap$}&\BV{sml}&\BV{sml}&\BV{sml}&\BV{sml}&\BV{sml}&\BV{sml}&\BV{blr}&\BV{blr}&\BV{sml}&\BV{lgt}&\BV{lgt}&\BV{lgt}\\
                       &\BV{glr}&\BV{blr}&\BV{bgr}&\BV{lgt}&\BV{rot}&\BV{vpc}&\BV{ocl}&\BV{shd}&\BV{ocl}&\BV{blr}&\BV{ocl}&\BV{glr}\\
\hline
\em{all} & \ctb{13.5}{43.6} & \ctb{21.0}{43.6} & \ctb{29.9}{43.6} & \ctb{31.2}{43.6} & \ctb{32.6}{43.6} & \ctb{33.2}{43.6} & \ctb{37.0}{43.6} & \ctb{39.1}{43.6} & \ctb{40.7}{43.6} & \ctb{40.9}{43.6} & \ctb{43.5}{43.6} & \ctb{43.6}{43.6}\\
\hline
\BV{clp}&\ctr{42.8}{13.5}&\ctr{27.8}{21.0} &\ctr{38.7}{29.9}&\ctr{66.7}{31.2}&\ctr{42.8}{32.6}&\ctr{46.0}{33.2}&\ctr{44.4}{37.0}&\ctr{53.8}{39.1}&\ctr{45.5}{40.7}&\ctr{49.1}{40.9}&\ctr{45.1}{43.5}&\ctr{45.7}{43.6}\\
\BV{lgt}&\ctr{0.0}{13.5} &\ctr{30.0}{21.0} &\ctr{40.0}{29.9}&\ctr{31.2}{31.2}&\ctr{37.5}{32.6}&\ctr{50.0}{33.2}&\ctr{52.3}{37.0}&\ctr{38.5}{39.1}&\ctr{10.0}{40.7}&\ctr{40.9}{40.9}&\ctr{43.5}{43.5}&\ctr{43.6}{43.6}\\
\BV{blr}&\ctr{0.0}{13.5} &\ctr{21.0}{21.0} &\ctr{18.2}{29.9}&\ctr{30.0}{31.2}&\ctr{24.6}{32.6}&\ctr{17.8}{33.2}&\ctr{37.0}{37.0}&\ctr{39.1}{39.1}&\ctr{11.1}{40.7}&\ctr{40.9}{40.9}&\ctr{52.2}{43.5}&\ctr{21.0}{43.6}\\
\BV{glr}&\ctr{13.5}{13.5}&\ctr{0.0}{21.0}  &\ctr{7.7}{29.9} &\ctr{0.0}{31.2} &\ctr{10.5}{32.6}&\ctr{15.0}{33.2}&\ctr{27.8}{37.0}&\ctr{33.3}{39.1}&\ctr{27.8}{40.7}&\ctr{21.0}{40.9}&\ctr{31.2}{43.5}&\ctr{43.6}{43.6}\\
\BV{bgr}&\ctr{7.7}{13.5} &\ctr{18.2}{21.0} &\ctr{29.9}{29.9}&\ctr{40.0}{31.2}&\ctr{27.7}{32.6}&\ctr{31.4}{33.2}&\ctr{37.2}{37.0}&\ctr{60.0}{39.1}&\ctr{33.0}{40.7}&\ctr{46.1}{40.9}&\ctr{51.4}{43.5}&\ctr{42.1}{43.6}\\
\BV{ocl}&\ctr{15.0}{13.5}&\ctr{11.1}{21.0} &\ctr{33.0}{29.9}&\ctr{14.3}{31.2}&\ctr{39.7}{32.6}&\ctr{41.0}{33.2}&\ctr{37.0}{37.0}&\ctr{83.3}{39.1}&\ctr{40.7}{40.7}&\ctr{52.2}{40.9}&\ctr{43.5}{43.5}&\ctr{31.2}{43.6}\\
\BV{rot}&\ctr{10.2}{13.5}&\ctr{24.6}{21.0} &\ctr{27.7}{29.9}&\ctr{37.5}{31.2}&\ctr{32.6}{32.6}&\ctr{31.8}{33.2}&\ctr{38.0}{37.0}&\ctr{50.0}{39.1}&\ctr{39.7}{40.7}&\ctr{43.0}{40.9}&\ctr{60.0}{43.5}&\ctr{32.2}{43.6}\\
\BV{zom}&n/a             &n/a              &n/a             &n/a             &n/a             &n/a             &\ctr{75.0}{37.0}&\ctr{100}{39.1} &n/a             &\ctr{100}{40.9} &n/a             &n/a \\
\BV{vpc}&\ctr{15.0}{13.5}&\ctr{17.8}{21.0} &\ctr{31.4}{29.9}&\ctr{50.0}{31.2}&\ctr{31.8}{32.6}&\ctr{33.2}{33.2}&\ctr{35.3}{37.0}&\ctr{58.3}{39.1}&\ctr{41.0}{40.7}&\ctr{35.3}{40.9}&\ctr{40.4}{43.5}&\ctr{46.0}{43.6}\\
\BV{sml}&\ctr{13.5}{13.5}&\ctr{21.0}{21.0} &\ctr{29.9}{29.9}&\ctr{31.2}{31.2}&\ctr{32.6}{32.6}&\ctr{33.2}{33.2}&\ctr{11.1}{37.0}&n/a             &\ctr{40.7}{40.7}&\ctr{30.0}{40.9}&\ctr{14.3}{43.5}&\ctr{0.0}{43.6} \\
\BV{shd}&n/a             &n/a              &n/a             &n/a             &n/a             &n/a             &\ctr{83.3}{37.0}&\ctr{39.1}{39.1}&n/a             &\ctr{38.5}{40.9}&\ctr{75.0}{43.5}&\ctr{50.0}{43.6}\\
\BV{rfl}&\ctr{75.0}{13.5}&\ctr{50.0}{21.0} &\ctr{39.3}{29.9}&n/a             &\ctr{46.3}{32.6}&\ctr{45.2}{33.2}&\ctr{69.6}{37.0}&\ctr{100}{39.1} &\ctr{68.2}{40.7}&\ctr{100}{40.9} &\ctr{50.0}{43.5}&\ctr{100}{43.6}
\end{tabular}
%
%
%}
}
%\vspace{0.2cm}
\caption{Challenge III. Open MIC performance on the combined set w.r.t. the selected triplets of 12 factors detailed in Table \ref{tab:museum_stats}. Top-1 accuracies for baselines for our JBLD approach are listed. The top row shows results w.r.t. the most difficult pairs of factors we chose to form triplets. The colors of each column are normalized w.r.t. the top cell in that column.
\label{tab:museum_cooc3_res}
}
\vspace{-0.3cm}
\end{table}

Tables \ref{tab:museum_cooc3} and \ref{tab:museum_cooc3_res} present the image counts and results w.r.t. triplets of factors co-occurring together. To obtain these results, we first selected 12 pairs of most challenging co-occurring factors in Table \ref{tab:museum_cooc} and then we further combined them with the 12 main factors from Table \ref{tab:museum_stats} to obtain triplets. As can be seen, ({\em sml+glr+lgt}) and ({\em sml+glr+blr}) combinations of factors were the most difficult to recognize and resulted in 0\% accuracy. Moreover, ({\em sml+bgr+glr}), ({\em sml+ocl+lgt}), ({\em sml+rot+glr}) and ({\em sml+blr+ocl}) resulted in 7.7, 10.0, 10.5, and 11.1\% accuracy which also highlights the difficult nature of these combinations of factors in domain adaptation and recognition.

\ifdefined\arxiv
\else
\begin{table}[t]
\vspace{-0.3cm}
\setlength{\tabcolsep}{0.15em}
\centering
%\hspace{-0.75cm}
%
%\parbox{0.71\textwidth}
\renewcommand{\arraystretch}{0.8}
{
%\small
\centering
%\makebox[\linewidth]{
\begin{tabular}{c|l|}
abbr. & $\quad$details \\
\hline
\BV{lcl} & \pbox{7.5cm}{\vspace{0.125cm}light object clipping \eg, side, base or top including small fragments below 20\% of the exhibit area\vspace{0.125cm}}\\
\BV{hcl} & \pbox{7.5cm}{\vspace{0.125cm}heavy object clipping of large fragments \eg, more than 20\% of the exhibit area\vspace{0.125cm}}\\
\BV{bcl} & \pbox{7.5cm}{\vspace{0.125cm}clipping of the base of sculptures/exhibits \etc\vspace{0.125cm}}\\
\BV{scl} & \pbox{7.5cm}{\vspace{0.125cm}side occlusions of exhibits by humans or other objects\vspace{0.125cm}}\\
\BV{fcl} & \pbox{7.5cm}{\vspace{0.125cm}frontal/central occlusions of exhibit by humans or other objects\vspace{0.125cm}}\\
\BV{ooc} & \pbox{7.5cm}{\vspace{0.125cm}unclassified kind of occlusion\vspace{0.125cm}}\\
\BV{lzo} & \pbox{7.5cm}{\vspace{0.125cm}close-ups of an exhibit\vspace{0.125cm}}\\
\BV{hzo} & \pbox{7.5cm}{\vspace{0.125cm}large close-ups or a heavy zoom on a part of exhibit\vspace{0.125cm}}\\
\BV{lro} & \pbox{7.5cm}{\vspace{0.125cm}small in-plane rotations by no more than 15 degrees due to a tilted camera \etc.\vspace{0.125cm}}\\
\BV{hro} & \pbox{7.5cm}{\vspace{0.125cm}large in-plane rotations by more than 15 degrees due to a tilted camera \etc.\vspace{0.125cm}}\\
\BV{lvp} & \pbox{7.5cm}{\vspace{0.125cm}mismatches by less than 15 degrees between the camera viewpoint and the normal to the surface of face of an exhibit\vspace{0.125cm}}\\
\BV{hvp} & \pbox{7.5cm}{\vspace{0.125cm}mismatches by more than 15 degrees between the camera viewpoint and the normal to the surface of face of an exhibit\vspace{0.125cm}}\\
\BV{spc} & \pbox{7.5cm}{\vspace{0.125cm}light specularities and other reflections from surface\vspace{0.125cm}}
\end{tabular}
%}
%
%
}
%\vspace{0.2cm}
\caption{Challenge III. Additional factors w.r.t. which we evaluate our dataset.}
\label{tab:museum_stats_add}
%\vspace{-0.3cm}
\end{table}
\fi

Tables \ref{tab:museum_stats_add} presents additional factors that we use in our analysis. We split ({\em clp}), ({\em rot}), ({\em vpc}) and ({\em zoo}) into their light and heavy variants. We also split ({\em occ}) into the side and frontal occlusions. We further combine ({\em glr}) and ({\em rfl}) into specularities ({\em spc}). Table \ref{tab:museum_III_add} shows that  the large/heavy variants of truncation, rotation, viewpoint, zoom and occlusions affect performance more than the small/light variants. This highlights the need to further investigate the aspects of invariance to photometric and geometric transformations in domain adaptation algorithms and CNN representations.

\begin{table}[t]
%\vspace{-0.3cm}
\setlength{\tabcolsep}{0.10em}
\centering
%\hspace{-0.9cm}
%
%\parbox{0.48\textwidth}
%{
\renewcommand{\arraystretch}{0.8}
{
\centering
\begin{tabular}{l|c c|| l| c c|| l| c c|}
                   &acc.      & files &                  &  acc.      & files & \BV{$\;$zoo=} &  acc.      & files\\
\hline
\BV{clp=lcl+}      &\multirow{2}{*}{\BW{65.3}}& \multirow{2}{*}{5316}&\BV{occ=scl+}     &\multirow{2}{*}{\BW{56.4}}&\multirow{2}{*}{1529}&\BV{zoo=}         &\multirow{2}{*}{\BW{70.0}}&\multirow{2}{*}{2278}\\
\BV{$\;\;$hcl+bcl} &                          &                      &\BV{$\;\;$fcl+ooc}&                          &                     &\BV{$\;\;$lzo+hzo}&                          &                     \\
\BV{lcl}           &\BW{70.6} & 2827  & \BV{scl}         &  \BW{56.0} & 1086  & \BV{lzo}      &  \BW{74.7} & 1173\\
\BV{hcl}           &\BW{59.0} & 2344  & \BV{fcl}         &  \BW{44.8} & 268   & \BV{hzo}      &  \BW{65.0} & 1106\\
\BV{bcl}           &\BW{65.4} & 739   & \BV{ooc}         &  \BW{56.9} & 851   &               &            &     \\
\hline
\BV{rot=}          &\multirow{2}{*}{\BW{65.0}}&\multirow{2}{*}{7344}& \BV{vpc=} &\multirow{2}{*}{\BW{58.6}}&\multirow{2}{*}{4571} & \BV{spc=} &\multirow{2}{*}{\BW{66.2}}&\multirow{2}{*}{3191} \\
\BV{$\;\;$lro+hro} &                          &                     & \BV{$\;\;$lvp+hvp} & & & \BV{$\;\;$glr+rfl} & &\\
\BV{lro}           &\BW{65.4} & 6724  & \BV{lvp}           &\BW{60.8} & 3241 & \BV{glr} &\BW{64.0} & 1346\\
\BV{hro}           &\BW{60.3} & 622   & \BV{hvp}           &\BW{53.0} & 1345 & \BV{rfl} &\BW{67.5} & 2000
%
%\hline
%\BW{65.3}&\BW{48.6}&\BW{51.6}&\BW{64.0}&\BW{65.9}&\BW{56.4}&\BW{65.0}&\BW{70.0}&\BW{58.6}&\BW{34.1}&\BW{70.4}&\BW{67.5}&\BW{81.0}\\
\end{tabular}
}
%}
%\vspace{-0.3cm}
\caption{Challenge III. Open MIC performance on the combined set w.r.t. additional factors detailed in Table \ref{tab:museum_stats_add}. Top-1 accuracies for our JBLD approach are listed.
\label{tab:museum_III_add}
}
\ifdefined\arxiv\else\vspace{-0.3cm}\fi
\end{table}
\ifdefined\arxiv
\newcommand{\FigBW}{5.0cm}
\else
\newcommand{\FigBW}{3.6cm}
\fi
\begin{figure}[t]%htbp % left bottom right top
\centering
\vspace{-0.3cm}
\centering\includegraphics[trim=0 0 0 0, clip=true, height=\FigBW]{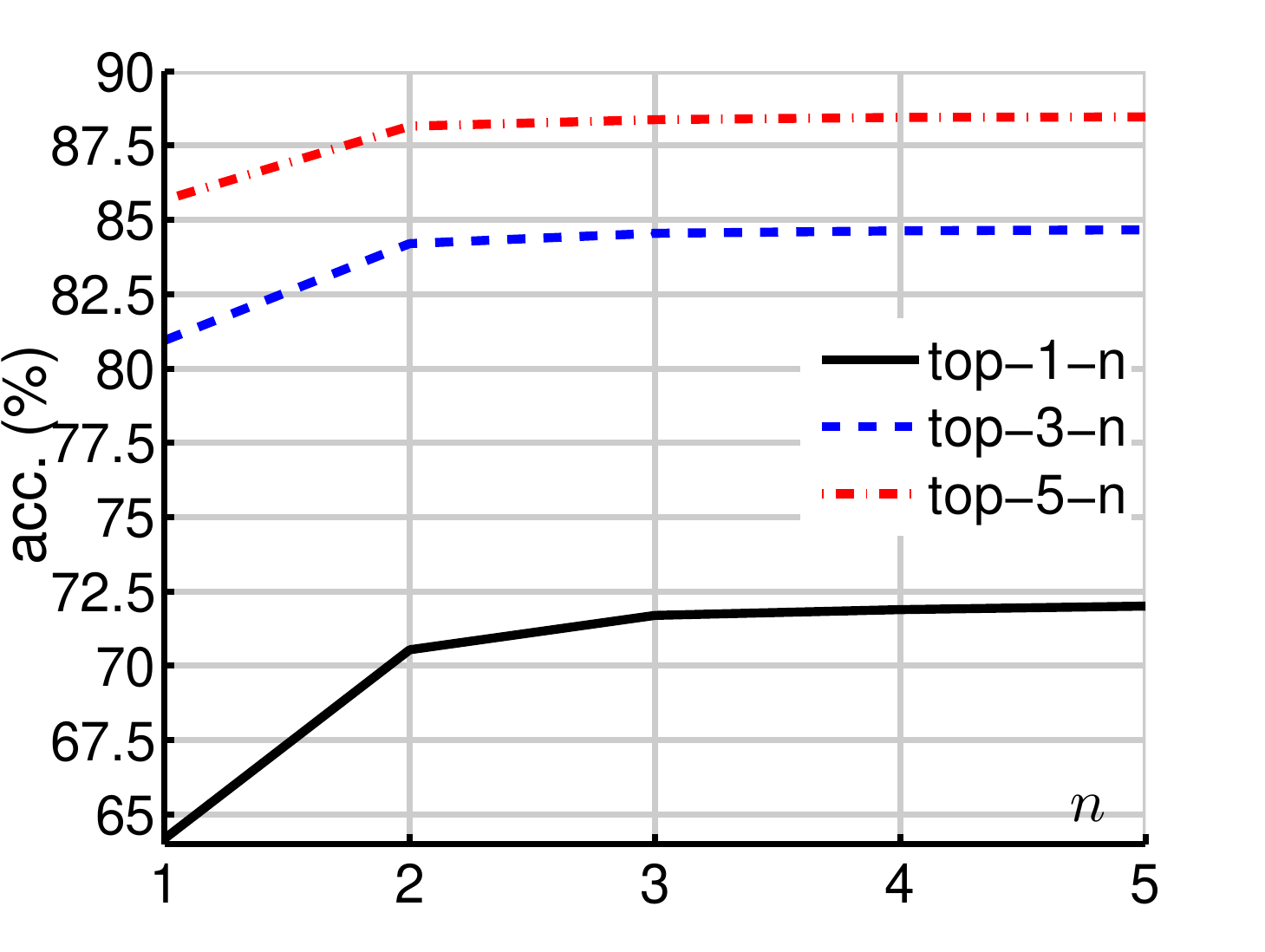}
\ifdefined\arxiv\else\vspace{-0.2cm}\fi
\caption{Challenge II. Open MIC performance on the combined set. In the plot, we list the mean top-$k$-$n$ accuracy (averaged over 5 data splits) w.r.t. $k$ and $n$ for our JBLD approach. We vary $k\!\in\!\{1,3,5\}$ and $n\!\in\!\idx{5}$.
}\vspace{-0.3cm}
\label{fig:supp1}
\end{figure}

Additionally, we revisit Challenge II and present the curves for our proposed top-$k$-$n$ measure on the combined set. Figure \ref{fig:supp1} shows how the performance of our JBLD approach varies w.r.t. $k$ and $n$ variables detailed in Section 5.2 of our main submission. By increasing $n$, we can see the gradual increase in accuracy which means that the classifier sometimes confuses the most salient exhibits in images with less salient objects. Nonetheless, even if $n\!=\!5$, the results on our new dataset are far from saturation leaving the scope for the future works to improve upon our baselines.

Lastly, we investigate the use of Mean Average Precision (MAP) in place of the accuracy as MAP can quantify the quality of recognition for datasets with multiple labels per image. For ({\em Shn}), ({\em Clk}) and ({\em Shx}) subsets, we obtain 71.5, 68.1 and 64.8\% MAP in contrast to 64.3, 61.2 and 48.5\% mean top-$1$ accuracy, respectively. Such results support our claim that the Open MIC dataset is challenging and the results are far from being saturated; making it a good choice for studying domain adaptation and few-shot learning.

\section{Derivatives of $d^2$ and $d'^2$ w.r.t. feat. vectors}
\label{app:der_sec_ord}

Suppose $\mPhi\!=\![\vphi_1,\cdots,\vphi_N]$ and $\mPhi^{*}\!\!=\![\vphi^*_1\!,\cdots,\vphi^*_N\!]$ are some feature vectors of quantity $N$ and $N^{*\!}$, respectively, which are used to evaluate $\cov$ and $\cov^{*}\!$. 
We have to first compute the derivative of the covariance matrix $\cov$ w.r.t. $\vphi_{m'n'}$. We proceed by computing der. of: i) the autocorrelation matrix in \eqref{eq:der_auto} and ii) the outer product of means $\vmu$ in \eqref{eq:der_meanout} and \eqref{eq:der_meanout2}:
\begin{align}
& \!\!\!\frac{\partial\sum_n\!\vphi_n\vphi_n^T}{\partial \phi_{m'n'}}\!=\!\vj_{m'}\vphi_{n'}^T\!+\!\vphi_{n'}\vj_{m'}^T,\label{eq:der_auto}\\
& \!\!\!\frac{\partial\vmu\vmu^T}{\partial \mu_{m'}}\!=\!\vj_{m'}\vmu^T\!+\!\vmu\vj_{m'}^T,\label{eq:der_meanout}\\
& \!\!\!\frac{\partial\vmu\vmu^T}{\partial\phi_{m'n'}}\!=\!\sum_m\!\frac{\partial\vmu\vmu^T}{\partial \mu_m}\frac{\partial \mu_m}{\partial\phi_{m'n'}}\!=\!\frac{1}{N}\!\left(\vj_{m'}\vmu^T\!+\!\vmu\vj_{m'}^T\right),\label{eq:der_meanout2}%\!\!\!
\end{align}
where $\vj_{m'}$ is a vector of zero entries except for position $m'$ which is equal one. 
Putting together \eqref{eq:der_auto}, \eqref{eq:der_meanout} and \eqref{eq:der_meanout2} yields the derivative of $\cov$ w.r.t. $\vphi_{m'n'}$:
\begin{align}
& \frac{\partial\left(\frac{1}{N}\!\sum_n\!\vphi_n\vphi_n^T\right)\!-\!\vmu\vmu^T}{\partial \phi_{m'n'}}\!=\!\frac{1}{N}\!\left(\vj_{m'}\left(\vphi_{n'}\!-\!\vmu\right)^T\!+\!\left(\vphi_{n'}\!-\!\vmu\right)\vj_{m'}^T\right).\label{eq:der_cov}
\end{align}

\vspace{0.05cm}
%\noindent{\textbf{Derivative of the Frobenius norm between covariances w.r.t. the feature vectors.}}
\noindent{\textbf{The derivatives}} of $d^2_g$ w.r.t. covariance $\cov$ as well as $\vphi_{m'n'}$ and $\vphi^*_{m'n'}$ are provided below:
\begin{align}
& \frac{\partial d^2(\cov,\!\cov^{*})}{\partial\cov}\!=\!2d\left(\cov,\cov^{*}\right)\!\frac{\partial d(\cov,\!\cov^{*})}{\partial\cov}\\
& \frac{\partial d^2(\cov,\!\cov^{*})}{\partial\phi_{m'n'}}\!=\!\sum_{m,n}\!\frac{\partial d^2(\cov,\!\cov^{*})}{\partial\Sigma_{mn}}\left(\frac{\partial\cov}{\partial \phi_{m'n'}}\right)_{mn}\nonumber\\
&\;=\!\frac{1}{N}\!\sum_{m,n}\!\frac{\partial d^2(\cov,\!\cov^{*})}{\partial\Sigma_{mn}}\left(\vj_{m'}\left(\vphi_{n'}\!-\!\vmu\right)^T\!\!+\!\left(\vphi_{n'}\!-\!\vmu\right)\vj_{m'}^T\right)_{mn}.
\end{align}
\noindent{\textbf{The derivatives}} of $d^2(\cov,\!\cov^{*})$ (after simplifying summations)  w.r.t. $\mPhi$ and $\mPhi^{*}\!$ are:
\begin{align}
& \!\!\!\!\frac{\partial d^2(\cov,\cov^{*})}{\partial\mPhi}\!=\!\frac{2}{N}\!\frac{\partial d^2(\cov,\cov^{*})}{\partial\cov}\!\left(\mPhi\!-\!\vmu\vOnes^T\right),\\
&\frac{\partial d^2(\cov,\cov^{*})}{\partial\mPhi^*}\!=\!\frac{2}{N^*}\!\frac{\partial d^2(\cov,\cov^{*})}{\partial\cov^*}\!\left(\mPhi^{*}\!\!-\!\vmu^{*}\vOnes^T\right).
\end{align}

\noindent{\textbf{The derivatives}} of $d'^2$ w.r.t. $\mPhi$ and $\mPhi^{*}\!$ are derived from:
\begin{align}
& \!\!\!\!\sum_{m,n}\!\frac{\partial d'^2}{\partial\Phi'_{mn}}\!\frac{\partial(\stkout{\mZ}\Phi)_{mn}}{\partial\mPhi}\!=\!\frac{2\stkout{\mZ}^T}{N}\!\frac{\partial d^2(\cov'\!,\cov^{'*})}{\partial\cov'}\!\left(\mPhi'\!\!-\!\vmu'\!\vOnes^T\right),
\end{align}
where $\mPhi'\!\!=\!\stkout{\mZ}\mPhi$, $\mPhi'^*\!\!=\!\stkout{\mZ}\mPhi^*\!$, $\vmu'\!\!=\!\stkout{\mZ}\vmu$ and $\vmu'^*\!\!=\!\stkout{\mZ}\vmu^*\!$ and $\stkout{\mZ}$ is some projection matrix. We get the following derivatives:
\begin{align}
& \!\!\!\!\!\!\!\!\frac{\partial d'^2}{\partial\mPhi}\!=\!\frac{2\stkout{\mZ}^T}{N}\!\frac{\partial d'^2}{\partial\cov'}\!\left(\mPhi'\!\!-\!\vmu'\!\vOnes^T\right),\;\;
\frac{\partial d'^2}{\partial\mPhi^*\!}\!=\!-\frac{2\stkout{\mZ}^T}{N}\!\frac{\partial d'^2}{\partial\cov'^*\!}\!\left(\mPhi'^*\!\!-\!\vmu'^*\!\vOnes^T\right).
\end{align}
Lastly, based on our Proposition 4 in the main submission, we know that our particular choice $\stkout{\mZ}$ deems $d^2\!=\!d'^2\!$, therefore $\frac{\partial d^2}{\partial\mPhi}\!=\!\frac{\partial d'^2}{\partial\mPhi}$ and $\frac{\partial d^2}{\partial\mPhi^*\!}\!=\!\frac{\partial d'^2}{\partial\mPhi^*\!}$.

%\noindent{\textbf{Derivative of the $\ell_2$-norm between the feature vector means w.r.t. the feature vectors.}}
\vspace{0.3cm}
\noindent{\textbf{The derivatives}} of $||\vmu\!-\!\vmu^{*}||_2^2$ w.r.t. $\vmu$, $\vphi_n$ and $\vphi^{*}_{n'}$ are:
\begin{align}
& \!\!\!\!\!\frac{\partial ||\vmu\!-\!\vmu^{*}||_2^2}{\partial\vmu}\!\!=\!2\left(\vmu\!-\!\vmu^{*}\right),\\
& \!\!\!\!\!\frac{\partial ||\vmu\!-\!\vmu^{*}||_2^2}{\partial\vphi_{n'}}\!\!=\!\!\frac{2\left(\vmu\!-\!\vmu^{*}\right)}{N},\;\frac{\partial ||\vmu\!-\!\vmu^{*}||_2^2}{\partial\vphi_{n'}^{*}}\!\!=\!\!\frac{2\left(\vmu\!-\!\vmu^{*}\right)}{N^*}.\!\!
\end{align}

\ifdefined\arxiv
\renewcommand{\SrcImgWW}{0.15}
\renewcommand{\SrcImgWWW}{0.15}
\renewcommand{\SrcImgHH}{2.78cm}
\else
\newcommand{\SrcImgWW}{0.14}
\newcommand{\SrcImgWWW}{0.15}
\newcommand{\SrcImgHH}{1.5cm}
\fi

\begin{figure}[t]%htbp % left bottom right top
\centering
\ifdefined\arxiv\hspace{-1.5cm}\else\hspace{-0.8cm}\fi
%\vspace{-0.1cm}
%
\begin{subfigure}[b]{\SrcImgWW\linewidth}
\centering\includegraphics[trim=0 0 0 0, clip=true, height=\SrcImgHH]{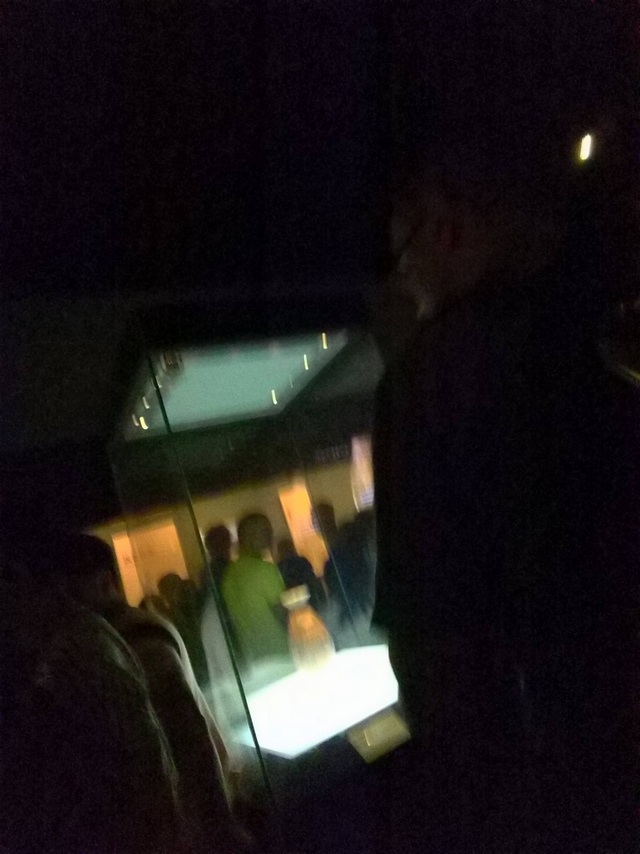}
\end{subfigure}
\begin{subfigure}[b]{\SrcImgWW\linewidth}
\centering\includegraphics[trim=0 0 0 0, clip=true, height=\SrcImgHH]{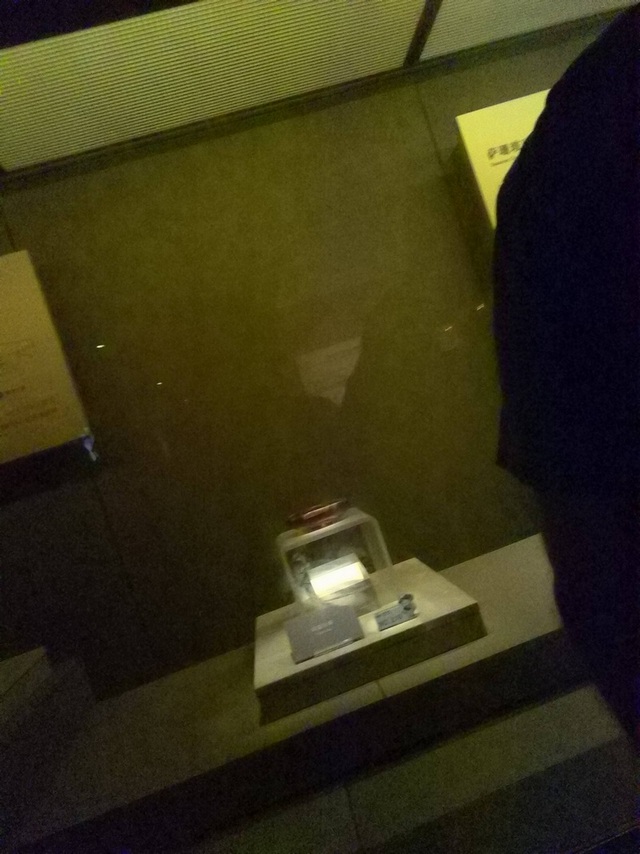}
\end{subfigure}
\begin{subfigure}[b]{\SrcImgWW\linewidth}
\centering\includegraphics[trim=0 0 0 0, clip=true, height=\SrcImgHH]{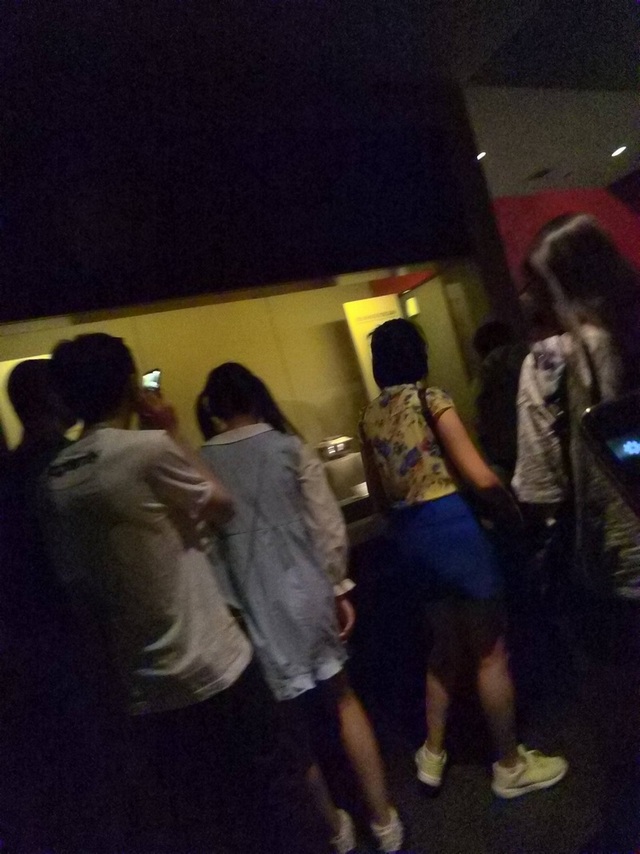}
\end{subfigure}
\begin{subfigure}[b]{\SrcImgWW\linewidth}
\centering\includegraphics[trim=0 0 0 0, clip=true, height=\SrcImgHH]{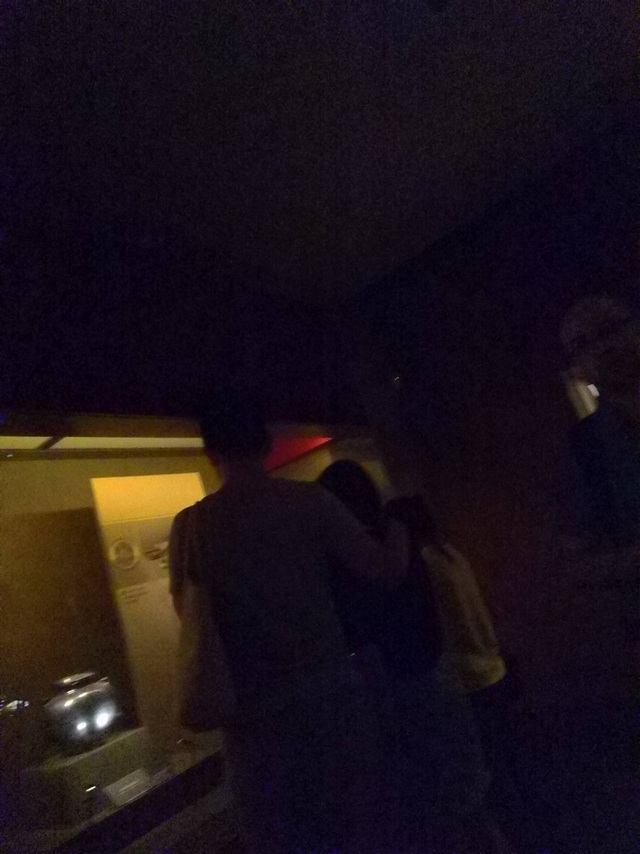}
\end{subfigure}
\begin{subfigure}[b]{\SrcImgWW\linewidth}
\centering\includegraphics[trim=0 0 0 0, clip=true, height=\SrcImgHH]{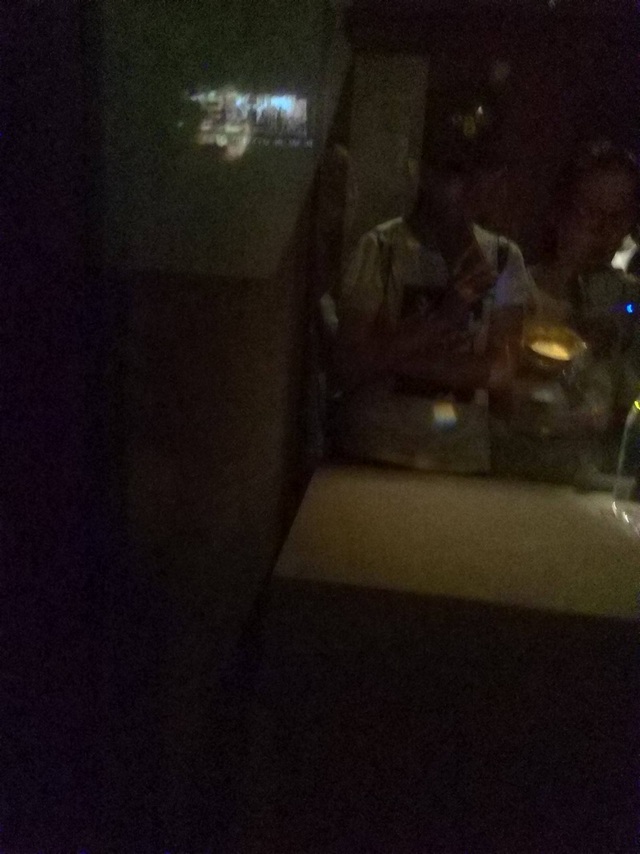}
\end{subfigure}
\begin{subfigure}[b]{\SrcImgWW\linewidth}
\centering\includegraphics[trim=0 0 0 0, clip=true, height=\SrcImgHH]{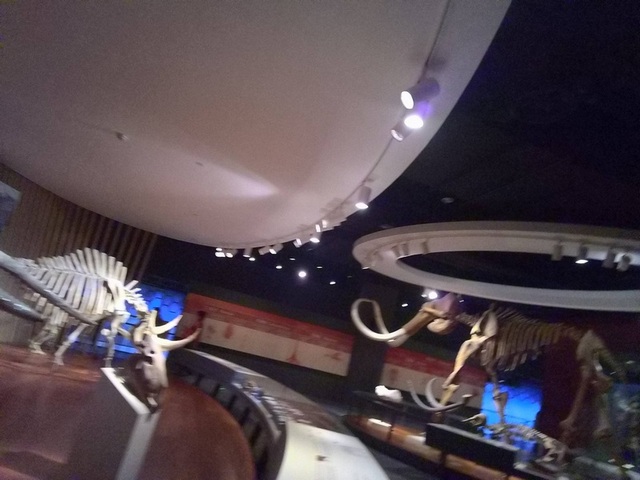}
\end{subfigure}
\caption{Some of the most difficult to identify exhibits from the target domain in the Open MIC dataset.}
\label{fig:images_target_hard}
\vspace{-0.3cm}
\end{figure}

\section{Comparison with the unsupervised domain adaptation.}
\label{sec:unsupervised}

\newcommand{\ShnII}{\color{blue!20!black!30!red}\em Shn}
\newcommand{\ClkII}{\color{blue!20!black!30!red}\em Clk}
\newcommand{\SclII}{\color{blue!20!black!30!red}\em Scl}
\newcommand{\SciII}{\color{blue!20!black!30!red}\em Sci}
\newcommand{\GlsII}{\color{blue!20!black!30!red}\em Gls}

\newcommand{\RelII}{\color{blue!20!black!30!red}\em Rel}
\newcommand{\NatII}{\color{blue!20!black!30!red}\em Nat}
\newcommand{\ShxII}{\color{blue!20!black!30!red}\em Shx}
\newcommand{\ClvII}{\color{blue!20!black!30!red}\em Clv}
\newcommand{\HonII}{\color{blue!20!black!30!red}\em Hon}

While the protocol for supervised domain utilizes the labeled source and the labeled target training data (a few of datapoints per class), the unsupervised domain adaptation assumes larger unannotated target dataset. Below, we evaluate methods such as the Unsupervised Domain Adaptation with Residual Transfer Networks ({\em RTN}) \cite{DRT_DA}, Deep Transfer Learning with Joint Adaptation Networks ({\em JAN}) \cite{JAN_DA} and Deep Hashing Network for Unsupervised Domain Adaptation ({\em DHN}) \cite{DHN_DA} on the ({\ShnII}), ({\ClkII}), and ({\HonII}) subsets of Open MIC. Table \ref{tab:unsuperv} shows that the unsupervised approaches score lower than JBLD despite we used ResNet-50 for all methods, increased numbers of target datapoints and tweaked all hyper-parameters. However, lower results compared to the supervised domain adaptation are expected as the supervised and unsupervised approaches follow very different training protocols.

\begin{table}[h]
%\vspace{-0.3cm}
\setlength{\tabcolsep}{0.10em}
\centering
\parbox{\ParBoxAW\textwidth}
{
\renewcommand{\arraystretch}{0.8}
{
%\fontsize{9}{10}\selectfont
\centering
\hspace{0.05cm}
\begin{tabular}{c c | c c | c c | c c }
\hline
{\em RTN+\ShnII} & 51.0 & {\em JAN+\ShnII} & 49.5 & {\em DHN+\ShnII} & 49.0 & {\em{\BL JBLD}+\ShnII} & {\bf 64.3}  \\
{\em RTN+\ClkII} & 54.7 & {\em JAN+\ClkII} & 51.0 & {\em DHN+\ClkII} & 52.2 & {\em{\BL JBLD}+\ClkII} & {\bf 61.2}  \\
{\em RTN+\HonII} & 66.0 & {\em JAN+\RelII} & 65.2 & {\em DHN+\RelII} & 64.6 & {\em{\BL JBLD}+\RelII} & {\bf 77.3}  \\
\hline
\end{tabular}
}
}
\caption{Evaluation of the unsupervised domain adaptation on the Open MIC dataset.}
\label{tab:unsuperv}
\vspace{-0.3cm}
\end{table}

\section{Evaluations on the Office-Home dataset.}
\label{sec:office-home}

For evaluation on the Office-Home dataset \cite{office_home}, we chose \dsClAr/\dsPrAr~domain pairs, 38 source and 12 target train images per class. The baseline ({\em So}) approach scored 59.5/60.0\% accuracy. For JBLD, we obtained {\bf 61.6}/{\bf 62.2\%}. For 20 source and 3 target training images per class, we obtained 48.1/49.3 ({\em So}) and {\bf 49.2}/{\bf 50.5\%} (JBLD) accuracy. Unsupervised approach ({\em DHN}) scored only 34.69/29.91\% in this setting.

\section{Other recent datasets.}
\label{sec:fine-grained_data}

A complementary to ours is a dataset for fine-grained domain adaptation \cite{multi_task} which contains `easily acquired' $\sim$1M cars of 2657 classes from websites for `fine-grained' domain adaptation on 170 classes and $\sim$100 samples per class. 
In contrast, it took us 6 months and 10 visits to several museums with volunteers to collect a specialist data which cannot be simply found on flicker. We used wearable cameras to capture the target images \eg, skeletons, pottery, tools, jewelery, which all are made of varied materials. Some pieces of art are non-rigid, some emit light, some contain moving parts, some looks extremely similar \etc. The target data exhibits big scale and viewpoint changes as well as occlusions, motion blur and light glares \etc.
\end{appendices}

%\begin{appendices}
%\renewcommand\title[1]{}
%\newcommand\titlerunning[1]{}
%\newcommand\authorrunning[1]{}
%\renewcommand\author[1]{}
%\newcommand\institute[1]{}
%\newcommand\maketitle{}
%\input{supp}
%\end{appendices}

{\small
%\bibliographystyle{plainnat}
%\bibliography{artwork_bib}

\begin{thebibliography}{38}
\providecommand{\natexlab}[1]{#1}
\providecommand{\url}[1]{\texttt{#1}}
\expandafter\ifx\csname urlstyle\endcsname\relax
  \providecommand{\doi}[1]{doi: #1}\else
  \providecommand{\doi}{doi: \begingroup \urlstyle{rm}\Url}\fi

\bibitem[Baxter et~al.(1995)Baxter, Caruana, Mitchell, Pratt, Silver, and
  Thrun]{transfer_workshop_1995}
Jonathan Baxter, Rich Caruana, Tom Mitchell, Lorien~Y. Pratt, Daniel~L. Silver,
  and Sebastian Thrun.
\newblock Learning to learn: {K}nowledge consolidation and transfer in
  inductive systems.
\newblock NIPS Workshop,
  \url{http://plato.acadiau.ca/courses/comp/dsilver/NIPS95_LTL/transfer.workshop.1995.html},
  1995.
\newblock Accessed: 30-10-2016.

\bibitem[Bhatia(2007)]{bhatia_pdm}
R.~Bhatia.
\newblock Positive definite matrices.
\newblock \emph{Princeton Univ Press}, 2007.

\bibitem[Bo and Sminchisescu(2009)]{bo_nystrom}
Liefeng Bo and Cristian Sminchisescu.
\newblock Efficient match kernels between sets of features for visual
  recognition.
\newblock \emph{NIPS}, 2009.

\bibitem[Cherian et~al.(2013)Cherian, Sra, Banerjee, and
  Papanikolopoulos]{anoop_logdet}
Arun Cherian, Suvrit Sra, Adrish Banerjee, and Nikolaos Papanikolopoulos.
\newblock {Jensen-Bregman LogDet Divergence with Application to Efficient
  Similarity Search for Covariance Matrices}.
\newblock \emph{TPAMI}, 35\penalty0 (9):\penalty0 2161--2174, 2013.
\newblock ISSN 0162-8828.
\newblock \doi{10.1109/tpami.2012.259}.
\newblock URL \url{http://dx.doi.org/10.1109/tpami.2012.259}.

\bibitem[Chopra et~al.(2013)Chopra, Balakrishnan, and
  Gopalan]{chopra_icml_workshop}
Sumit Chopra, Suhrid Balakrishnan, and Raghuraman Gopalan.
\newblock Dlid: Deep learning for domain adaptation by interpolating between
  domains.
\newblock \emph{ICML Workshop}, 2013.

\bibitem[Daum{\'e} et~al.(2010)Daum{\'e}, Kumar, and Saha]{frustrating_domain}
Hal Daum{\'e}, III, Abhishek Kumar, and Avishek Saha.
\newblock Frustratingly easy semi-supervised domain adaptation.
\newblock \emph{Proceedings of the 2010 Workshop on Domain Adaptation for
  Natural Language Processing}, pages 53--59, 2010.

\bibitem[Donahue et~al.(2014)Donahue, Jia, Vinyals, Hoffman, Zhang, Tzeng, and
  Darrell]{donahue_decaf}
Jeff Donahue, Yangqing Jia, Oriol Vinyals, Judy Hoffman, Ning Zhang, Eric
  Tzeng, and Trevor Darrell.
\newblock Decaf: A deep convolutional activation feature for generic visual
  recognition.
\newblock \emph{ICML}, 2014.

\bibitem[Ganin et~al.(2016)Ganin, Ustinova, Ajakan, Germain, Larochelle,
  Laviolette, Marchand, and Lempitsky]{ganin_jmlr_adversal}
Yaroslav Ganin, Evgeniya Ustinova, Hana Ajakan, Pascal Germain, Hugo
  Larochelle, Fran\c{c}ois Laviolette, Mario Marchand, and Victor Lempitsky.
\newblock Domain-adversarial training of neural networks.
\newblock \emph{JMLR}, 17\penalty0 (1):\penalty0 2096--2030, 2016.
\newblock ISSN 1532-4435.

\bibitem[Gebru et~al.(2017)Gebru, Hoffman, and Fei{-}Fei]{multi_task}
Timnit Gebru, Judy Hoffman, and Li~Fei{-}Fei.
\newblock Fine-grained recognition in the wild: {A} multi-task domain
  adaptation approach.
\newblock \emph{CoRR}, abs/1709.02476, 2017.
\newblock URL \url{http://arxiv.org/abs/1709.02476}.

\bibitem[Ghifary et~al.(2014)Ghifary, Kleijn, and Zhang.]{dann_com}
M.~Ghifary, W.~B. Kleijn, and M.~Zhang.
\newblock Domain adaptive neural networks for object recognition.
\newblock \emph{CoRR}, abs/1409.6041, 2014.

\bibitem[Girshick et~al.(2014)Girshick, Donahue, Darrell, and
  Malik]{girshick_rich_feat}
Ross Girshick, Jeff Donahue, Trevor Darrell, and Jitendra Malik.
\newblock Rich feature hierarchies for accurate object detection and semantic
  segmentation.
\newblock \emph{CVPR}, pages 580--587, 2014.
\newblock \doi{10.1109/CVPR.2014.81}.
\newblock URL \url{http://dx.doi.org/10.1109/CVPR.2014.81}.

\bibitem[Gong et~al.(2012)Gong, Shi, Sha, and Grauman]{office_calt10}
B.~Gong, Y.~Shi, F.~Sha, and K.~Grauman.
\newblock Geodesic flow kernel for unsupervised domain adaptation.
\newblock \emph{CVPR}, pages 2066--2073, 2012.

\bibitem[Gross et~al.(2010)Gross, Matthews, Cohn, Kanade, and Baker]{multi_pie}
Ralph Gross, Iain Matthews, Jeffrey Cohn, Takeo Kanade, and Simon Baker.
\newblock Multi-pie.
\newblock \emph{Image Vision Comput.}, 28\penalty0 (5):\penalty0 807--813,
  2010.
\newblock ISSN 0262-8856.
\newblock \doi{10.1016/j.imavis.2009.08.002}.
\newblock URL \url{http://dx.doi.org/10.1016/j.imavis.2009.08.002}.

\bibitem[Herath et~al.(2017)Herath, Harandi, and Porikli]{samita_domain}
S.~Herath, M.~Harandi, and F.~Porikli.
\newblock Learning an invariant hilbert space for domain adaptation.
\newblock \emph{CVPR}, 2017.

\bibitem[Koniusz et~al.(2016)Koniusz, Tas, and Porikli]{me_domain}
Piotr Koniusz, Yusuf Tas, and Fatih Porikli.
\newblock Domain adaptation by mixture of alignments of second- or higher-order
  scatter tensors.
\newblock \emph{CoRR}, abs/1409.1556, 2016.

\bibitem[Krizhevsky et~al.(2012)Krizhevsky, Sutskever, and
  Hinton]{krizhevsky_alexnet}
Alex Krizhevsky, Ilya Sutskever, and Geoffrey~E. Hinton.
\newblock {ImageNet} classification with deep convolutional neural networks.
\newblock \emph{NIPS}, pages 1106--1114, 2012.

\bibitem[Kuzborskij et~al.(2016)Kuzborskij, Carlucci, and
  Caputo]{kuzborskij_cvpr16}
Ilja Kuzborskij, Fabio~Maria Carlucci, and Barbara Caputo.
\newblock When na{\"i}ve bayes nearest neighbors meet convolutional neural
  networks.
\newblock \emph{CVPR}, 2016.

\bibitem[L.~Fei-Fei;~Fergus(2006)]{feifei_oneshot}
R.;~Perona L.~Fei-Fei;~Fergus.
\newblock One-shot learning of object categories.
\newblock \emph{TPAMI}, 28:\penalty0 594--611, April 2006.

\bibitem[Li et~al.(2016)Li, Tommasi, Orabona, V{\'a}zquez, L{\'o}pez, Xu, and
  Larochelle]{transfer_workshop_2016}
W.~Li, T.~Tommasi, F.~Orabona, D.~V{\'a}zquez, M.~L{\'o}pez, J.~Xu, and
  H.~Larochelle.
\newblock Task-cv: Transferring and adapting source knowledge in computer
  vision.
\newblock ECCV Workshop, \url{http://adas.cvc.uab.es/task-cv2016}, 2016.
\newblock Accessed: 22-11-2016.

\bibitem[Long et~al.(2016{\natexlab{a}})Long, Wang, and Jordan]{DRT_DA}
Mingsheng Long, Jianmin Wang, and Michael~I. Jordan.
\newblock Unsupervised domain adaptation with residual transfer networks.
\newblock \emph{CoRR}, abs/1602.04433, 2016{\natexlab{a}}.
\newblock URL \url{http://arxiv.org/abs/1602.04433}.

\bibitem[Long et~al.(2016{\natexlab{b}})Long, Wang, and Jordan]{JAN_DA}
Mingsheng Long, Jianmin Wang, and Michael~I. Jordan.
\newblock Deep transfer learning with joint adaptation networks.
\newblock \emph{CoRR}, abs/1605.06636, 2016{\natexlab{b}}.
\newblock URL \url{http://arxiv.org/abs/1605.06636}.

\bibitem[Pennec et~al.(2006)Pennec, Fillard, and Ayache]{PEN06}
Xavier Pennec, Pierre Fillard, and Nicholas Ayache.
\newblock {A Riemannian Framework for Tensor Computing}.
\newblock \emph{IJCV}, 66\penalty0 (1):\penalty0 41--66, 2006.
\newblock ISSN 0920-5691.
\newblock \doi{10.1007/s11263-005-3222-z}.
\newblock URL \url{http://dx.doi.org/10.1007/s11263-005-3222-z}.

\bibitem[Rebuffi et~al.(2017)Rebuffi, Bilen, and Vedaldi]{decathlon_challenge}
Sylvestre-Alvise Rebuffi, Hakan Bilen, and Andrea Vedaldi.
\newblock Learning multiple visual domains with residual adapters.
\newblock Part of the PASCAL in Detail Workshop Challenge,
  \url{http://www.robots.ox.ac.uk/~vgg/decathlon/}, 2017.
\newblock Accessed: 30-10-2017.

\bibitem[Russakovsky et~al.(2015)Russakovsky, Deng, Su, Krause, Satheesh, Ma,
  Huang, Karpathy, Khosla, Bernstein, Berg, and Fei-Fei]{ILSVRC15}
Olga Russakovsky, Jia Deng, Hao Su, Jonathan Krause, Sanjeev Satheesh, Sean Ma,
  Zhiheng Huang, Andrej Karpathy, Aditya Khosla, Michael Bernstein,
  Alexander~C. Berg, and Li~Fei-Fei.
\newblock {ImageNet} large scale visual recognition challenge.
\newblock \emph{IJCV}, 115\penalty0 (3):\penalty0 211--252, 2015.
\newblock \doi{10.1007/s11263-015-0816-y}.

\bibitem[Saenko et~al.(2010)Saenko, Kulis, Fritz, and Darrell]{saenko_office}
Kate Saenko, Brian Kulis, Mario Fritz, and Trevor Darrell.
\newblock Adapting visual category models to new domains.
\newblock \emph{ECCV}, pages 213--226, 2010.
\newblock URL \url{http://dl.acm.org/citation.cfm?id=1888089.1888106}.

\bibitem[Sermanet et~al.(2014)Sermanet, Eigen, Zhang, Mathieu, Fergus, and
  Lecun]{sermanet_overfeat}
Pierre Sermanet, David Eigen, Xiang Zhang, Michael Mathieu, Rob Fergus, and
  Yann Lecun.
\newblock Overfeat: Integrated recognition, localization and detection using
  convolutional networks.
\newblock \emph{ICLR}, 2014.
\newblock URL \url{http://arxiv.org/abs/1312.6229}.

\bibitem[Simonyan and Zisserman(2015)]{simonyan_vgg}
K.~Simonyan and A.~Zisserman.
\newblock Very deep convolutional networks for large-scale image recognition.
\newblock \emph{ICLR}, abs/1409.1556, 2015.

\bibitem[Sun et~al.(2015)Sun, Feng, and Saenko]{frustrating_domain_return}
Baochen Sun, Jiashi Feng, and Kate Saenko.
\newblock Return of frustratingly easy domain adaptation.
\newblock \emph{CoRR}, abs/1511.05547, 2015.
\newblock URL \url{http://arxiv.org/abs/1511.05547}.

\bibitem[Szegedy et~al.(2015)Szegedy, Liu, Jia, Sermanet, Reed, Anguelov,
  Erhan, Vanhoucke, and Rabinovich]{google_net}
Christian Szegedy, Wei Liu, Yangqing Jia, Pierre Sermanet, Scott Reed, Dragomir
  Anguelov, Dumitru Erhan, Vincent Vanhoucke, and Andrew Rabinovich.
\newblock Going deeper with convolutions.
\newblock \emph{CVPR}, 2015.
\newblock URL \url{http://arxiv.org/abs/1409.4842}.

\bibitem[Tommasi et~al.(2010)Tommasi, Orabona, and Caputo]{tommasi_cvpr10}
Tatiana Tommasi, Francesco Orabona, and Barbara Caputo.
\newblock Safety in numbers: Learning categories from few examples with multi
  model knowledge transfer.
\newblock \emph{CVPR}, pages 3081--3088, 2010.
\newblock \doi{10.1109/CVPR.2010.5540064}.

\bibitem[Tommasi et~al.(2014)Tommasi, Tuytelaars, and Caputo]{tomassi_tesbed}
Tatiana Tommasi, Tinne Tuytelaars, and Barbara Caputo.
\newblock A testbed for cross-dataset analysis.
\newblock \emph{Technical Report}, 2014.
\newblock URL \url{https://arxiv.org/abs/1402.5923}.

\bibitem[Tommasi et~al.(2016)Tommasi, Lanzi, Russo, and Caputo]{tommasi_eccv16}
Tatiana Tommasi, Martina Lanzi, Paolo Russo, and Barbara Caputo.
\newblock Learning the roots of visual domain shift.
\newblock \emph{ECCV Workshop}, 2016.

\bibitem[Tzeng et~al.(2015)Tzeng, Hoffman, Darrell, and Saenko]{tzeng_transfer}
E.~Tzeng, J.~Hoffman, T.~Darrell, and K.~Saenko.
\newblock Simultaneous deep transfer across domains and tasks.
\newblock \emph{ICCV}, pages 4068--4076, 2015.

\bibitem[Venkateswara et~al.(2017{\natexlab{a}})Venkateswara, Eusebio,
  Chakraborty, and Panchanathan]{DHN_DA}
Hemanth Venkateswara, Jose Eusebio, Shayok Chakraborty, and Sethuraman
  Panchanathan.
\newblock Deep hashing network for unsupervised domain adaptation.
\newblock \emph{CoRR}, abs/1706.07522, 2017{\natexlab{a}}.
\newblock URL \url{http://arxiv.org/abs/1706.07522}.

\bibitem[Venkateswara et~al.(2017{\natexlab{b}})Venkateswara, Eusebio,
  Chakraborty, and Panchanathan]{office_home}
Hemanth Venkateswara, Jose Eusebio, Shayok Chakraborty, and Sethuraman
  Panchanathan.
\newblock Deep hashing network for unsupervised domain adaptation.
\newblock In \emph{CVPR}, 2017{\natexlab{b}}.

\bibitem[Wang and Hebert(2016)]{xiong_eccv16}
Yu-Xiong Wang and Martial Hebert.
\newblock Learning to learn: Model regression networks for easy small sample
  learning.
\newblock \emph{ECCV}, 2016.

\bibitem[Yeh et~al.(2014)Yeh, Huang, and Wang]{yeh_cca_hetero}
Yi-Ren Yeh, Chun-Hao Huang, and Yu-Chiang~Frank Wang.
\newblock Heterogeneous domain adaptation and classification by exploiting the
  correlation subspace.
\newblock \emph{Transactions on Image Processing}, 23\penalty0 (5), 2014.

\bibitem[Zhou et~al.(2014)Zhou, Lapedriza, Xiao, Torralba, and
  Oliva]{places_dataset}
B.~Zhou, A.~Lapedriza, J.~Xiao, A.~Torralba, and A.~Oliva.
\newblock Learning deep features for scene recognition using places database.
\newblock \emph{NIPS}, 2014.

\end{thebibliography}

%\input{main-arxiv.bbl}
}

\end{document}